\documentclass[english]{article}
\usepackage[T1]{fontenc}
\usepackage[latin9]{inputenc}
\usepackage{bm}
\usepackage{amsmath,mathtools}
\usepackage{amssymb}
\usepackage{geometry}
\usepackage[unicode=true,
 bookmarks=false,
 breaklinks=false,
 pdfborder={0 0 1},
 colorlinks=false]{hyperref}
\hypersetup{colorlinks,
linkcolor=red,
anchorcolor=blue,
citecolor=blue,
urlcolor=blue}
\usepackage{amsthm}
\usepackage{comment}
\usepackage{natbib}
\usepackage{booktabs}

\usepackage{graphicx}

\usepackage[linesnumbered,ruled,vlined]{algorithm2e}

\SetCommentSty{mycommfont}
\usepackage{algorithmic}

\usepackage{float}
\usepackage{multirow}
 
\usepackage{dsfont}
\usepackage{tcolorbox}
\usepackage{tabulary}
\usepackage{colortbl}

\usepackage{color}
\definecolor{yxc}{RGB}{255,0,0}
\definecolor{yjc}{RGB}{125,0,0}
\definecolor{ytw}{RGB}{255,69,0}
\definecolor{gen}{RGB}{0,0,200}
\definecolor{cc}{RGB}{231,117,0}

\allowdisplaybreaks

\DeclareMathOperator{\ind}{\mathds{1}}  

\newcommand{\defn}{\coloneqq}

\newcommand{\cov}{\mathsf{Cov}}

\newcommand{\cA}{\mathcal{A}}
\newcommand{\cB}{\mathcal{B}}
\newcommand{\cC}{\mathcal{C}}
\newcommand{\cD}{\mathcal{D}}
\newcommand{\cE}{\mathcal{E}}
\newcommand{\cF}{\mathcal{F}}
\newcommand{\cG}{\mathcal{G}}

\newcommand{\cN}{\mathcal{N}}

\newcommand{\bE}{\mathbb{E}}
\newcommand{\bP}{\mathbb{P}}
\newcommand{\bR}{\mathbb{R}}

\newcommand{\KL}{\mathsf{KL}}

\newcommand{\diff}{\,\mathrm{d}}

\newcommand{\numpf}[2]{\overset{(\mathrm{#1})}{#2}}
\newcommand{\ol}{\overline}

\newcommand{\wt}{\widetilde}
\newcommand{\clip}{\mathsf{clip}}
\newcommand{\data}{\mathsf{data}}
\newcommand{\sde}{\mathsf{SDE}}
\newcommand{\ode}{\mathsf{ODE}}
\newcommand{\deter}{\mathsf{det}}
\newcommand{\tr}{\mathsf{tr}}
\newcommand{\veps}{\varepsilon}
\newcommand{\TV}{\mathsf{TV}}
\newcommand{\setc}{\mathrm{c}}
\newcommand{\md}{\mathsf{mid}}
\newcommand{\acc}{\mathsf{acc}}
\newcommand{\score}{\mathsf{score}}
\newcommand{\distequal}{\overset{\mathrm d}=}
\newcommand{\aux}{\mathsf{aux}}

\title{Provable Acceleration for Diffusion Models \\under Minimal Assumptions}

\author{Gen Li\footnote{The authors contributed equally. Corresponding author: Changxiao Cai.} \thanks{Department of Statistics, The Chinese University of Hong Kong, Hong Kong; Email: \href{genli@cuhk.edu.hk}{genli@cuhk.edu.hk}.}
\and 
Changxiao Cai\footnotemark[1] \thanks{Department of Industrial and Operations Engineering, University of Michigan, Ann Arbor, USA; Email: \href{mailto:cxcai@umich.edu}{cxcai@umich.edu}.}}
\date{October 2024; ~~ Revised: February 2025}

\begin{document}

\theoremstyle{plain} 
\newtheorem{lemma}{\bf Lemma} 
\newtheorem{proposition}{\bf Proposition}
\newtheorem{theorem}{\bf Theorem}
\newtheorem{corollary}{\bf Corollary} 
\newtheorem{claim}{\bf Claim}

\theoremstyle{remark}
\newtheorem{assumption}{\bf Assumption} 
\newtheorem{definition}{\bf Definition} 
\newtheorem{condition}{\bf Condition}
\newtheorem{property}{\bf Property} 
\newtheorem{example}{\bf Example}
\newtheorem{fact}{\bf Fact}
\newtheorem{remark}{\bf Remark}

\maketitle

\begin{abstract}
Score-based diffusion models, while achieving minimax optimality for sampling, are often hampered by slow sampling speeds due to the high computational burden of score function evaluations.
Despite the recent remarkable empirical advances in speeding up the score-based samplers, theoretical understanding of acceleration techniques remains largely limited.
To bridge this gap, we propose a novel training-free acceleration scheme for stochastic samplers. Under minimal assumptions---namely, $L^2$-accurate score estimates and a finite second-moment condition on the target distribution---our accelerated sampler provably achieves $\varepsilon$-accuracy in total variation within $\widetilde{O}(d^{5/4}/\sqrt{\varepsilon})$ iterations, thereby significantly improving upon the $\widetilde{O}(d/\varepsilon)$ iteration complexity of standard score-based samplers for $\varepsilon\leq 1/\sqrt{d}$. Notably, our convergence theory does not rely on restrictive assumptions on the target distribution or higher-order score estimation guarantees.
\end{abstract}

\medskip


\tableofcontents{}


\section{Introduction}
\label{sec:intro}

Score-based generative models (SGMs), also referred to as diffusion models \citep{song2019generative,song2020score}, have emerged as a powerful framework for sampling from high-dimensional probability distributions, achieving remarkable success across diverse domains. Notable examples include image generation \citep{dhariwal2021diffusion,rombach2022high}, natural language processing \citep{austin2021structured,li2022diffusion}, medical imaging \citep{song2021solving,chung2022score}, and computational biology \citep{trippe2022diffusion,gruver2024protein}. Beyond these computer science and scientific applications, SGMs have also found important applications in the operations research domain by formulating sequential decision-making as generative sequence modeling. This novel perspective has led to impressive performance in planning \citep{janner2022planning,chi2023diffusion}, policy learning \citep{wang2022diffusion,chen2022offline}, and imitation learning \citep{pearce2023imitating,hansen2023idql}. We refer interested readers to \citet{yang2023diffusion,chen2024overview,tang2024score} for detailed surveys on recent advances in methods, applications, and theories of diffusion models.

Diffusion models involve two stochastic processes: a forward process and a reverse process.
The forward process progressively diffuses a sample from the target data distribution into pure noise (typically Gaussian noise). The reverse process, guided by the (Stein) score functions of the forward process, transforms pure noise into a sample from the data distribution, thereby achieving the goal of generative modeling. At the heart of constructing the reverse process lies score matching, the task of learning the score functions along the forward process, which is often accomplished via neural networks in practice \citep{hyvarinen2005estimation,vincent2011connection,song2020sliced}. 
The reverse process can be implemented through either stochastic differential equation (SDE)-based or ordinary differential equation (ODE)-based dynamics, exemplified by the Denoising Diffusion Probabilistic Model (DDPM) \citep{ho2020denoising} and the Denoising Diffusion Implicit Model (DDIM) \citep{song2020denoising}, respectively.
Recent theoretical developments have shown that score-based samplers, when equipped with accurate score estimates, can achieve comparable or superior iteration complexity to classical methods such as Langevin dynamics \citep{bakry2014analysis}---notably, without requiring structural assumptions such as log-concavity or even smoothness of the target distribution \citep{chen2022sampling,chen2023improved,benton2023linear,li2024d}.

While diffusion models have demonstrated impressive empirical performance and provably attain the minimax convergence rate of sampling for smooth densities \citep{oko2023diffusion,zhang2024minimax,dou2024optimal}, they face a significant computational challenge---generating high-quality outputs typically requires a large number of iterative steps. Since each iteration involves a neural network evaluation for the score function, the iteration nature makes them substantially slower than single-step samplers such as variational auto-encoders (VAEs) \citep{kingma2013auto} and generative adversarial networks (GANs) \citep{goodfellow2014generative}. This computational bottleneck highlights the pressing need to accelerate diffusion models without compromising their exceptional output quality.


To address this challenge, various acceleration strategies have been proposed, which can be broadly classified into two categories based on whether they require additional learning. Examples of ``training-based'' methods, such as distillation \citep{salimans2022progressive} and consistency models \citep{song2023consistency}, aim to reduce the computational burdens by adapting pre-trained models into related architectures.
While these approaches demonstrate significant empirical improvements, the costs incurred by additional training processes remain prohibitively high for large-scale models. In contrast, ``training-free'' methods leverage the pre-trained score functions and directly modify the sampling procedure, avoiding resorting to additional training. Hence, this approach offers universal applicability to the off-the-shelf pre-trained diffusion models. Some notable examples in this category include DPM-Solver \citep{lu2022dpm}, DPM-Solver++ \citep{lu2022dpm+}, Unipc \citep{zhao2024unipc}, which have achieved substantial empirical speedups.
However, theoretical understanding of training-free accelerated samplers remains inadequate.


Motivated by this, we investigate accelerating the convergence of the score-based samplers in a training-free manner.
Recent advances in both theory and practice have demonstrated that once $L^2$-accurate score estimates are available, score-based samplers can offer strong sampling guarantees, without requiring structural assumptions on the target distribution. Therefore, we focus on the general setting under minimal conditions---only assuming access to $L^2$-accurate score estimates and a finite second moment of the target distribution.
Given this theoretical focal point, we first present a brief overview of the state-of-the-art iteration complexity of the plain score-based samplers in this setting. Discussions about other sampling scenarios are deferred to Section~\ref{sec:related-work}.
Here and throughout, the iteration complexity refers to the number of iteration steps required to yield a distribution that is $\veps$ close to the target in total variation. 
In \citet{chen2022sampling}, the authors proved the DDPM sampler can achieve an iteration complexity of $\wt O(d^5/\veps^2)$, where $d$ is the dimension of the data. This iteration complexity was improved to $\wt O(d^2/\veps^2)$ by \citet{chen2023improved}, then to $\wt O(d/\veps^2)$ by \citet{benton2023linear}, and recently to $\wt O(d/\veps)$ by \citet{li2024d}. In light of this remarkable theoretical progress, a fundamental question naturally raises: can we develop a score-based sampler that converges faster than $\wt O(d/\veps)$?

While a variety of training-free accelerated sampling algorithms have been developed that achieve this goal, theoretical understanding of acceleration is far from mature. As we elaborate below, existing convergence theories often rely on additional assumptions on the target data distribution or estimation guarantees on the score function. A summary is presented in Table~\ref{table:comparison}.
\begin{itemize}
	\item \emph{Stringent distribution assumption.} One stream of works achieves acceleration by imposing structural assumptions on the target data distribution. For instance, \citet{li2024improved} introduced a DDPM-based sampler with an iteration complexity $\wt O(L d^{1/3}/\veps^{2/3})$ under the condition that the score function is $L$-Lipschitz. In addition, assuming the first $(p+1)$-th derivatives of the score estimates are bounded by $L$, \citet{huang2024convergence} presented a variant of the ODE sampler with an iteration complexity bound $O\big((Ld)^{1+1/p}/\veps^{1/p}\big)$. In \citet{huang2024reverse}, the authors developed a fast sampling algorithm based on a reverse transition kernel framework, achieving $\veps$-accuracy in $\wt O(L^4 d^2 /\veps^{2/p})$ iterations when the score functions are $p$-th order $L$-Lipschitz.
	\item \emph{High-order score estimation requirement.} Another strand of accelerated samplers is developed based on accurate higher-order estimates of the score function. For example, assuming access to reliable estimates of the Jacobian matrix of the score function, 
	\citet{li2024accelerating} proposed an ODE-based procedure with an iteration complexity of $\wt O(d^3/\sqrt{\veps})$. This accelerates the vanilla ODE-based sampler when given accurate Jacobian estimates, which has an iteration complexity of $\wt O(d/\veps)$ \citep{li2024sharp}.
	
\end{itemize}
As a result, this leads to a critical question:
\begin{center}
	\emph{Can we achieve provable acceleration for score-based samplers under minimal assumptions on the data distributions and score estimation?}
\end{center}

\begin{table}[t]
 \centering\label{table:comparison}%
 \begin{tabular}{|c|c|c|}
  \hline 
  \multirow{2}{*}{Sampler} & Iteration complexity & Additional assumptions \tabularnewline
  & (to achieve $\veps$ TV error) & ($s_t^\star(x)$: score function, $s_t(x)$: score estimator)  \tabularnewline
  \hline 
  Stochastic & \multirow{2}{*}{$Ld^{1/3}/\veps^{2/3}$}  & \multirow{2}{*}{$\|\nabla s_t^\star(x) \| \leq L$}\tabularnewline
  \citep{li2024improved} &  & \tabularnewline
  \hline 
  Deterministic & \multirow{2}{*}{$(Ld)^{1+1/p}/\veps^{1/p}$} & \multirow{2}{*}{$\|\nabla^{(p+1)} s_t(x) \| \leq L$} \tabularnewline
  \citep{huang2024convergence} &  &  \tabularnewline
  \hline 
  Stochastic & \multirow{2}{*}{$L^4 d^2/\veps^{2/p}$}  & \multirow{2}{*}{$\|\nabla^{(p)} s_t^\star(x) \| \leq L$}\tabularnewline
  \citep{huang2024reverse} & & \tabularnewline
  \hline 
  \multirow{1}{*}{Deterministic} & \multirow{2}{*}{$d^3/\sqrt{\veps}$} & 
  \multirow{2}{*}{$\int_{\bR^d} \|J_{s_{t}}(x)- J_{s_{t}^{\star}}(x)\| p_{X_{t}}(x) \diff x$ small}
  \tabularnewline
  \citep{li2024sharp} &  & \tabularnewline
  \hline 
  {\cellcolor{lightgray}Stochastic}  & {\cellcolor{lightgray}}  & {\cellcolor{lightgray}}\tabularnewline
  {\cellcolor{lightgray}\textbf{(this paper)}} & {\cellcolor{lightgray}\multirow{-2}{*}{$d^{5/4}/\sqrt{\veps}$}} & {\cellcolor{lightgray}\multirow{-2}{*}{None}}  \tabularnewline
  \hline 
 \end{tabular}\caption{Comparison with prior iteration complexities better than $d/\veps$ (ignoring log factors). 
  Here $J_{f}: \mathbb{R}^d \to \mathbb{R}^{d\times d}$ denotes the Jacobian matrix of a function $f:\mathbb{R}^d\to\mathbb{R}^d$.}
\end{table}

\subsection{Contributions}
Encouragingly, our results provide an affirmative answer to the above question. We develop a novel sampling scheme that achieves an iteration complexity of $\wt O(d^{5/4}/\sqrt{\veps})$, significantly improving upon the $\wt O(d/\veps)$ iteration complexity of standard score-based samplers for $\varepsilon \leq 1/\sqrt{d}$.
This acceleration substantially reduces the computational burden required to achieve minimax optimal sampling.

Our sampler incorporates a momentum term into the sampling update rule by leveraging higher-order approximations to the probability flow ODE. 
Notably, our convergence theory replies \emph{solely} on access to $L^2$-accurate score estimates and a finite second-order moment condition on the target data distribution. This illustrates that $L^2$-accurate first-order score estimates are sufficient to achieve provable sampling acceleration, without the need for higher-order score estimates or additional smoothness assumptions on the target distributions.

While our sampler shares similar intuitions with accelerated ODE-based samplers in prior works \citep{lu2022dpm+,li2024accelerating}, in contrast to their fully deterministic strategies, it introduces additional randomness to mitigate potentially large worst-case errors in approximating the probability flow ODE. This randomization effectively averages out such errors, ultimately leading to provable acceleration under minimal assumptions.
Technically, we achieve this by developing a novel analysis framework that rigorously characterizes the smoothness property of the score functions induced by the probability flow ODE and quantifies the higher-order approximation errors up to low-probability events.
To the best of our knowledge, our result provides the first provable acceleration for score-based samplers that need minimal score estimation requirements and accommodates a broad class of target data distributions.


\subsection{Other related work}
\label{sec:related-work}

\paragraph*{Convergence theory for diffusion models.}
Early convergence guarantees for SDE-based samplers were either qualitative \citep{de2021diffusion,liu2022let,pidstrigach2022score}, replied on $L^\infty$-accurate score estimates \citep{de2021diffusion,albergo2023stochastic}, or exhibited exponential dependence \citep{de2022convergence,block2020generative}.
\citet{lee2022convergence} established the first polynomial iteration complexity given $L^2$-accurate score estimates, albeit assuming a log-Sobelev inequality on the target distribution. \citet{chen2022sampling,lee2023convergence} later relaxed this assumption by requiring Lipschitz scores and bounded support/moment conditions on the target distribution. 
For ODE-based samplers, \citet{chen2023restoration} provided the first convergence guarantee, though without explicit polynomial dependencies and requiring exact score estimates. \citet{chen2024probability} improved upon this by studying variants of the ODE that incorporate additional stochastic corrector steps. \citet{li2023towards} established an iteration complexity of $\wt O(d^2/\veps)$ for ODE-based samplers assuming accurate Jacobian estimates of scores, which was later improved to $\wt O(d/\veps)$ \citep{li2024sharp}. Recent work has also explored convergence in 2-Wasserstein distance \citep{gao2024convergence,tang2024contractive}. 

\paragraph*{Training-free acceleration schemes.}
Training-free accelerated samplers typically leverage efficient numerical methods for solving reverse SDE/ODEs. For ODE-based samplers, researchers have exploited semi-linear structures using higher-order ODE solvers \citep{lu2022dpm,lu2022dpm+}, exponential integrators \citep{zhang2022fast}, and predictor-corrector frameworks \citep{zhao2024unipc}. Acceleration for SDE-based samplers, though less explored due to the inherent complexity of solving SDEs, has progressed through stochastic Improved Euler's method \citep{jolicoeur2021gotta}, stochastic Adams method \citep{xue2024sa}, and stochastic Runge-Kutta methods \citep{wu2024stochastic}.
In addition to resorting to efficient ODE/SDE solvers, alternative acceleration approaches include parallel sampling \citep{chen2024accelerating,gupta2024faster} and exploitation of low-dimensional structures underlying the target distributions \citep{li2024adapting,huang2024denoising,azangulov2024convergence}.

\paragraph*{Other theory for diffusion models.}
Beyond convergence analysis, another line of work focused on the sample complexity of score estimation.
\citet{block2020generative} provided a sample complexity bound in terms of Rademacher complexity, and  \citet{oko2023diffusion,chen2023score} established the sample complexity using neural networks to estimate scores.
From the perspective of nonparametric statistics, \citet{wibisono2024optimal,zhang2024minimax,dou2024optimal} proposed kernel-based methods to achieve optimal score estimation.
The minimax optimality of diffusion models was then established for various target density classes, including Besov \citep{oko2023diffusion}, Sobolev \citep{zhang2024minimax}, and H\"older spaces \citep{dou2024optimal}.
Sample complexity reduction through low-dimensional data structures has also been investigated in \cite{chen2023score,wang2024diffusion}.
In addition, \citet{han2024neural} established optimization guarantees for score matching using two-neural networks trained by gradient descent.
In addition to SGMs, convergence of flow-based generative modeling has also been studied in \citet{cheng2024convergence,xu2024normalizing}.

\subsection{Notation}

\label{subsec:Notations}
For any integer $N>0$, denote by $[N]:=\{1,2,\cdots,N\}$. For any matrix $A$, we use $\|A\|$, $\tr(A)$, and $\deter(A)$ to denote its spectral norm, trace, and determinant, respectively.
For two probability distributions $P, Q$, $\mathsf{KL}(P\,\|\,Q)=\int \log(\frac{\mathrm d P}{\mathrm d Q}) \diff P$ stands for the KL divergence and $\TV(P,Q)=\frac12\int |\mathrm d P - \mathrm d Q|$ represents the total variation. For random vectors $X,Y$ with probability density functions $p_X,p_Y$, we use the notations $\KL(X\,\|\,Y)=\KL(p_X\,\|\,p_Y)$ and $\TV(X,Y) = \TV(p_X,p_Y)$ interchangeably. Throughout the paper, we slightly abuse notation by using $p_X$ to denote both the distribution and the density of $X$.
Let $\mathds{1}\{ \cdot \}$ denote the indicator function. For any event $\cE$, we denote $\bE_{\cE}[\cdot] \defn \bE\big[\cdot \ind\{\cE\}\big]$.

For any two functions $f(n),g(n) > 0$, $f(n)\lesssim g(n)$ or $f(n)=O\big(g(n)\big)$ means $f(n)\leq Cg(n)$ for some absolute constant $C>0$; $f(n)\gtrsim g(n)$ or $f(n)=\Omega\big(g(n)\big)$ indicates $f(n)\geq C'g(n)$ for some absolute constant $C'>0$; $f(n)\asymp g(n)$ or $f(n)=\Theta\big(g(n)\big)$ represents that $Cf(n)\leq g(n)\leq C'f(n)$ for some absolute constants $C' > C > 0$. The notations $\wt{O}(\cdot)$, $\wt{\Omega}(\cdot)$, and $\wt{\Theta}(\cdot)$ hide logarithmic factors.  In addition, $f(n)=o(g(n))$ denotes $\limsup_{n\rightarrow\infty}f(n)/g(n)=0$ and $f(n)=\omega(g(n))$ means $\liminf_{n\rightarrow\infty}f(n)/g(n)=\infty$.

\subsection{Organization}
The rest of the paper is organized as follows. Section~\ref{sec:problem} reviews the background of SGMs and introduces our problem setup. Section~\ref{sec:result} presents our proposed accelerated sampler along with its theoretical guarantees. The convergence analysis is detailed in Section~\ref{sec:analysis}; with all technical proofs and lemmas deferred to the appendix. While our primary focus is on theory, we provide numerical experiments on a toy example in the appendix to validate our convergence results. Finally, we conclude with a discussion of future directions in Section~\ref{sec:discussion}. 
\section{Problem formulation}
\label{sec:problem}

In this section, we provide a brief introduction to SGMs and introduce the assumptions for our algorithm and theory.

\subsection{Preliminaries}
\label{sec:prelim}

\paragraph*{Forward process.}
Starting from the target data distribution $X_0 \sim p_{\mathsf{data}}$ in $\bR^d$, the forward process evolves as follows:
\begin{align}
\label{eq:forward}
X_t =\sqrt{\alpha_t}X_{t-1} + \sqrt{1-\alpha_t} W_t, \quad t=1,2,\dots,T,
\end{align}
where $\alpha_1,\dots,\alpha_T\in(0,1)$ are the learning rates and $W_1,\dots,W_T\overset{\mathrm{i.i.d.}}{\sim}\mathcal{N}(0,I_d)$ are standard Gaussian random vectors in $\bR^d$.
For continence of notation, let us denote
\begin{align}
\label{eq:overline-alpha}
\overline{\alpha}_t:=\prod_{k=1}^t\alpha_k,\quad t=1,2,\dots,T.
\end{align}
This allows us to express
\begin{align}
	X_t =\sqrt{\ol \alpha_t}X_{0} + \sqrt{1-\ol\alpha_t} \,\ol W_t,\quad t=1,2,\dots,T, \label{eq:forward-dist}
\end{align}
where $\ol W_t\sim\mathcal{N}(0,I_d)$ is a standard Gaussian random vector independent of $X_{0}$.
%
In particular, $p_{X_T}$ is approximately a standard multivariate normal distribution when $\ol\alpha_T$ is sufficiently small.
The continuum limit of the forward process \eqref{eq:forward} can be modeled by the following SDE:
\begin{align}
\label{eq:forward-SDE}
\mathrm{d}X_{t} = -\frac{1}{2}\beta_t X_{t}\,\mathrm{d}t + \sqrt{\beta_t} \,\mathrm{d}B_{t}, \quad X_0 \sim p_{\data}; \quad t\in[0,T]
\end{align}
for some function $\beta_t:[0,T]\rightarrow \bR$, where $(B_t)_{t\in[0,T]}$ denotes a standard Brownian motion in $\bR^d$.

\paragraph*{Reverse process.}
The core of SGMs lies in constructing the time-reversal of the forward process.
Classical results in SDE \citep{anderson1982reverse,haussmann1986time} establish that for a solution $(X_t)_{t\in[0,T]}$ to the forward SDE \eqref{eq:forward-SDE}, the time reversal process $(Y^\sde_t)_{t\in[0,T]}$, defined by $Y^\sde_t \defn X_{T-t}$, satisfies the following reverse SDE: 
\begin{align}
	\label{eq:reverse-SDE}
	\mathrm{d}Y_t^\sde = \frac12\beta_{T-t}\Big( Y_t^\sde + 2\nabla \log p_{X_{T-t}}\big(Y_t^\sde\big)\Big)\diff t + \sqrt{\beta_{T-t}}\diff B_t, \quad Y_0^\sde \sim p_{X_T}; \quad t\in[0,T].
\end{align}
Here, $p_{X_t}$ denotes the marginal distribution of $X_t$ in \eqref{eq:forward-SDE}, and the gradient of the logarithm of the probability density, $\nabla\log p_{X_{t}}(x)$, is known as the \emph{score function} of $p_{X_t}(x)$, where the gradient is with respect to $x$. 
\begin{definition}\label{def-score} The score function of $p_{X_t}$, denoted by $s^\star_t(\cdot): \mathbb{R}^d \rightarrow \mathbb{R}^d$, is defined as
\begin{align}
s_t^\star(x) &:= \nabla \log p_{X_t}(x) 
,\quad \forall x\in\bR^d.
\end{align}
\end{definition}


Alternatively, there also exists a deterministic process $(Y^\ode_t)_{t\in[0,T]}$ that, when initialized at $Y^\ode_0 \sim p_{X_T}$, shares the same marginal distributions as \eqref{eq:reverse-SDE}, i.e., $Y^\ode_t \overset{\mathrm d}{=} Y^\sde_t$ for all $t\in[0,T]$. This process, referred to as \emph{probability flow ODE} \citep{song2020score}, is characterized by:
\begin{align}
	\label{eq:prob-flow-ode}
	\mathrm{d}Y_t^\ode = \frac12\beta_{T-t}\Big( Y_t^\ode + \nabla \log p_{X_{T-t}}\big(Y_t^\ode\big)\Big)  \diff t, \quad Y_0^\ode \sim p_{X_T}; \quad t\in[0,T].
\end{align}
We remark that the two processes $(Y^\sde_t)_t$ and $(Y^\ode_t)_t$ have identical marginal distributions but distinct joint or path-wise distributions.

Notably, the reverse processes---both the SDE \eqref{eq:reverse-SDE} and ODE \eqref{eq:prob-flow-ode}---are fully characterized by the score functions of the forward process. This pivotal role makes them the cornerstone of successful generative modeling.

\subsection{Assumptions}

\paragraph*{Score function estimation.}

In practical applications, the true score functions $(s_t^\star)_{t\in[T]}$ are unknown and must be learned based on samples drawn from the target distribution. The estimation of the score functions, known as score matching \citep{hyvarinen2005estimation,vincent2011connection}, is typically accomplished by minimizing the $L^2(p_{X_t})$ loss via neural networks \citep{song2020score}. Given these score estimates $(s_t)_{t\in[T]}$, we can then start from $Y\sim\cN(0,I_d)$ and implement a discretized reverse process. 

This approach motivates our analysis under the assumption of $L^2$-accurate score estimates across all time steps.
\begin{assumption}\label{assu:score-error}
The score estimates $(s_t)_{t\in[T]}$ for the score functions $(s_t^\star)_{t\in[T]}$ satisfy
\begin{align} \label{eq:score-error}
\varepsilon_{\mathsf{score}}^2 = \frac{1}{T}\sum_{t = 1}^T \mathbb{E}_{X_t\sim p_{X_t}}\Big[\big\|s_t(X_t) - s^\star_t(X_t)\big\|_2^2\Big]
=: \frac{1}{T}\sum_{t = 1}^T \varepsilon_t^2.
\end{align}
\end{assumption}

\paragraph*{Target data distribution.}
Next, we impose a mild assumption on the target data distribution.
\begin{assumption}\label{assu:distribution}
The target data distribution has a bounded second moment in the sense that
\begin{align}
\mathbb{E}\big[\|X_0\|_2^2\big] < T^{C_R}
\end{align}
for arbitrarily large constant $C_R > 0$.
\end{assumption}
In short, this assumption requires the second moment of the target data distribution to be at most polynomially large in the iteration number $T$.
As $T$ typically grows polynomially in the dimension $d$, this assumption allows the second moment to be exceedingly large, given that the exponent $C_R$ can be arbitrarily large.
In particular, any distribution with constant-order second moments naturally falls under this requirement.
We express the condition in terms of $T$ in order to express the convergence guarantees in cleaner and more concise manner.
This is among the weakest assumptions imposed in sampling analysis, one that is typically satisfied by empirical data in practical sampling applications.

\section{Main results}
\label{sec:result}

In this section, we introduce an accelerated score-based sampler and present the theoretical guarantees on its convergence rate.

\subsection{Accelerated sampler}

\paragraph*{Learning rate schedule.}
Let us first introduce the learning rate schedule $(\alpha_t)_{t\in[T]}$ for the proposed sampler. Recall the definition of $\ol\alpha_t\defn \prod_{1\leq k \leq t} \alpha_k$ in \eqref{eq:overline-alpha} and the distribution characterization in \eqref{eq:forward-dist} that $X_t \sim \cN\big(\sqrt{\ol \alpha_t}\,x, (1-\ol\alpha_t)I_d\big)$ given $X_0=x$. We can specify the learning rates $(\alpha_t)_{t\in[T]}$ through $(\ol\alpha_t)_{t\in[T]}$ by the following recursive relationship:
\begin{align}
\label{eq:learning-rate}
\overline{\alpha}_{T} = \frac{1}{T^{C_0}},
\qquad\text{and}\qquad
\overline{\alpha}_{t-1} = \overline{\alpha}_{t} + C_1\frac{\log T}{T} \overline{\alpha}_{t}(1-\overline{\alpha}_{t}), \quad t = T,\dots,2,
\end{align}
where $C_0, C_1 > 0$ are sufficiently large absolute constants with $C_1/C_0$ large enough. 
As demonstrated in Lemma~\ref{lemma:step-size}, these choices ensure
\begin{align*}
	1 - \alpha_1 \leq T^{-C_1/4}\qquad \text{and}\qquad 1 - \alpha_t \lesssim \frac{\log T}{T},\quad t\geq 2.
\end{align*}
\begin{remark}
Our results remain valid for a broader class of learning rate schedules satisfying
\begin{align*}
\alpha_1 = 1-\frac{1}{\mathsf{poly}(T)},\qquad \overline{\alpha}_{T} = \frac{1}{\mathsf{poly}(T)},
\qquad\text{and}\qquad
\frac{1-\alpha_t}{1-\overline{\alpha}_{t}} \asymp \frac{\log T}{T}, \quad t = T,\dots,2.
\end{align*}
These requirements ensure that $1-\overline{\alpha}_{1}$ is sufficiently small ($p_{X_1}$ is close to $p_{X_0}$) and that $\overline{\alpha}_{T}$ is small ($p_{X_T}$ is nearly Gaussian), which is realized by the chosen ratio $(1-\alpha_t)/(1-\ol\alpha_t)\asymp \log T/T$.
More generally, our analysis framework can accommodate arbitrary learning rate designs and the resulting convergence rate depends on how the learning rate $1-\alpha_t$ compares to $1-\ol\alpha_t$.
\end{remark}


\paragraph*{Sampling procedure.}
With the learning rate schedule in hand, we are now ready to present our acceleration procedure. 

Initialized at $Y_{T} \sim \mathcal{N}(0, I_d)$, the sampler employs the update rule as follows. 
Working backward from $t=T,\dots,2$, we first compute an intermediate point $Y^\md_t\equiv Y^\md_t(Y_t)$ based on the current iterate $Y_t$ and noise $Z^\md_t$:
\begin{subequations}
\label{eq:sampler}
\begin{align}
Y^\md_t  &\defn \frac{1}{\sqrt{\alpha_{t}}}\Big(Y_{t} + \frac{1-\alpha_t}{2\alpha_t}s_t(Y_t)\Big) + (1-\alpha_t)Z^\md_t, \label{eq:sampler-mid}
\end{align}
where $Z_t^\md\overset{\mathrm{i.i.d.}}{\sim} \mathcal{N}(0, I_d)$, $t\geq 2$ are standard Gaussian random vectors in $\bR^d$. Next, using the points $Y_t,Y^\md_t$ and noise $Z^\md_t,Z_t$, we generate the next iterate $Y_{t-1}\equiv Y_{t-1}(Y_t,Y^\md_t)$ by
\begin{align}
Y_{t-1} &\defn \frac{1}{\sqrt{\alpha_t}}\bigg(Y_{t} + (1-\alpha_t)\Big(s_t(Y_t) 
+ \clip_t\Big\{\alpha_{t}^{3/2}s_{t-1}(Y^\md_t) - s_t\big(Y_t + (1-\alpha_t)Z^\md_t\big)\Big\}\Big) 
+ \sigma_tZ_{t}\bigg) \label{eq:sampler-final}.
\end{align}
\end{subequations}
Here, $Z_t\overset{\mathrm{i.i.d.}}{\sim} \mathcal{N}(0, I_d)$, $t\geq 2$ are standard Gaussian random vectors independent of $(Z_t^\md)_{t=2}^T$, the variance parameter $\sigma_t$ is chosen as $\sigma_t^2 \defn \alpha_t - 1/(3-2\alpha_t)$, and $\clip_t\{\cdot\}:\bR^d \rightarrow \bR^d$ is a thresholding function given by
\begin{align}
	\clip_t\{x\} := x\,\ind\bigg\{\|x\|_2 \leq C_{\clip} (1-\alpha_t)\bigg(\frac{d\log T}{1-\overline{\alpha}_t}\bigg)^{3/2}\bigg\}, \label{def:clip}
\end{align}
for some absolute constant $C_{\clip} > 0$. The last iterate $Y_1$ serves as the sample generated by our sampler.

In short, our proposed sampler first employs a DDPM-type update rule to generate a mid-point $Y^\md_t$. This intermediate step allows to incorporate an extra term $$\clip_t\Big\{\alpha_{t}^{3/2}s_{t-1}(Y^\md_t) - s_t\big(Y_t + (1-\alpha_t)Z^\md_t\big)\Big\},$$ into the subsequent update for $Y_{t-1}$. 
This additional term acts as a refined second-order approximation for the score function, effectively reducing the discretization error and leading to a faster convergence rate.  To mitigate potentially large approximation errors in worst-case scenarios, we introduce a thresholding procedure---a technique also used in practical applications \citep{saharia2022photorealistic}. 

While this idea of leveraging a second-order approximation shares similarities with existing ODE-based accelerated samplers (e.g.\ \citet{lu2022dpm}), our method distinguishes itself by injecting additional random noise in each iteration. This stochastic component effectively reduces the second-order discretization error on average, thereby enabling faster sampling speeds without the  second-order accuracy assumptions on the Jacobians of scores required in prior works \citep{li2024accelerating}. 
Moreover, even though our sampler introduced random noise, it draws on insights from discretizing a probability flow ODE---a connection we will on in Section~\ref{sec:analysis} when we discuss the rationale of the acceleration scheme.


\begin{remark}
	Instead of running the sampling procedure all the way to $t=0$, we stop the update at $t=1$ and take $p_{X_1}$ as the new target distribution. 
	Since $X_1 = \sqrt{\alpha_1}X_0+\sqrt{1-\alpha_1}Z_1$ with $Z_1\sim\cN(0,I_d)$ and the learning rate $\alpha_1$ is chosen such that $1-\alpha_1=o(1)$, the distributions $p_{X_1}$ and $p_{X_0}$ are exceedingly close.
	This is often referred to \emph{early stopping}, as the score functions can blow up as $t\rightarrow 0$ for non-smooth target data distributions, rendering  non-trivial guarantees in total variation or KL divergence unachievable.
	 The early stopping technique is widely employed in both real-world applications \citep{song2020score} and theoretical analysis \citep{chen2022sampling,benton2023linear}.
\end{remark}

\subsection{Convergence guarantee}

We now present the convergence theory for our proposed sampler, with proof outline deferred to Section~\ref{sec:analysis}.
\begin{theorem} \label{thm:main}
Suppose that Assumptions \ref{assu:score-error} and \ref{assu:distribution} hold.  
The output $Y_1$ of the sampler \eqref{eq:sampler} with the learning rate schedule \eqref{eq:learning-rate} satisfies
\begin{align}\label{eq:TV-main}
\mathsf{TV}\left(p_{X_1}, p_{Y_{1}}\right) 
\le \sqrt{\mathsf{KL}\left(p_{X_1}\parallel p_{Y_{1}}\right)} 
\le C \bigg( \frac{d^{5/2}\log^{5} T}{T^{2}} + \bigg(1+\frac{d^{3/2}\log^3 T}{T}\bigg)  \varepsilon_{\mathsf{score}}\sqrt{\log T}\bigg),
\end{align}
for some absolute constant $C > 0$.
\end{theorem}

In a word, Theorem~\ref{thm:main} demonstrates the sampling quality our proposed acceleration scheme, measured in the total variation distance, is determined by two principal factors. The first term in the error bound \eqref{eq:TV-main} reflects the discretization error incurred by approximating the continuous reverse process, while the second term arises from the estimation error of the score functions. When the number of iterations $T$ is sufficiently large, the score matching error term dominates and the error bound in \eqref{eq:TV-main} becomes $\wt O(\veps_\score)$. 

We now elaborate on the important implications of Theorem~\ref{thm:main}.
\begin{enumerate}
	\item \emph{Iteration complexity.} 
	Given a score estimation error $\veps_\score$, for any desired accuracy $\veps \geq \veps_\score\sqrt{\log T}$, the proposed sampler provably achieves $\TV(p_{X_1},p_{Y_1})\lesssim \veps$ in
	\begin{align*}
		\wt O\bigg(\frac{d^{5/4}}{\sqrt{\veps}}+d^{3/2}\bigg)
	\end{align*}
	iterations, where the $\wt O(d^{3/2})$ term can be interpreted as a burn-in cost.
	By contrast, the state-of-the-art iteration complexity for plain score-based samplers scales as $\wt O(d/\veps)$ \citep{li2024d,li2024sharp}. Therefore, our algorithm achieves a significant acceleration in the high-accuracy regime where $\veps \lesssim 1/\sqrt{d}$, specifically by a factor of 
	$$\wt O\bigg(\frac{d^{1/4}}{\sqrt{\veps}}\bigg).$$
	It is worth noting that existing score matching error guarantees for neural networks \citep{oko2023diffusion,cui2023analysis,cole2024score} and kernel-based estimators \citep{wibisono2024optimal,zhang2024minimax,dou2024optimal} satisfy $\veps_\score=o(1)$ as the training sample size grows. Hence, the high-accuracy regime $\veps_\score\sqrt{\log T} \leq \veps \lesssim 1/\sqrt{d}$ captures the most important and practically relevant setting.

	\item \emph{Mild target distribution assumption.} 
	The established convergence theory only imposes a finite second moment condition on the target data distribution. It does not rely on structural assumptions commonly required for acceleration sampling analysis in the literature, such as Lipschitz/smooth score functions \citep{huang2024convergence,huang2024reverse,li2024improved}. We achieve this by rigorously analyzing the probability flow ODE and demonstrating that it implicitly enforces these desirable properties, except on events of low probability. This theoretical framework ensures the broad applicability of the proposed sampler across a diverse family of data distributions.

	\item \emph{Minimal score estimation requirement.}
	The developed method provably achieves a faster convergence rate using only first-order score function estimates, illustrating that acceleration is feasible without smooth score estimates or higher-order score estimation guarantees (e.g., accurate estimates of the Jacobian matrices of score functions). 
	In essence, we show that the additional random noise in our sampler implicitly ``averages out'' second-order discretization errors, thus circumventing the higher-order score estimation requirements imposed by prior works \citep{li2024accelerating}.
	This reduced requirement for score matching confirms the practicality of the proposed acceleration scheme.

	\item \emph{Error metric.} Theorem~\ref{thm:main} establishes sampling error guarantees of $p_{Y_1}$ with respect to $p_{X_1}$ rather than the target distribution $p_{X_0}$, a metric widely used in the literature \citep{benton2023linear,li2024d}. The key observation is that ${X_1}={\sqrt{\alpha_1}X_0}+\sqrt{1-\alpha_1}Z_1$ deviates only slightly from ${X_0}$, provided that the learning rate $1-\alpha_1$ is vanishingly small. Consequently, $p_{X_1}$ and $p_{X_0}$ remain exceedingly close, making $\TV(p_{X_1},p_{Y_1})$ a valid and practical measure of error. It is worth noting that since $p_{X_1}$ is the convolution of $p_{X_0}$ with a Gaussian distribution, it inherits a smoothness property that allows our analysis and theory to accommodate target distributions whose  score functions may not be Lipschitz.

\end{enumerate}

Finally, we remark that recent works show that score-based diffusion models, when equipped with $L^2$-minimax optimal score estimators (with respect to $\veps_\score$), achieve the minimax rate of sampling for smooth densities \citep{oko2023diffusion,zhang2024minimax,dou2024optimal}. Our acceleration framework thus improves the computational efficiency of theses minimax optimal samplers, reducing the computational cost of attaining minimax optimal sampling performance.


\section{Analysis}
\label{sec:analysis}

In this section, we outline the proof strategy for Theorem~\ref{thm:main}. The detailed proofs are provided in the appendices.

\subsection{Preliminaries}

Before proceeding with the details, we first summarize several key properties that will be useful for the proof.

\paragraph*{Properties of the learning rates.}

To begin with, Lemma~\ref{lemma:step-size} collects several important properties of the learning rates $(\alpha_t)_{t\in[T]}$ chosen in \eqref{eq:learning-rate}.
The proof can be found in Appendix~\ref{sec:pf-lemma:step-size}.
\begin{lemma}
\label{lemma:step-size}
The learning rates $(\alpha_t)_{t\in[T]}$ specified in \eqref{eq:learning-rate} satisfy that for all $t=2,\dots,T$:
\begin{subequations}
	\begin{align}
	1-\alpha_t &\leq C_1 \frac{\log T}{T}\label{eq:alpha-t-lb}; \\
	\frac{1-\alpha_t}{1-\ol\alpha_t} &\leq C_1\frac{\log T}{T} \label{eq:1-alpha-1-olalpha}; \\
	\frac{1-\ol\alpha_t}{1-\ol\alpha_{t-1}} &\leq 1+ 2C_1\frac{\log T}{T} \label{eq:1-olalpha-O1}; 
	\end{align}
	where $C_0,C_1$ are defined in \eqref{eq:learning-rate}. In addition, $\alpha_1$ satisfies
	\begin{align}
		1-\alpha_1 & \leq \frac{1}{T^{C_1/4}}\label{eq:alpha-1-lb}.
	\end{align}
\end{subequations}
\end{lemma}


\paragraph*{Distance between $p_{X_T}$ and $p_{Y_T}$.}

Recall that $X_{T}\overset{\mathrm d}{=}\sqrt{\ol \alpha_T}X_{0} + \sqrt{1-\ol\alpha_T} Z$ with $Z\sim \cN(0,I_d)$ independent of $X_{0}$ and that $Y_{T}\sim \cN(0,I_d)$. 
It is well-known that the forward process mixes exponentially fast towards the standard normal distribution, as formalized by the following lemma.
The proof is deferred to Appendix~\ref{sub:proof_of_lemma_ref_lemma_mixing}.
\begin{lemma}\label{lemma:mixing}
Suppose Assumption~\ref{assu:distribution} holds. Then the KL divergence between $p_{X_{T}}$ and $p_{Y_{T}}$ satisfies
	\begin{align}
	\KL\big(p_{X_{T}}\,\|\,p_{Y_{T}}\big) & \lesssim T^{-10},
	\label{eq:KL-step-T}
\end{align}
as long as $C_0$ chosen in the learning rates schedule \eqref{eq:learning-rate} is large enough.
\end{lemma}

\paragraph{Preparation.}
Let $T_0 \geq 2$ be the largest integer such that $C_1 (d \log^{5/2} T)/T \geq 1/2$ where $C_1$ is defined in \eqref{eq:learning-rate}. If $T\leq T_0$, we have $(d^{5/2}\log^5 T)/T^2 > \max\{\sqrt{d}/(2C_1)^{2},d^{5/2}\log^5(2)/4\}$ and hence the result of Theorem~\ref{thm:main} naturally holds when the absolute constant $C$ is chosen large enough. 
Therefore, in the rest of the proof, we assume that $T > T_0$ so that $(C_1 d \log^2 T)/T \leq 1/2$.

We are now prepared to present the proof strategy for Theorem~\ref{thm:main}, organized into three main steps.
\subsection{Step 1: constructing auxiliary processes}
To begin with, we introduce several auxiliary processes which serve analytical purposes only and are not implemented in the practical sampling process. They allow us to compare our proposed sampler against idealized or modified processes in a controlled manner.
\paragraph{Sequences $(Y_{t}^{\star,\md})_{t=2}^T$ and $(Y_{t}^{\star})_{t=1}^{T-1}$ constructed using true scores.}
We first define two reverse processes---$(Y_{t}^{\star,\md})_{t=2}^T$ and $(Y^\star_t)_{t=1}^{T-1}$---that evolve backward in time using the forward process $(X_t)_{t=2}^T$ along with the true score functions $\{s_t^\star\}_{t=2}^T$.  These processes help us measure the purely second-order discretization error (up to some low probability event), independent of any score estimation error.
Concretely, for each $t=T,\cdots,2$, we define
\begin{subequations}
\label{eq:Y-star-defn}
\begin{align}
Y_{t}^{\star,\md} &\defn \frac{1}{\sqrt{\alpha_{t}}}\Big(X_{t} + \frac{1-\alpha_t}{2\alpha_t}s_t^\star(X_t)\Big), \label{eq:Y-star-defn-1/2} \\
Y_{t-1}^{\star} &\defn X_{t} + (1-\alpha_t)\Big(s_t^\star(X_t) 
+ \Big\{\alpha_{t}^{3/2}s_{t-1}^\star\big(Y_{t}^{\star,\md}\big) - s_t^\star(X_t)\Big\}\ind\{X_t\in\cE_t\}\Big) \label{eq:Y-star-defn-1}.
\end{align}
\end{subequations}
Here $\mathcal{E}_{t}$ is the set of ``typical'' points under the distribution $p_{X_t}$:
\begin{align}
\mathcal{E}_{t} &:= \Big\{ x\in\bR^d:-\log p_{X_{t}}(x)\leq C_2 d\log T\quad\text{and}\quad\|x\|_{2}\leq C_3 \Big(\sqrt{\overline{\alpha}_{t}}T^{2C_{R}}+\sqrt{(1-\overline{\alpha}_{t})d\log T}\Big)\Big\}\label{eq:def-E-t},
\end{align}
for some sufficiently large absolute constants $C_2,C_3>0$. In words, the set $\cE_t$ excludes points that are either extremely unlikely (having tiny density) or have exceedingly large norm.
Intuitively, the practical sampling process $\{(Y_{t}^{\md},Y_{t})\}_t$ approximates these auxiliary processes $\{(Y_{t}^{\star,\md},Y_{t}^{\star})\}_t$, provided that the score estimation error is small.

\paragraph{Sequence $(X^\aux_t)_{t=1}^T$ constructed using probability flow ODE.}
Next, we introduce an auxiliary forward process $(X^\aux_t)_{t=1}^T$ that evolves backward through a probability flow ODE and has the same marginal distributions as the original forward process $(X_t)_{t=1}^T$. 
This equivalence crucially allows us to replace $(X_t)$ with $(X_t^\aux)$ when bounding $\TV(p_{X_1},p_{Y_1})$, which is more tractable thanks to the special structure of $(X^\aux_t)$.

To motivate the construction, we first establish the following lemma, which demonstrates that for each $t>1$, $X_{t-1}$ in the forward process is distributed identically to a deterministic function of $X_t$ plus some independent Gaussian noise. The proof can be found in Appendix~\ref{sec:pf-lem:discretization}.
\begin{lemma} \label{lem:discretization}
For each $t=2,\dots,T$, there exists a mapping $\Phi_t(\cdot): \mathbb{R}^{d} \to \mathbb{R}^{d}$ satisfying
\begin{align}
\frac{1}{\sqrt{\alpha_t}}\big(\Phi_t(X_t) + \sigma_t Z_t\big) \overset{ \diff }{=} X_{t-1} \label{eq:phi=d=X-t-1},
\end{align}
where we recall $Z_t \sim \mathcal{N}(0, I_d)$ is independent of $X_t$ and $\sigma_t^2 = \alpha_t - 1/(3-2\alpha_t)$.
\end{lemma}

\begin{remark}
	As will become clear in the proof, the mapping $\Phi_t$ is constructed by leveraging the probability flow ODE \eqref{eq:prob-flow-ode} associated with the forward process $(X_t)$, effectively serving as a first-order discretization of the ODE. It is noteworthy that $\Phi_t$ is designed to ensure $\Phi_t(X_t)$ \emph{plus} noise recovers the distribution of $X_{t-1}$. This is in contrast to the fully deterministic strategies where $\Phi_t(X_t)$ alone is to designed to match $X_{t-1}$ in distribution, as used in prior acceleration work \citep{lu2022dpm,li2024accelerating}. This injection of additional random noise is key to reducing the worst-case second-order approximation errors, thereby yielding a faster convergence rate.
\end{remark}

With these mappings $\{\Phi_t\}_{t=2}^T$, we now define the auxiliary process $(X^\aux_t)_{t\in[T]}$ recursively:
\begin{align}
X^\aux_{T} = X_{T}, \quad \text{and}\quad X^\aux_{t-1} \defn \frac{1}{\sqrt{\alpha_t}}\big(\Phi_t(X^\aux_t) + \sigma_t Z_t\big),\quad t = 2,\dots,T. \label{eq:X_aux}
\end{align}
By Lemma~\ref{lem:discretization}, we have $X^\aux_t \distequal X_t$ for all $t\in[T]$. This {distributional equivalence} underpins the subsequent decomposition of $\mathsf{TV}(p_{X_{1}}, p_{Y_{1}})$.
%

\paragraph{Error decomposition.}

Combining the auxiliary process $(X^\aux_t)_{t\in[T]}$ with our proposed sampler $(Y_t)_{t\in[T]}$ allows us to bound the total variation distance between $p_{X_1}$ and $p_{Y_1}$ as follows:
\begin{align} 
\mathsf{TV}^{2}\big(p_{X_{1}}, p_{Y_{1}}\big)
& \numpf{i}{=} \mathsf{TV}^{2}\big(p_{X^\aux_{1}}, p_{Y_{1}}\big) 
\numpf{ii}{\leq}\frac{1}{2}\mathsf{KL}\big(p_{X^\aux_{1}}\,\|\,p_{Y_{1}}\big)
\numpf{iii}{\leq} \frac{1}{2}\mathsf{KL}\big(p_{X^\aux_{1},\ldots,X^\aux_{T}}\,\|\,p_{Y_{1},\ldots,Y_{T}}\big)\nonumber \\
 & \numpf{iv}{=}\frac{1}{2}\mathsf{KL}\big(p_{X^\aux_{T}}\,\|\,p_{Y_{T}}\big)+\frac{1}{2}\sum_{t=2}^{T}\mathbb{E}_{x_{t}\sim p_{X^\aux_{t}}}\Big[\mathsf{KL}\big(p_{X^\aux_{t-1}\mid X^\aux_{t}}(\,\cdot\mid x_{t})\,\|\,p_{Y_{t-1}|Y_{t}}(\,\cdot\mid x_{t})\big)\Big] \notag\\
 & \numpf{v}{\leq} O(T^{-10}) + \frac12\sum_{t=2}^{T}\mathbb{E}_{x_{t}\sim p_{X_{t}}}\Big[\mathsf{KL}\big(p_{X^\aux_{t-1}\mid X^\aux_{t}}(\,\cdot\mid x_{t})\,\|\,p_{Y_{t-1}|Y_{t}}(\,\cdot\mid x_{t})\big)\Big].
 \label{eq:error-decomposition}
\end{align}
Here (i) holds as $X^\aux_1 \distequal X_1$; (ii) applies Pinsker's inequality; (iii) arises from the data-processing inequality; (iv) results from the chain rule of the KL divergence; (v) applies \eqref{eq:KL-step-T} in Lemma~\ref{lemma:mixing} and $X^\aux_t \distequal X_t$ for all $t$. 

Therefore, bounding $\mathsf{TV}(p_{X_{1}}, p_{Y_{1}})$ reduces to controlling
\begin{align*}
	\mathbb{E}_{x_{t}\sim p_{X_{t}}}\Big[\mathsf{KL}\big(p_{X^\aux_{t-1}\mid X^\aux_{t}}(\,\cdot\mid x_{t})\,\|\,p_{Y_{t-1}\mid Y_{t}}(\,\cdot\mid x_{t})\big)\Big]
\end{align*} for each $t=2,\dots,T$. 
As we will see in Step 2, the explicit dependence of $X^\aux_{t-1}$ on $X^\aux_{t}$ in the auxiliary process allows us to show that the above KL divergence---and thus $\mathsf{TV}(p_{X_{1}}, p_{Y_{1}})$---is governed by the second-order approximation error of our sampler.

\subsection{Step 2: controlling KL divergence between $p_{X^\aux_{t-1}\mid X^\aux_{t}}$ and $p_{Y_{t-1}\mid Y_{t}}$}

In this section, let us fix an arbitrary $2\leq t\leq T$ and focus on $\mathbb{E}_{x_{t}\sim p_{X_{t}}}\big[\mathsf{KL}\big(p_{X^\aux_{t-1}\mid X^\aux_{t}}(\,\cdot\mid x_{t})\,\|\,p_{Y_{t-1}\mid Y_{t}}(\,\cdot\mid x_{t})\big)\big]$. 

\paragraph{Step 2.1: relating KL divergence to second-order approximation error.}
Recall the definition of $\Phi(\cdot)$ in \eqref{eq:phi=d=X-t-1}, which can be interpreted as a first-order discretization of the probability ODE \eqref{eq:prob-flow-ode}. We now define another mapping $\Psi_t(\cdot):\bR^d \rightarrow \bR^d$ by
\begin{align}
\label{eq:psi-defn}
	\Psi_t(x) \defn \frac{1}{1-\alpha_t} \big(\Phi_t(x)-x-(1-\alpha_t)s^\star_t(x)\big).
\end{align}
To provide some intuition, $\Psi_t$ represents a second-order correction term. Our proposed sampler is designed to construct an approximation of $\Psi_t$ in order to to reduce the discretization error and, consequently, accelerate sampling. 

In what follows, we will relate $\mathsf{KL}\big(p_{X^\aux_{t-1}\mid X^\aux_{t}}(\,\cdot\mid x_{t})\,\|\,p_{Y_{t-1}\mid Y_{t}}(\,\cdot\mid x_{t})\big)$ to how well our sampler's update matches this ``ideal'' update given by $\Psi_t$.
Specifically, for any $x_t\in\bR^d$, by the data processing inequality and the independence between $Z^\md_t $ and $(X^\aux_{t},Y_{t})$, we can derive
\begin{align*}
&\mathsf{KL}\Big(p_{X^\aux_{t-1}\mid X^\aux_{t}}(\,\cdot\mid x_{t})\,\Vert\,p_{Y_{t-1}\mid Y_{t}}(\,\cdot\mid x_{t})\Big) \leq  \mathsf{KL}\Big(p_{X^\aux_{t-1},Z^\md_t \mid X^\aux_{t}}(\,\cdot\mid x_{t})\,\Vert\,p_{Y_{t-1},Z^\md_t \mid Y_{t}}(\,\cdot\mid x_{t})\Big)\\
&\qquad= \int_{z_t^\md}p_{Z^\md_t}(z_t^\md)\int_{x_{t-1}} p_{X^\aux_{t-1}\mid X^\aux_{t},Z^\md_t}(x_{t-1}\mid x_{t},z_t^\md)\log \frac{p_{X^\aux_{t-1}\mid X^\aux_{t},Z^\md_t}(x_{t-1}\mid x_{t},z_t^\md)}{p_{Y_{t-1}\mid Y_{t},Z^\md_t}(x_{t-1}\mid x_{t},z_t^\md)} \diff x_{t-1} \diff z_t^\md \\
&\qquad= \mathbb{E}_{z^\md_t  \sim p_{Z^\md_t}}\Big[\mathsf{KL}\Big(p_{X^\aux_{t-1}\mid X^\aux_{t},Z^\md_t }(\,\cdot\mid x_{t},z^\md_t \big)\,\Vert\,p_{Y_{t-1}|Y_{t},Z^\md_t }\big(\,\cdot\mid x_{t},z^\md_t \big)\Big)\Big].
\end{align*}
Next, recall the definitions of $Y_{t-1}$ and $X^\aux_{t-1}$ in \eqref{eq:sampler} and \eqref{eq:X_aux}, respectively. By construction, $X^\aux_{t-1}$ (resp.~$Y_{t-1}$) is normally distributed conditioned on $X^\aux_t=x_t$ and $Z^\md_t=z^\md_t$ (resp.~$Y_t=x_t$ and $Z^\md_t=z^\md_t$), where the randomness arises from $Z_t$. Exploiting their expressions and invoking the KL divergence formula for normal distributions, we can continue the above derivation as
\begin{align*}
&\mathbb{E}_{z^\md_t  \sim p_{Z^\md_t}}\Big[\mathsf{KL}\Big(p_{X^\aux_{t-1}\mid X^\aux_{t},Z^\md_t }(\,\cdot\mid x_{t},z^\md_t \big)\,\Vert\,p_{Y_{t-1}|Y_{t},Z^\md_t }\big(\,\cdot\mid x_{t},z^\md_t \big)\Big)\Big] \\
&\qquad=\frac{1}{2\sigma_t^2}
\mathbb{E}\bigg[\Big\| \Phi_t(X_t) - X_t - (1-\alpha_t)s_t(X_t)    \nonumber \\
& \qquad\qquad\qquad\quad - (1-\alpha_t)  \clip_t\Big\{\alpha_{t}^{3/2}s_{t-1}\big(Y_{t}^{\md}(X_t)\big) - s_t\big(X_t + (1-\alpha_t)Z^\md_t \big)\Big\} \Big\|_2^2  \,\Big|\,X_t = x_{t}\bigg],\\
&\qquad\numpf{i}{=}\frac{(1-\alpha_t)^2}{2\sigma_t^2}
\mathbb{E}\bigg[\Big\| \Psi_t(X_t) - \clip_t\Big\{\alpha_{t}^{3/2}s_{t-1}\big(Y_{t}^{\md}(X_t)\big) - s_t\big(X_t + (1-\alpha_t)Z^\md_t \big)\Big\} \nonumber \\
& \qquad \qquad \qquad \qquad \qquad  + s_t^\star(X_t) - s_t(X_t)\Big\|_2^2 \,\Big|\, X_t = x_{t}\bigg], \nonumber \\
& \qquad \numpf{ii}{\lesssim} (1-\alpha_t) \mathbb{E}\bigg[\Big\| \clip_t\Big\{\alpha_{t}^{3/2}s_{t-1}\big(Y_{t}^{\md}(X_t)\big) - s_t\big(X_t + (1-\alpha_t)Z^\md_t \big)\Big\}-\Psi_t(X_t) \Big\|_2^2 \,\Big|\, X_t = x_t\bigg] \nonumber \\
& \qquad \quad  + (1-\alpha_t) \big\| s_t^\star(x_t) - s_t(x_t)\big\|_2^2,
\end{align*}
where (i) plugs in the expression of $\Psi_t(\cdot)$ in \eqref{eq:psi-defn}; (ii) holds because of $1-\alpha_t \lesssim \log T/T=o(1)$ in \eqref{eq:alpha-t-lb} from Lemma~\ref{lemma:step-size} and thus $\sigma_t^2 =(1-\alpha_t) + {4(1-\alpha_t)^2}/{(3-2\alpha_t)} = \big(1+o(1)\big) (1-\alpha_t)$.
As a remark, the notation $Y_{t}^{\md}(X_t)$ emphasizes that $Y_{t}^{\md}(X_t)$ is computed based on $X_t$ according to \eqref{eq:sampler-mid} here.

Taking the expectation over $x_t\sim p_{X_t}$, we arrive at the key inequality
\begin{align}
&\frac{1}{1-\alpha_t}\mathbb{E}_{x_{t}\sim p_{X_{t}}}\big[\mathsf{KL}\big(p_{X^\aux_{t-1}\mid X^\aux_{t}}(\,\cdot\mid x_{t})\,\|\,p_{Y_{t-1}\mid Y_{t}}(\,\cdot\mid x_{t})\big)\big] \nonumber  \\
& \qquad\lesssim
\mathbb{E}\bigg[\Big\| \clip_t\Big\{\alpha_{t}^{3/2}s_{t-1}\big(Y_{t}^{\md}(X_t)\big) - s_t\big(X_t + (1-\alpha_t)Z^\md_t \big)\Big\}-\Psi_t(X_t) \Big\|_2^2 \bigg]   + \bE\Big[ \big\|s_t^\star(X_t) - s_t(X_t)\big\|_2^2 \Big] \nonumber \\
& \qquad\leq
\mathbb{E}\bigg[\Big\| \clip_t\Big\{\alpha_{t}^{3/2}s_{t-1}\big(Y_{t}^{\md}(X_t)\big) - s_t\big(X_t + (1-\alpha_t)Z^\md_t \big)\Big\}-\Psi_t(X_t) \Big\|_2^2 \bigg] +  \veps_t^2 \label{eq:KL-temp},
\end{align}
where the last step arises from Assumption~\ref{assu:score-error} on score estimation.


This result reveals that the KL divergence between the ideal backward distribution $p_{X^\aux_{t-1}\mid X^\aux_{t}}$ and our sample's backward distribution $p_{Y_{t-1}\mid Y_{t}}$ is governed by two factors: the discrepancy between the true second-order term encoded by $\Psi_t$ and those realized by our algorithm, and the error in score matching.
In what follows, we will focus on controlling the second-order approximation error term in \eqref{eq:KL-temp}.

\paragraph{Step 2.2: bounding second-order approximation error based on typical points.}
Recalling the typical set $\cE_t$ defined in \eqref{eq:def-E-t}, we denote the event
\begin{align}
 \cA_t \defn \big\{ X_t \in \cE_t\big\}.
\end{align}
Let us first examine the behavior of the second-order correction term $\Psi_t(X_t)$ when $X_t$ lies in this typical set $\cE_t$, which is formalized by the following lemma.
Here, we remind the reader of the notation that for any event $\cE$, we denote $\bE_{\cE}[\cdot] \defn \bE\big[\cdot \ind\{\cE\}\big]$. The proof is provided in Appendix~\ref{sec:pf-lem:outlier}.
\begin{lemma} \label{lem:outlier}
The mapping $\Psi_t(\cdot)$ defined in \eqref{eq:psi-defn} satisfies
\begin{align}
\mathbb{E}_{\cA_t}\bigg[\Big\|\alpha_{t}^{3/2}s_{t-1}^\star\big(Y_{t}^{\star,\md}\big) - s_t^\star(X_t) - \Psi_t(X_t) \Big\|_2^2 \bigg] \lesssim (1-\alpha_t)^4\bigg(\frac{d}{1-\overline{\alpha}_t}\bigg)^5. \label{eq:psi-Y-star-dist}
\end{align}
Moreover, one has
\begin{align}
	\cA_t \subset \bigg\{\big\|\Psi_t(X_t)\big\|_2 \leq C_4 (1-\alpha_t)\bigg(\frac{d\log T}{1-\overline{\alpha}_t}\bigg)^{3/2}\bigg\} , \label{eq:E-subset-Psi}
\end{align}
for some absolute constant $C_4 > 0$, and
\begin{align}
\mathbb{E}_{\cA_t^\setc}\bigg[\big\|\Psi_t(X_t)\big\|_2^2 + (1-\alpha_t)^2\bigg(\frac{d\log T}{1-\overline{\alpha}_t}\bigg)^{3}\bigg] 
\lesssim \frac{(1-\alpha_t)^2}{T^{10}} \bigg(\frac{d}{1-\ol{\alpha}_t}\bigg)^3. \label{eq:psi-E-c-ub}
\end{align}
\end{lemma}
\begin{remark}
	Through a careful analysis, we are able to characterize the intrinsic higher-order smoothness properties of true scores in establishing \eqref{eq:KL-temp}. This allows us to pin down the second-order approximation error of the proposed sampler without imposing additional higher-order smooth assumptions on the true scores or score estimators required by prior works \citep{li2024accelerating,huang2024convergence,huang2024reverse,li2024improved}.
\end{remark}

In a word, Lemma~\ref{lem:outlier} tells us that on the event $\cA_t$, the second-order correction term $\Psi_t(X_t)$ is close to the true score functions evaluated at the auxiliary process $Y_{t}^{\star,\md}$ constructed in \eqref{eq:Y-star-defn} in Step 1 and the forward process $X_t$. In addition, the $\ell_2$ norm $\|\Psi_t(X_t)\|_2$ is bounded above on $\cA_t$, which directly motivates our choice of the thresholding function $\clip_t(\cdot)$ in \eqref{def:clip}. Off the typical set, the probability of those points is small and $\|\Psi_t(X_t)\|_2$ remains reasonably bounded.
By comparing the terms in \eqref{eq:KL-temp} with the left-hand side of \eqref{eq:psi-Y-star-dist}, one can expect that the right-hand side of \eqref{eq:psi-Y-star-dist} serves as the dominant term in the upper bound for the KL divergence \eqref{eq:KL-temp}.

To proceed, we must address a subtlety: the expression \eqref{eq:KL-temp} involves the score estimators evaluated at random variables $Y_{t}^{\md}(X_t)$ and $X_t + (1-\alpha_t)Z^\md_t$. However, the assumption on score matching (Assumption~\ref{assu:score-error}) only guarantees the estimation error of the score estimators when applied to the forward process $(X_t)$. To address this issue, we can apply a change of measure. Specifically, we define the sets $\cF_{t,1},\cF_{t,2},\cF_{t,3}$:
\begin{subequations}\label{def:set-F}
\begin{align}
\cF_{t,1} &\defn \bigg\{x\in\bR^d\colon  p_{X_t + (1-\alpha_t)Z^\md_t }(x) \leq 2\,p_{X_{t}}(x) \bigg\};\\
\cF_{t,2} &\defn \bigg\{x\in\bR^d \colon p_{Y^\md_t(X_t)}(x) \leq 4\,p_{X_{t-1}}(x) \bigg\};\\
\cF_{t,3} &\defn \bigg\{x\in\bR^d \colon p_{Y_{t}^{\star,\md} + (1-\alpha_t)Z^\md_t }(x) \leq 2\,p_{X_{t-1}}(x)\bigg\},
\end{align}
\end{subequations}
and then define the event $\cB_t$:
\begin{align}
\label{eq:event-pdf-ratio}
 \cB_t \defn \Big\{ X_t + (1-\alpha_t)Z^\md_t\in\cF_{t,1},\,Y_{t}^{\md}(X_t)\in\cF_{t,2},\,Y_{t}^{\star,\md} + (1-\alpha_t)Z^\md_t \in\cF_{t,3} \Big\}.
\end{align}
In other words, these sets contain points where the densities of the auxiliary random variables $X_t + (1-\alpha_t)Z^\md_t$, $Y^\md_t(X_t)$, and $Y_{t}^{\star,\md} + (1-\alpha_t)Z^\md_t$ are upper bounded by a constant factor relative to those of the forward process $(X_t)$. Hence, on the event $\cB_t$, we can use the score matching assumption effectively by performing a change of measure.

With the above preparation in place, we now decompose the first term of \eqref{eq:KL-temp} into contributions from ``typical'' points (those in $\cA_t\cap\cB_t$) versus the complementary set, using the bound of $\|\Psi_t(X_t)\|_2$ on $\cA_t$. Concretely, we derive:
\begin{align}
&\mathbb{E}\bigg[\Big\|\clip_t\Big\{\alpha_{t}^{3/2}s_{t-1}\big(Y_{t}^{\md}\big) - s_t\big(X_t + (1-\alpha_t)Z^\md_t \big)\Big\}-\Psi_t(X_t)\Big\|_2^2\bigg] \nonumber \\
& \qquad \numpf{i}{\lesssim} \mathbb{E}_{\cA_t\cap\cB_t}\bigg[\Big\|\alpha_{t}^{3/2}s^\star_{t-1}\big(Y_{t}^{\md}\big) - s^\star_t\big(X_t + (1-\alpha_t)Z^\md_t \big) - \Psi_t(X_t)\Big\|_2^2\bigg] \notag\\
& \qquad \quad +\mathbb{E}_{\cA_t\cap\cB_t}\bigg[\Big\|\alpha_{t}^{3/2}s_{t-1}\big(Y_{t}^{\md}\big)  - \alpha_{t}^{3/2}s_{t-1}^\star\big(Y_{t}^{\star,\md}\big)- s_t\big(X_t + (1-\alpha_t)Z^\md_t \big) + s_t^\star(X_t) \Big\|_2^2 \bigg] \notag\\
& \qquad\quad + \mathbb{P}\big(\cA_t\cap\cB_t^\setc\big)(1-\alpha_t)^2\bigg(\frac{d\log T}{1-\overline{\alpha}_t}\bigg)^{3}
+ \mathbb{E}_{\cA_t^\setc}\bigg[\big\|\Psi_t(X_t)\big\|_2^2+(1-\alpha_t)^2\bigg(\frac{d\log T}{1-\ol{\alpha}_t}\bigg)^3 \bigg] \notag \\ 
& \qquad \numpf{ii}{\lesssim} \mathbb{E}_{\cA_t\cap\cB_t}\bigg[\Big\|\alpha_{t}^{3/2}s_{t-1}\big(Y_{t}^{\md}\big)  - \alpha_{t}^{3/2}s_{t-1}^\star\big(Y_{t}^{\star,\md}\big)- s_t\big(X_t + (1-\alpha_t)Z^\md_t \big) + s_t^\star(X_t) \Big\|_2^2 \bigg] \notag\\
& \qquad\quad + \mathbb{P}\big(\cB_t^\setc\big)(1-\alpha_t)^2\bigg(\frac{d\log T}{1-\overline{\alpha}_t}\bigg)^{3}
+ (1-\alpha_t)^4\bigg(\frac{d}{1-\overline{\alpha}_t}\bigg)^5 + \frac{(1-\alpha_t)^2}{T^{10}} \bigg(\frac{d}{1-\ol{\alpha}_t}\bigg)^3,\label{eq:clip-Psi-dist}
\end{align}
where (i) uses the bound of $\|\Psi_t(X_t)\|_2$ on the event $\cA_t$ in \eqref{eq:E-subset-Psi} and our choice of the thresholding function \eqref{def:clip} ensuring $\sup_{x\in\bR^d}\big\| \clip_t\{x\}\big\|_2 \lesssim (1-\alpha_t)\big(\frac{d\log T}{1-\overline{\alpha}_t}\big)^{3/2}$; (ii) invokes \eqref{eq:psi-Y-star-dist} and \eqref{eq:psi-E-c-ub} in Lemma~\ref{lem:outlier}.

Therefore, it remains to control the first two quantities in the above decomposition. To this end, the second term in \eqref{eq:clip-Psi-dist} is controlled by Lemma~\ref{lem:TV-crude} below. The proof is deferred to Appendix~\ref{sec:proof-lem:TV-crude}.
	\begin{lemma} \label{lem:TV-crude}
	Under Assumption~\ref{assu:score-error} on score matching, we can bound
	\begin{align}
	\mathbb{P}\big(\cB_t^{\mathrm{c}}\big) &\lesssim \bigg(\frac{(1-\alpha_t)d\log T}{1-\overline{\alpha}_t}\bigg)^2 + \frac1{T^2}+\veps_t^2\log T. \label{eq:clip-Psi-dist-term2}
	\end{align}
	\end{lemma}
Additionally, the first quantity in \eqref{eq:clip-Psi-dist} is controlled in Lemma~\ref{lemma:2-order-approx-error} below. The proof is provided in Appendix~\ref{sub:proof_of_lemma_ref_lemma_2_order_approx_error}.
\begin{lemma}\label{lemma:2-order-approx-error}
Under Assumption~\ref{assu:score-error} on score matching, one has
\begin{align}
& \mathbb{E}_{\cA_t\cap\cB_t}\bigg[\Big\|\alpha_{t}^{3/2}s_{t-1}\big(Y_{t}^{\md}\big)  - \alpha_{t}^{3/2}s_{t-1}^\star\big(Y_{t}^{\star,\md}\big)- s_t\big(X_t + (1-\alpha_t)Z^\md_t \big) + s_t^\star(X_t) \Big\|_2^2 \bigg] \notag\\
&\qquad \lesssim \frac{1}{1-\ol\alpha_t} \bigg(\frac{1-\alpha_t}{1-\overline{\alpha}_t}\bigg)^4(d\log T)^5
+ \varepsilon_{t-1}^2 +  \varepsilon_t^2 +\frac{1}{T^{10}}\frac{d}{1-\ol\alpha_{t}},
 \label{eq:clip-Psi-dist-term3}
\end{align}
\end{lemma}

Taking \eqref{eq:clip-Psi-dist-term2} and \eqref{eq:clip-Psi-dist-term3} collectively with the previous decomposition \eqref{eq:clip-Psi-dist}, straightforward calculation yields the following  bound for the second-order approximation error in \eqref{eq:KL-temp}:
\begin{align}
&(1-\alpha_t) \mathbb{E} \bigg[\Big\|\clip_t\Big\{\alpha_{t}^{3/2}s_{t-1}\big(Y_{t}^{\md}(X_t)\big) - s_t\big(X_t + (1-\alpha_t)Z^\md_t \big)\Big\}
- \Psi_t(X_t)\Big\|_2^2\bigg] \nonumber \notag\\
	&\qquad \lesssim \bigg(\frac{(1-\alpha_t)d\log T}{1-\overline{\alpha}_t}\bigg)^5
+ (1-\alpha_t)\big(\veps_{t-1}^2 + \veps_t^2\big) +\frac{1}{T^{10}}\frac{(1-\alpha_t)d}{1-\ol\alpha_{t}} \nonumber \\
& \qquad \quad + \frac1{T^2}\bigg(\frac{(1-\alpha_t)d\log T}{1-\overline{\alpha}_t}\bigg)^{3}+\veps_t^2\log T\bigg(\frac{(1-\alpha_t)d\log T}{1-\overline{\alpha}_t}\bigg)^{3} + \frac{1}{T^{10}} \bigg(\frac{(1-\alpha)d}{1-\ol{\alpha}_t}\bigg)^3 \nonumber \\
	& \qquad \lesssim \bigg(\frac{d\log^{2} T}{T}\bigg)^5 + \bigg(\frac{d\log^2 T}{T}\bigg)^3 \veps_t^2\log T + \frac{\log T}{T}\big(\veps_{t-1}^2 + \veps_t^2\big) \label{eq:2nd-order-approx}
\end{align}
where 
the last line follows from $1-\alpha_t\leq (1-\alpha_t)/(1-\ol\alpha_t) \lesssim \log T/T$ from \eqref{eq:alpha-t-lb}--\eqref{eq:1-alpha-1-olalpha} in Lemma~\ref{lemma:step-size} and $d \log^2 T = o(T)$.

\paragraph{Step 2.3: obtaining bounds for KL divergence.}
Plugging the second-order approximation error bound \eqref{eq:2nd-order-approx} into \eqref{eq:KL-temp} allows us to control the KL divergence as
\begin{align}
	&\bE_{x_t\sim p_{X_t}} \Big[\mathsf{KL}\big(p_{X^\aux_{t-1}\mid X^\aux_{t}}(\,\cdot\mid x_{t})\,\|\,p_{Y_{t-1}\mid Y_{t}}(\,\cdot\mid x_{t})\big)\Big] \notag\\
	& \qquad \lesssim \bigg(\frac{d\log^{2} T}{T}\bigg)^5 + \bigg(\frac{d\log^2 T}{T}\bigg)^3 \veps_t^2\log T + \frac{\log T}{T}\big(\veps_{t-1}^2 + \veps_t^2\big)  + (1-\alpha_t)\veps_t^2 \notag \\
	& \qquad \lesssim \bigg(\frac{d\log^{2} T}{T}\bigg)^5 + \bigg(\frac{d\log^2 T}{T}\bigg)^3 \veps_t^2\log T + \frac{\log T}{T}\big(\veps_{t-1}^2 + \veps_t^2\big) \label{eq:cond-KL-ub},
\end{align}
where the last holds as $1-\alpha_t \lesssim \log T/T$ from \eqref{eq:alpha-t-lb} in Lemma~\ref{lemma:step-size}.
\subsection{Step 3: putting all pieces together}
To finish up, substituting the KL divergence bounds \eqref{eq:KL-step-T} and \eqref{eq:cond-KL-ub} into the total variation bound \eqref{eq:error-decomposition}, we arrive at
\begin{align*}
	\TV(p_{X_1},\,p_{Y_1}) &\lesssim \frac{1}{T^{10}} +  \frac{d^{5/2}\log^5 T}{T^2} + \frac{d^{3/2} \log^{3} T}{T} \sqrt{\frac{\log T}{T}\sum_{t=1}^T \veps_t^2} + \sqrt{\frac{\log T}{T}\sum_{t=1}^T \veps^2_t} \\
	& \asymp \frac{d^{5/2}\log^5 T}{T^2} + \bigg(1+\frac{d^{3/2} \log^3 T}{T}\bigg) \veps_\score \sqrt{\log T},
\end{align*}
where the last step arises from the definition of $\veps_\score$ in \eqref{eq:score-error}.
This completes the proof of Theorem~\ref{thm:main}.	

\section{Discussion}
\label{sec:discussion}

This work uncovers the feasibility of provable acceleration of score-based samplers under minimal assumptions: namely, $L^2$-accurate score estimates and a finite second moment of the target distribution. We have proposed a training-free accelerated sampler that attains an iteration complexity of $\wt O(d^{5/4}/\sqrt{\veps})$, establishing a theoretical foundation for efficient sampling speedups.

Several important directions remain for future investigation. One pressing issue is to sharpen our convergence theory with respect to the data dimension $d$ through a more refined analysis.
Moving beyond the second-order approximation developed in this work, another promising direction is to explore whether higher-order ODE approximations can yield improved iteration complexity.
Finally, for target distributions exhibiting low-dimensional structures, developing specialized acceleration schemes that exploit these intrinsic properties presents another important avenue for sampling efficiency.

\section*{Acknowledgements}

Gen Li is supported in part by the Chinese University of Hong Kong
Direct Grant for Research.

\bibliographystyle{apalike}
\bibliography{bibfileDF}

\appendix

\section{Numerical experiments}
\label{sec:numerical}

In this section, we present numerical evidence to support our theoretical claims. Due to our theoretical focus, we use a toy example to validate the claimed accelerated convergence rate.

In order to quantify the distance between the target distribution and the distribution generated by our sampler, we choose the Gaussian distribution as the target, which admits a closed form expression of the KL divergence.
Since our primary objective is to evaluate the convergence rate of the proposed acceleration scheme, we implement the sampler using the true scores.

Specifically, we consider a $d$-dimensional Gaussian distribution with independent coordinates. The last $d-k$ coordinates are fixed at zero, while for each $i\in[k]$, the $i$-th coordinate is distributed according to $\cN(0,\sigma_i^2)$ with $\sigma_i^2\sim\mathsf{Unif}[0,10]$. The learning rate schedule is set according to \eqref{eq:learning-rate}.

Figure~\ref{fig:KL} displays the KL divergence---our measure of sampling error---versus the total iterations $T$ for various parameter settings. For comparison, we also plot the performance the original DDPM sampling algorithm \citep{ho2020denoising}. Moreover, we include a reference line corresponding to the convergence rate $\Theta(\mathsf{poly}\log T /T^4)$, where we focus on the $T$-dependence and the $T^{-4}$ scaling is guided by Theorem~\ref{thm:main}.

The results clearly show that our proposed acceleration scheme achieves a substantially faster convergence rate than the original DDPM. In particular, the observed convergence rate aligns closely with the reference line if we set $\mathsf{poly}\log T$ to $\log^4 T$. This validates our theory that the sampling error, when measured in the KL divergence, scale as $T^{-4}$ up to logarithmic factors (since the slope in the figure is roughly $-4$), which in turn implies a total variation distance $\wt O(T^{-2})$. 

Finally, we note that our current convergence guarantee exhibits $\log^{10} T$ dependency, which is higher than what we observe in practice. This discrepancy indicates that the logarithmic dependency in our bounds is likely to be suboptimal. To sharpen the log factors needs much more efforts, which will be left for future work.

\begin{figure}[t]
\centering
\begin{tabular}{ccc}
\includegraphics[width=0.33\textwidth]{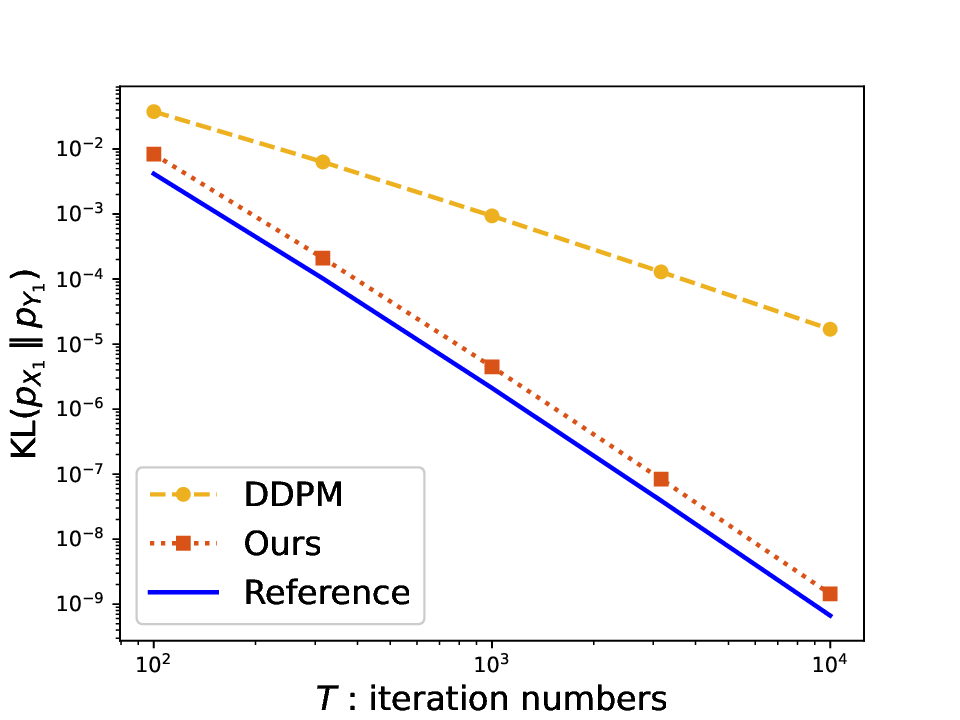} 
& \includegraphics[width=0.33\textwidth]{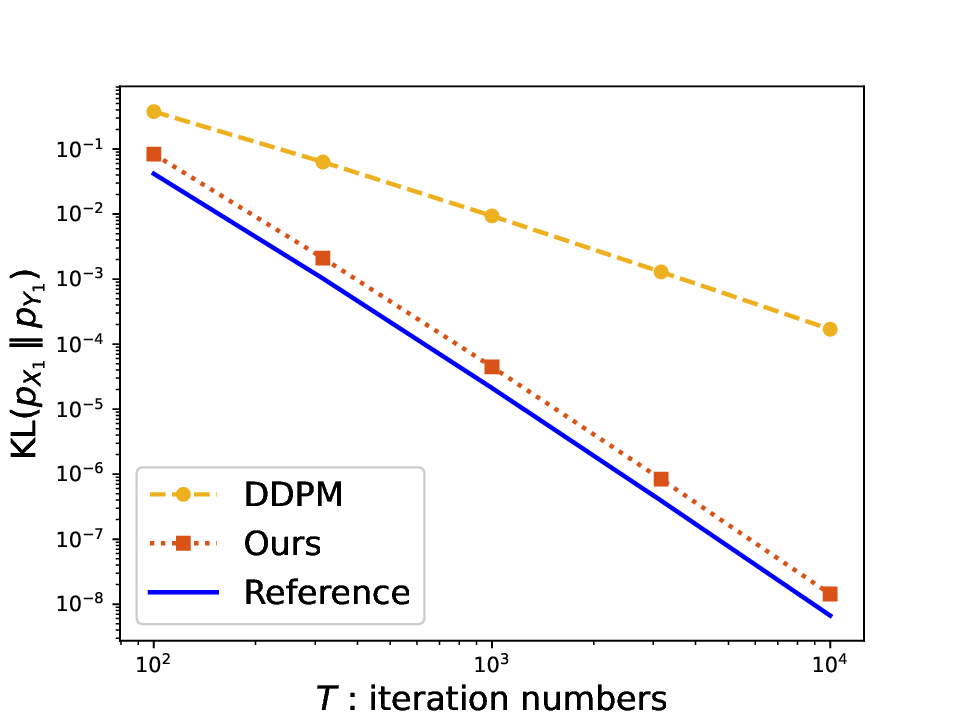}
& \includegraphics[width=0.33\textwidth]{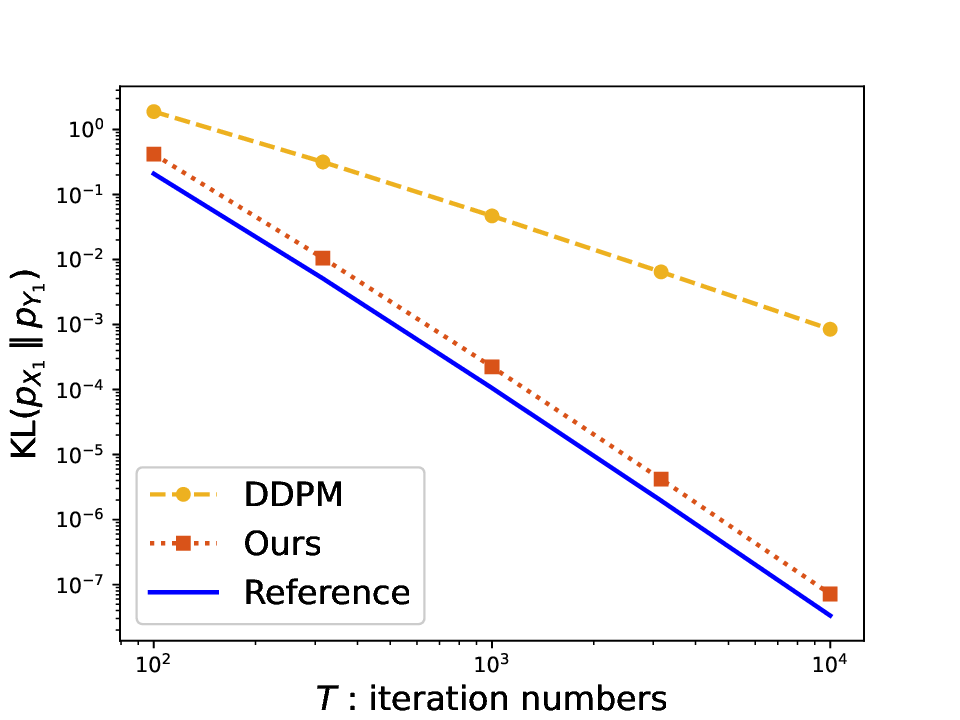}\tabularnewline
(a) & (b)& (c)\tabularnewline
\end{tabular}
\caption{Sampling error of the proposed acceleration method, DDPM, and fitted rate $T\mapsto\Theta(\log^4 T /T^4)$: (a) $k=10,d=10$; (b) $k=10,d=100$; (c) $k=100,d=500$.\label{fig:KL}}
\end{figure}
\section{Proof of lemmas for Theorem~\ref{thm:main}}

\subsection{Proof of Lemma~\ref{lem:discretization}}
\label{sec:pf-lem:discretization}

We begin by introducing the following SDE:
\begin{align}
	\mathrm d \wt{X}_\tau = - \frac{1}{2(1-\tau)} \wt X_\tau \diff \tau + \frac{1}{\sqrt{1-\tau}} \diff B_\tau,\quad \wt X_0 \sim p_{\mathsf{data}}; \quad \tau\in[0,1).
\end{align}
It is straightforward to verify that the solution to this SDE satisfies
\begin{align}\label{eq:def-Xtau}
\wt{X}_\tau \overset{\mathrm d}{=} \sqrt{1-\tau}X_0 + \sqrt{\tau}Z,
\end{align}
where $Z\sim \mathcal{N}(0, I_d)$ is a standard Gaussian random vector in $\bR^d$ independent of $X_0$. In particular, we know that $X_t \distequal \wt{X}_{1-\ol\alpha_t}$ for all $t\in[T]$ by \eqref{eq:forward-dist}.

For any $\tau\in(0,1)$, we denote by $\tilde{s}^\star(\cdot,\tau):\bR^d \rightarrow \bR^d$ the score function of $p_{\wt{X}_\tau}$, i.e.
\begin{align}
\tilde{s}^\star(x,\tau) &:= \nabla \log p_{\wt{X}_\tau}(x) = - \frac{1}{\tau} \bE \big[\wt{X}_\tau - \sqrt{1-\tau} \wt X_0 \mid \wt{X}_\tau = x\big],\quad \forall x\in\bR^d. \label{def-score-cont}
\end{align}
where the last expression can be derived by standard properties of Gaussian random vectors.

As discussed in \eqref{eq:forward-SDE} and \eqref{eq:prob-flow-ode} from Section~\ref{sec:prelim}, when initialized at $x_{\tau_0}^\star\sim p_{\wt X_{\tau_0}}$ for any $\tau_0\in(0,1)$, the process $(x_\tau^\star)_{\tau\in(\tau_0,1)}$ that solves the following probability flow ODE, 
\begin{align}
\label{eq:ODE-star}
\frac{\diff}{\diff \tau} \frac{x_{\tau}^{\star}}{\sqrt{1-\tau}} = -\frac{\tilde{s}^\star(x^\star_\tau,\tau)}{2(1-\tau)^{3/2}},
\end{align}
has the same marginal distributions as $(\wt X_\tau)_{\tau\in(\tau_0,1)}$. This further implies that  for all $t\in[T]$: $$X_t \distequal \wt{X}_{1-\ol\alpha_t} \distequal x_{1-\ol\alpha_t}^\star.$$

Finally, let us introduce the function $\theta^\star(\cdot):(0,1)\rightarrow \bR^d$ by
\begin{align}
	\label{eq:s-star-defn}
	 \theta^\star(\tau) & \defn \tilde{s}^\star(x^\star_\tau,\tau) = - \frac{1}{\tau} \bE \big[\wt{X}_\tau - \sqrt{1-\tau} \wt X_0 \mid \wt{X}_\tau = x_\tau^\star\big].
\end{align}

Equipped with the notations, let us explicitly construct the map $\Phi_t(\cdot):\bR^d \rightarrow \bR^d$.
Fix an arbitrary $t=2,\dots,T$. Let us denote 
\begin{align}
	\tau_t \defn 1-\overline{\alpha}_t \quad \text{and} \quad \tau_{t-1} \defn 1-\overline{\alpha}_t(3-2\alpha_t). \label{eq:tau-t-defn}
\end{align}
Given an arbitrary $x\in\bR^d$, the ODE \eqref{eq:ODE-star} allows us to define the function $\varphi^\star_t(\cdot):\bR^d \rightarrow \bR^d$ by
\begin{align}
\label{eq:x-star-tau-integral}
\frac{\varphi^\star_t(x)}{\sqrt{1-\tau_{t-1}}}
\defn \frac{x}{\sqrt{1-\tau_{t}}} -  \frac12 \int_{\tau_t}^{\tau_{t-1}}\frac{\theta^\star(\tau)}{(1-\tau)^{3/2}} \diff \tau.
\end{align}
We can subsequently define $\Phi_t(\cdot)$ as
\begin{align}
\label{eq:phi-construction}
\Phi_t(x) \defn \sqrt{\frac{1-\tau_t}{1-\tau_{t-1}}} \varphi^\star_t(x) = x-  \frac12 \sqrt{1-\tau_t}\int_{\tau_t}^{\tau_{t-1}}\frac{\theta^\star(\tau)}{(1-\tau)^{3/2}} \diff \tau, \quad \forall x\in\bR^d.
\end{align}

Note that the probability flow ODE \eqref{eq:ODE-star} ensures that if $x_{\tau_{t}}^{\star} \sim p_{\wt{X}_{\tau_{t}}}$, then $x_{\tau_{t-1}}^{\star}$ satisfies $x_{\tau_{t-1}}^{\star} \sim p_{\wt{X}_{\tau_{t-1}}}$.
Therefore, it is straightforward to verify that
\begin{align*}
\frac{1}{\sqrt{\alpha_t}}\big(\Phi_t(X_t) + \sigma_t Z_t\big) 
& \overset{\mathrm d}{=} \frac{1}{\sqrt{\alpha_t}}\big(\Phi_t(\wt{X}_{\tau_t}) + \sigma_t Z_t\big) 
= \frac{1}{\sqrt{\alpha_t}}\bigg(\sqrt{\frac{1-\tau_t}{1-\tau_{t-1}}} \wt{X}_{\tau_{t-1}} + \sigma_t Z_t\bigg) \\
&\overset{\mathrm d}{=} \frac{1}{\sqrt{\alpha_t}}\bigg(\sqrt{1-\tau_t}X_0 + \sqrt{\frac{(1-\tau_t)\tau_{t-1}}{1-\tau_{t-1}}}Z + \sigma_t Z_t\bigg)
\overset{\mathrm d}{=} X_{t-1}.
\end{align*}
Here, the first steps holds since $X_t \overset{\mathrm d}{=} \wt{X}_{\tau_{t}}$ by our construction and $\tau_t = 1-\ol\alpha_t$; the second step arises from our construction of the maps $\varphi^\star_t$ and $\Phi_t$ in \eqref{eq:x-star-tau-integral}--\eqref{eq:phi-construction}; the third step uses $\wt{X}_{\tau_{t-1}} \overset{\mathrm d}{=} \sqrt{1-\tau_{t-1}}X_0 + \tau_{t-1}Z$; the last step is true because of \eqref{eq:forward-dist}, $(1-\tau_t)/\alpha_t = \ol{\alpha}_t/\alpha_t = \ol{\alpha}_{t-1}$, and
\begin{align*}
	\frac{(1-\tau_t)\tau_{t-1}}{1-\tau_{t-1}} + \sigma_t^2 = \frac{\overline{\alpha}_t\big(1-\overline{\alpha}_t(3-2\alpha_t)\big)}{\overline{\alpha}_t(3-2\alpha_t)} + \alpha_t - \frac{1}{3-2\alpha_t} = -\ol{\alpha}_t+\alpha_t = \alpha_t (1-\ol{\alpha}_{t-1}).
\end{align*}
This establishes \eqref{eq:phi=d=X-t-1}. 

\subsection{Proof of Lemma~\ref{lem:outlier}}
\label{sec:pf-lem:outlier}
\paragraph{Proof of Claim \eqref{eq:psi-Y-star-dist}.}
We shall establish an equivalent statement:
\begin{align}
\label{eq:psi-Y-star-dist-strong}
	\mathbb{E}_{\cA_t}\Big[\big\|\Phi_t(X_t) - Y_{t-1}^\star \big\|_2^2 \Big] \lesssim (1-\alpha_t)^6\bigg(\frac{d}{1-\overline{\alpha}_t}\bigg)^5.
\end{align}
In view of the definitions of $\Psi_t$ and $Y_{t-1}^\star$ in \eqref{eq:psi-defn}, and \eqref{eq:Y-star-defn-1}, respectively, one has
\begin{align*}
 	Y_{t-1}^\star - \Phi_t(X_t) & = X_{t} + (1-\alpha_t)\Big(s_t^\star(X_t) 
+ \Big\{\alpha_{t}^{3/2}s_{t-1}^\star\big(Y_{t}^{\star,\md}\big) - s_t^\star(X_t)\Big\}\ind\{X_t\in\cE_t\}\Big)   \\ 
 	& \quad -\big( (1-\alpha_t)\Psi_t(X_t) +X_t+(1-\alpha_t)s^\star_t(X_t) \big) \\ 
 	& = (1-\alpha_t)\bigg(\Big\{\alpha_{t}^{3/2}s_{t-1}^\star\big(Y_{t}^{\star,\md}\big) - s_t^\star(X_t)\Big\}\ind\{X_t\in\cE_t\}-\Psi_t(X_t)\bigg).
\end{align*} 
Hence, \eqref{eq:psi-Y-star-dist} follows as an immediate consequence of \eqref{eq:psi-Y-star-dist-strong}. Therefore, we shall focus on proving \eqref{eq:psi-Y-star-dist-strong}. 

Towards this, let us fix an arbitrary $2\leq t\leq T$ and denote
\begin{align}
	\tau_t^\md \defn 1-\overline{\alpha}_{t-1}.
\end{align}
For any $x\in\bR^d$, we define $\varphi^\md_t(\cdot):\bR^d\rightarrow \bR^d$ by
\begin{align}
\frac{\varphi^\md_t(x)}{\sqrt{1-\tau_t^\md}}
& \defn \frac{x}{\sqrt{1-\tau_{t}}} - \frac{1}{2}(\tau_t^\md - \tau_t)\frac{\theta^\star(\tau_t)}{(1-\tau_t)^{3/2}} 
\label{eq:x-tau-minus-defn},
\end{align}
and $\varphi^\acc_t(\cdot):\bR^d\rightarrow \bR^d$ as
\begin{align}
\frac{\varphi^\acc_t(x)}{\sqrt{1-\tau_{t-1}}}
&\defn \frac{x}{\sqrt{1-\tau_{t}}} 
- \frac{1}{2}\int_{\tau_t}^{\tau_{t-1}}
\bigg[\frac{\theta^\star(\tau_t)}{(1-\tau_t)^{3/2}} 
+ \frac{\tau - \tau_t}{\tau_t^\md - \tau_t}\bigg(\frac{\tilde{s}^\star\big(\varphi^\md_t(x),{\tau_t^\md}\big)}{(1-\tau_t^\md)^{3/2}} - \frac{\theta^\star(\tau_t)}{(1-\tau_t)^{3/2}}\bigg)\bigg] \diff \tau \label{eq:x-tau-t-1-defn} \\
&= \frac{x}{\sqrt{1-\tau_{t}}} 
- \frac{1}{2}(\tau_{t-1}-\tau_t)
\bigg[\frac{\theta^\star({\tau_t})}{(1-\tau_t)^{3/2}} 
+\frac12\frac{\tau_{t-1} - \tau_t}{\tau_t^\md - \tau_t}\bigg(\frac{\tilde{s}^\star\big(\varphi^\md_t(x),{\tau_t^\md}\big)}{(1-\tau_t^\md)^{3/2}} - \frac{\theta^\star({\tau_t})}{(1-\tau_t)^{3/2}}\bigg)\bigg] \label{eq:x-tau-t-1-expression}.
\end{align}
In particular, these allow us to express $Y_{t-1}^{\star}$ as
\begin{align}
	Y_{t-1}^{\star}= \sqrt{\frac{1-\tau_t}{1-\tau_{t-1}}} \varphi^\acc_t(X_t). \label{eq:Y-t-1-alt}
\end{align}
Combined with the definition of $\Phi_t$ in \eqref{eq:phi-construction}, this suggests that it suffices to control the distance between $\varphi^\star_t(x_{\tau_t}^\star)$ and $\varphi^\acc_t(x_{\tau_t}^\star)$.

To this end, by \eqref{eq:x-star-tau-integral} and \eqref{eq:x-tau-t-1-defn}, we know that
\begin{align}
 \frac{\varphi^\acc_t(x_{\tau_t}^\star)-\varphi^\star_t(x_{\tau_t}^\star)}{\sqrt{1-\tau_{t-1}}} 
	& = \frac{1}{2}\int_{\tau_t}^{\tau_{t-1}}\bigg[\frac{\theta^\star(\tau)}{(1-\tau)^{3/2}} - \frac{\theta^\star(\tau_t)}{(1-\tau_t)^{3/2}}
		- \frac{\tau - \tau_t}{\tau_t^\md - \tau_t}
		\bigg(\frac{\tilde{s}^\star\big(\varphi^\md_t(x_{\tau_t}^\star),{\tau_t^\md}\big)}{(1-\tau_t^\md)^{3/2}} - \frac{\theta^\star(\tau_t)}{(1-\tau_t)^{3/2}}\bigg)\bigg] \diff \tau \nonumber \\
	& = \frac{1}{2}\int_{\tau_t}^{\tau_{t-1}}\bigg[\frac{\theta^\star(\tau)}{(1-\tau)^{3/2}}
	- \frac{\theta^\star(\tau_t)}{(1-\tau_t)^{3/2}} - \frac{\tau - \tau_t}{\tau_t^\md - \tau_t} \bigg(\frac{\theta^\star(\tau_t^\md)}{(1-\tau_t^\md)^{3/2}}
	 - \frac{\theta^\star(\tau_t)}{(1-\tau_t)^{3/2}}\bigg)\bigg]  \diff \tau \nonumber \\
& \quad - \frac14 \frac{(\tau_{t-1} - \tau_t)^2}{\tau_t^\md - \tau_t}
\frac{\tilde{s}^\star\big(\varphi^\md_t(x_{\tau_t}^\star),{\tau_t^\md}\big)-\theta^\star(\tau_t^\md)}{(1-\tau_t^\md)^{3/2}}\label{eq:x-tau-t-1-dist-temp} . 
\end{align}
By the fundamental theorem of calculus, one has
\begin{align*}
\frac{\theta^\star(\tau')}{(1-\tau')^{3/2}} - \frac{\theta^\star(\tau_t)}{(1-\tau_t)^{3/2}}
= \int_{\tau_t}^{\tau'}\frac{\diff}{\diff \tau} \frac{\theta^\star(\tau)}{(1-\tau)^{3/2}} \diff \tau, \quad \forall \tau',
\end{align*}
and
\begin{align*}
\frac{\theta^\star(\tau'')}{(1-\tau'')^{3/2}} - \frac{\theta^\star(\tau_t)}{(1-\tau_t)^{3/2}} - (\tau''-\tau_t)\frac{\diff}{\diff\tau}\frac{\theta^\star(\tau_t)}{(1-\tau_t)^{3/2}}
	&= \int_{\tau_t}^{\tau''}\bigg(\frac{\diff}{\diff \tau'} \frac{\theta^\star(\tau')}{(1-\tau')^{3/2}} - \frac{\diff}{\diff \tau}\frac{\theta^\star(\tau_t)}{(1-\tau_t)^{3/2}}\bigg) \diff \tau' \\
		&= \int_{\tau_t}^{\tau''}\int_{\tau_t}^{\tau'}\frac{\diff^2}{\diff \tau^2} \frac{\theta^\star(\tau)}{(1-\tau)^{3/2}}  \diff \tau  \diff \tau', \quad \forall \tau''.
\end{align*}
Combining these expressions yields
\begin{align*}
& \frac{\theta^\star(\tau'')}{(1-\tau'')^{3/2}} - \frac{\theta^\star(\tau_t)}{(1-\tau_t)^{3/2}} 
	- \frac{\tau'' - \tau_t}{\tau_t^\md - \tau_t} \bigg(\frac{\theta^\star(\tau_t^\md)}{(1-\tau_t^\md)^{3/2}} - \frac{\theta^\star(\tau_t)}{(1-\tau_t)^{3/2}}\bigg) \\
	& = \int_{\tau_t}^{\tau''}\int_{\tau_t}^{\tau'}\frac{\diff^2}{\diff \tau^2} \frac{\theta^\star(\tau)}{(1-\tau)^{3/2}}  \diff \tau  \diff \tau' 
		- (\tau''-\tau_t) \bigg[
		\frac{1}{\tau_t^\md - \tau_t} \bigg(\frac{\theta^\star(\tau_t^\md)}{(1-\tau_t^\md)^{3/2}} - \frac{\theta^\star(\tau_t)}{(1-\tau_t)^{3/2}}\bigg)
		-\frac{\diff}{\diff\tau}\frac{\theta^\star(\tau_t)}{(1-\tau_t)^{3/2}}\bigg] \\
		& = \int_{\tau_t}^{\tau''}\int_{\tau_t}^{\tau'}\frac{\diff^2}{\diff \tau^2} \frac{\theta^\star(\tau)}{(1-\tau)^{3/2}}  \diff \tau  \diff \tau'
		- \frac{\tau''-\tau_t}{\tau_t^\md-\tau_t} \int_{\tau_t}^{\tau_t^\md}\int_{\tau_t}^{\tau'}\frac{\diff^2}{\diff \tau^2} \frac{\theta^\star(\tau)}{(1-\tau)^{3/2}}  \diff \tau  \diff \tau'.
\end{align*}
Consequently, we can rewrite \eqref{eq:x-tau-t-1-dist-temp} as
\begin{align*}
\frac{\varphi^\acc_t(x_{\tau_t}^\star)-\varphi^\star_t(x_{\tau_t}^\star)}{\sqrt{1-\tau_{t-1}}} 
		& = \frac12 \int_{\tau_t}^{\tau_{t-1}} \int_{\tau_t}^{\tau''}\int_{\tau_t}^{\tau'}\frac{\diff^2}{\diff \tau^2} \frac{\theta^\star(\tau)}{(1-\tau)^{3/2}}  \diff \tau  \diff \tau' \diff \tau''  -\frac{(\tau_{t-1} - \tau_t)^2}{4(\tau_t^\md - \tau_t)} \int_{\tau_t}^{\tau_t^\md}\int_{\tau_t}^{\tau'}\frac{\diff^2}{\diff \tau^2} \frac{\theta^\star(\tau)}{(1-\tau)^{3/2}}  \diff \tau  \diff \tau' \nonumber \\
		& \quad - \frac{(\tau_{t-1} - \tau_t)^2}{4(\tau_t^\md - \tau_t)}\frac{\tilde{s}^\star\big(\varphi^\md_t(x_{\tau_t}^\star),{\tau_t^\md}\big)-\theta^\star(\tau_t^\md)}{(1-\tau_t^\md)^{3/2}}.
\end{align*}
Taking the expectation over ${x_{\tau_t}^\star\sim p_{\wt{X}_{\tau_t}}}$, we obtain
\begin{align}
\mathbb{E}\Bigg[&\bigg\| \frac{\varphi^\acc_t(x_{\tau_t}^\star)-\varphi^\star_t(x_{\tau_t}^\star)}{\sqrt{1-\tau_{t-1}}} \bigg\|_2^2 \ind\big\{x_{\tau_t}^\star\in\cE_t \big\}\Bigg] \nonumber \\
& \numpf{i}{\lesssim} (\tau_{t-1}-\tau_t)^5\int_{\tau_t}^{\tau_{t-1}}\mathbb{E}\Bigg[\bigg\|\frac{\diff^2}{\diff \tau^2} \frac{\theta^\star(\tau)}{(1-\tau)^{3/2}}\bigg\|_2^2 \Bigg] \diff \tau  + \frac{(\tau_{t-1} - \tau_t)^4}{(\tau_t^\md - \tau_t)^2} (\tau^\md_t-\tau_t)^3\int_{\tau_t}^{\tau^\md_{t}}\mathbb{E}\Bigg[\bigg\|\frac{\diff^2}{\diff \tau^2} \frac{\theta^\star(\tau)}{(1-\tau)^{3/2}}\bigg\|_2^2 \Bigg] \diff \tau \nonumber \\
	&\quad+ \frac{(\tau_{t-1}-\tau_t)^4}{(\tau_{t}-\tau_t^\md)^{2}}\mathbb{E}\Bigg[\bigg\|\frac{\tilde{s}^\star\big(\varphi^\md_t(x_{\tau_t}^\star),{\tau_t^\md}\big)-\theta^\star(\tau_t^\md)}{(1-\tau_t^\md)^{3/2}}\bigg\|_2^2\ind\big\{x_{\tau_t}^\star\in\cE_t \bigg\}\Bigg] \nonumber \\
& \numpf{ii}{\asymp} (\tau_{t}-\tau_{t-1})^5\int_{\tau_{t-1}}^{\tau_{t}}\mathbb{E}\Bigg[\bigg\|\frac{\diff^2}{\diff \tau^2} \frac{\theta^\star(\tau)}{(1-\tau)^{3/2}}\bigg\|_2^2 \Bigg] \diff \tau \nonumber \\
	&\quad+ (\tau_{t}-\tau_{t-1})^2\mathbb{E}\Bigg[\bigg\|\frac{\tilde{s}^\star\big(\varphi^\md_t(x_{\tau_t}^\star),{\tau_t^\md}\big)-\theta^\star(\tau_t^\md)}{(1-\tau_t^\md)^{3/2}}\bigg\|_2^2\ind\big\{x_{\tau_t}^\star\in\cE_t \bigg\}\Bigg]
	  \label{eq:x-tau-t-1-dist-temp2}.
\end{align}
Here, (i) use Jensen's inequality that $(\int_0^T f\diff x)^2\leq T\int_0^T f^2\diff x$ for any $T>0$; (ii) is because of the following inequality:
\begin{align}
	\tau_t - \tau_t^\md = \ol{\alpha}_{t-1} - \ol{\alpha}_t = \ol{\alpha}_{t-1}(1-\alpha_t) \asymp 2\ol{\alpha}_t(1-\alpha_t) = \tau_t - \tau_{t-1} \label{eq:tau-minus-tau-dist},
\end{align}
where we use $\alpha_t\asymp 1$ from \eqref{eq:alpha-t-lb} in Lemma~\ref{lemma:step-size}.

\begin{itemize}
\item 
To bound the first term in \eqref{eq:x-tau-t-1-dist-temp2}, we introduce Lemma~\ref{lemma:u1-u2-l2norm} that controls the $\ell_2$ norm of the derivatives of $\theta^\star(\tau)$.
\begin{lemma}
\label{lemma:u1-u2-l2norm}
For any constant $k>0$, there exists some constant $C_k>0$ that only depends on $k$ such that for all $\tau\in(0,1)$:
	\begin{subequations}
	\label{eq:u-l2norm-ub-lemma}
	\begin{align}
	\mathbb{E}_{x^\star_\tau\sim \wt X_\tau}\Bigg[\bigg\|\frac{\diff}{\diff \tau} \frac{\theta^\star(\tau)}{(1-\tau)^{3/2}}\bigg\|_2^k\Bigg]
	&\leq C_k \frac{1}{(1-\tau)^k}\bigg(\frac{d}{\tau(1-\tau)}\bigg)^{3k/2}; \label{eq:u1-l2norm-ub-lemma}\\
	\mathbb{E}_{x^\star_\tau\sim \wt X_\tau}\Bigg[\bigg\|\frac{\diff^2}{\diff \tau^2} \frac{\theta^\star(\tau)}{(1-\tau)^{3/2}}\bigg\|_2^k\Bigg]
	&\leq C_k \frac{1}{(1-\tau)^k}\bigg(\frac{d}{\tau(1-\tau)}\bigg)^{5k/2}. \label{eq:u2-l2norm-ub-lemma}
	\end{align}
Moreover, if $\tau$ satisfies $-\log p_{\wt X_\tau}(x_\tau^\star) \leq 2C_2d\log T$, we further have
\begin{align}
	\label{eq:u1-l2norm-high-prob-ub}
\bigg\|\frac{\diff}{\diff \tau} \frac{\theta^\star(\tau)}{(1-\tau)^{3/2}}\bigg\|_2^k \leq C_k \frac{1}{(1-\tau)^k}\bigg(\frac{d\log T}{\tau(1-\tau)}\bigg)^{3k/2}.
\end{align}
	\end{subequations}
\end{lemma}
\begin{proof}
	See Appendix \ref{sec:pf-lemma:u1-u2-l2norm}.
\end{proof}
Therefore, applying \eqref{eq:u2-l2norm-ub-lemma} yields 
\begin{align}
\label{eq:x-tau-t-1-dist-temp2-term1-ub}
	\int_{\tau_{t-1}}^{\tau_{t}}\mathbb{E}\Bigg[\bigg\|\frac{\diff^2}{\diff \tau^2} \frac{\theta^\star(\tau)}{(1-\tau)^{3/2}}\bigg\|_2^2 \Bigg] \diff \tau \lesssim \int_{\tau_{t-1}}^{\tau_{t}} \frac{1}{(1-\tau)^2}\bigg(\frac{d}{\tau(1-\tau)}\bigg)^{5} \diff \tau \lesssim  \frac{\tau_{t}-\tau_{t-1}}{(1-\tau_t)^2}\bigg(\frac{d}{\tau_t(1-\tau_t)}\bigg)^{5}.
\end{align}
Here, the last step holds as \eqref{eq:1-alpha-1-olalpha} in Lemma~\ref{lemma:step-size} gives
\begin{align}
\label{eq:tau-t-t-1}
	\frac{\tau_{t}-\tau_{t-1}}{\tau_t(1-\tau_t)} = \frac{2(1-\alpha_t)}{1-\ol\alpha_t} \lesssim \frac{\log T}{T} = o(1),
\end{align}
which further implies that
\begin{align}
	\bigg|\frac{1}{\tau(1-\tau)}-\frac{1}{\tau_t(1-\tau_t)}\bigg|= \frac{|\tau_t-\tau|}{\tau_t(1-\tau_t)}\frac{|1-\tau-\tau_t|}{\tau(1-\tau)} = o(1)\frac{1}{\tau(1-\tau)},\quad \forall \tau\in[\tau_{t-1},\tau_t]. \label{eq:tau-1-tau-inv-diff}
\end{align}

\item 
Regarding the second term in \eqref{eq:x-tau-t-1-dist-temp2},  we find it helpful to define the Jacobian $\wt J:\bR^d \times (0,1)\rightarrow \bR^{d\times d}$
\begin{align}
\wt J\big(\sqrt{1-\tau}x,\tau\big) &\defn \frac{\partial}{\partial x}\frac{\tilde{s}^\star\big(\sqrt{1-\tau}x,\tau\big)}{(1-\tau)^{3/2}}
\label{eq:J-defn} \\
&= -\frac{1}{\tau(1-\tau)}I_d + \frac{1}{\tau^2} \Big(\bE\big[(x - \wt{X}_0)(x - \wt{X}_0\big)^{\top}\mid \wt{X}_\tau = \sqrt{1-\tau}x\big] \nonumber \\
& \hspace{10em} -\bE\big[x - \wt{X}_0\mid \wt{X}_\tau = \sqrt{1-\tau}x\big]\bE\big[x - \wt{X}_0\mid \wt{X}_\tau = \sqrt{1-\tau}x\big]^\top \Big) \label{eq:J-expression}.
\end{align}
Combined with the definition $\theta^\star(\tau_t^\md)\defn \tilde{s}^\star\big(x_{\tau_t^\md}^\star,\tau_t^\md\big)$, this allows us to express
\begin{align}
\frac{\theta^\star(\tau_t^\md) - \tilde{s}^\star\big(\varphi^\md_t(x_{\tau_t}^\star),\tau_t^\md\big)}{(1-\tau_t^\md)^{3/2}} 
= \int_0^1 \wt J\big(\gamma x_{\tau_t^\md}^{\star} + (1-\gamma)\varphi^\md_t(x_{\tau_t}^\star),\tau_t^\md\big) \diff  \gamma \cdot \frac{x_{\tau_t^\md}^{\star} - \varphi^\md_t(x_{\tau_t}^\star)}{\sqrt{1-\tau_t^\md}}. \label{eq:x-tau-t-1-dist-temp2-term2-temp}
\end{align}

\begin{itemize}
	\item 
For the vector term, recall the definition of $\varphi^\md_t(x_{\tau_t}^\star)$ in \eqref{eq:x-tau-minus-defn}. From the ODE \eqref{eq:ODE-star}, we can write
\begin{align*}
\frac{x_{\tau_t^\md}^{\star} - \varphi^\md_t(x_{\tau_t}^\star)}{\sqrt{1-\tau_t^\md}}
&= - \frac{1}{2}\int_{\tau_t}^{\tau_t^\md}\bigg(\frac{\theta^\star(\tau)}{(1-\tau)^{3/2}}
- \frac{\theta^\star(\tau_t)}{(1-\tau_t)^{3/2}}\bigg) \diff \tau 
= - \frac{1}{2}\int_{\tau_t}^{\tau_t^\md}\int_{\tau_t}^{\tau'}\frac{\diff}{\diff \tau} \frac{\theta^\star(\tau)}{(1-\tau)^{3/2}} \diff \tau \diff \tau'.
\end{align*}
Applying Lemma~\ref{lemma:u1-u2-l2norm} again shows that
\begin{align}
	\bE\Bigg[\bigg\|\frac{x_{\tau_t^\md}^{\star} - \varphi^\md_t(x_{\tau_t}^\star)}{\sqrt{1-\tau_t^\md}} \bigg\|_2^4 \Bigg] 
&\numpf{i}{\lesssim} (\tau_t-\tau_t^\md)^7\int^{\tau_t}_{\tau_t^\md} \bE \Bigg[\bigg\|\frac{\diff}{\diff \tau} \frac{\theta^\star(\tau)}{(1-\tau)^{3/2}} \bigg\|_2^4 \Bigg]  \diff \tau \nonumber \\
&\numpf{ii}{\lesssim} (\tau_t-\tau_t^\md)^7\int^{\tau_t}_{\tau_t^\md} \frac{1}{(1-\tau)^4} \bigg(\frac{d}{\tau(1-\tau)}\bigg)^6  \diff \tau  \nonumber \\
	&\lesssim \frac{(\tau_{t}-\tau_t^\md)^8}{(1-\tau_t)^4} \bigg(\frac{d}{\tau_t(1-\tau_t)}\bigg)^6 \asymp \frac{(\tau_{t}-\tau_{t-1})^8}{(1-\tau_t)^4} \bigg(\frac{d}{\tau_t(1-\tau_t)}\bigg)^6 \label{eq:x-tau-minus-dist-l2norm-ub},
\end{align}
where (i) applies Jensen' inequality; (ii) uses \eqref{eq:u1-l2norm-ub-lemma}; the last line follows from \eqref{eq:tau-1-tau-inv-diff} and $\tau_t-\tau_t^\md\asymp \tau_t-\tau_{t-1}$ in \eqref{eq:tau-minus-tau-dist}.

\item Turning to the Jacobian term, we need Lemma~\ref{lemma:x-star-x0-dist-l2norm} below to control it.
\begin{lemma}
\label{lemma:x-star-x0-dist-l2norm}
For any $\tau\in(0,1)$ and any integer $k > 0$, the following holds:
\begin{align}
\mathbb{E}\Big[\big\|\wt{X}_\tau - \sqrt{1-\tau}\wt{X}_0 \big\|_2^k\Big] \leq C_k (\tau d)^{k/2}, \label{eq:x-star-x0-dist-l2norm}
\end{align}
where $C_k>0$ is some constant that only depends on $k$.
In particular, one has
\begin{align}
\label{eq:s-star-l2norm-ub}
	\mathbb{E}\Big[\big\|\tilde{s}^\star\big(\wt{X}_\tau,\tau\big)\big\|_2^{k}\Big] \leq C_k \bigg(\frac{d}{\tau}\bigg)^{k/2}.
\end{align}
\end{lemma}
\begin{proof}
	See Appendix \ref{sec:pf-lemma:x-star-x0-dist-l2norm}.
\end{proof}
Equipped with Lemma~\ref{lemma:x-star-x0-dist-l2norm}, we claim that
\begin{align}
\label{eq:J-l2norm-ub}
\sup_{\gamma \in [0,1]}\mathbb{E}\bigg[\Big\|\wt J\big(\gamma x_{\tau_t^\md}^\star + (1-\gamma) \varphi^\md_t(x_{\tau_t}^\star),{\tau_t^\md}\big)\Big\|^4 \ind\big\{x_{\tau_t}^\star \in \cE_t\big\}\bigg] \lesssim \bigg(\frac{d}{\tau_t(1-\tau_t)}\bigg)^4,
\end{align}
and defer the proof to the end of this section. 

\item Combining \eqref{eq:x-tau-minus-dist-l2norm-ub} and \eqref{eq:J-l2norm-ub} with \eqref{eq:x-tau-t-1-dist-temp2-term2-temp} shows that
\begin{align}
	& \mathbb{E}_{x_{\tau_t}^\star \sim p_{\wt{X}_{\tau_{t}}}}\Bigg[\bigg\|\frac{\tilde{s}^\star\big(\varphi^\md_t(x_{\tau_t}^\star),{\tau_t^\md}\big)-\theta^\star(\tau_t^\md)}{(1-\tau_t^\md)^{3/2}}\bigg\|_2^2 \ind\big\{x_{\tau_t}^\star \in \cE_t\big\}\Bigg] \nonumber \\
	& \qquad\lesssim \sqrt{\bE\Bigg[\bigg\|\frac{x_{\tau_t^\md}^{\star} - \varphi^\md_t(x_{\tau_t}^\star)}{\sqrt{1-\tau_t^\md}} \bigg\|_2^4 \Bigg] \mathbb{E}\bigg[\Big\|\wt J\big(\gamma x_{\tau_t^\md}^{\star} + (1-\gamma)\varphi^\md_t(x_{\tau_t}^\star),\tau_t^\md\big)\Big\|^4 \ind\big\{x_{\tau_t}^\star \in \cE_t\big\}\bigg]} \nonumber \\
	& \qquad \lesssim \frac{(\tau_{t}-\tau_{t-1})^{4}}{(1-\tau_t)^2} \bigg(\frac{d}{\tau_t(1-\tau_t)}\bigg)^3 \cdot \bigg(\frac{d}{\tau_t(1-\tau_t)}\bigg)^2
	= \frac{(\tau_{t}-\tau_{t-1})^{4}}{(1-\tau_t)^2} \bigg(\frac{d}{\tau_t(1-\tau_t)}\bigg)^5.
	\label{eq:x-tau-t-1-dist-temp2-term2-ub}
\end{align}
\end{itemize}
\item 
Substituting \eqref{eq:x-tau-t-1-dist-temp2-term1-ub} and \eqref{eq:x-tau-t-1-dist-temp2-term2-ub} 
into \eqref{eq:x-tau-t-1-dist-temp2} demonstrates that
\begin{align}
	&\mathbb{E}\Bigg[\bigg\| \sqrt{\frac{1-\tau_t}{1-\tau_{t-1}}}\big(\varphi^\acc_t(x_{\tau_t}^\star)-\varphi^\star_t(x_{\tau_t}^\star)\big) \bigg\|_2^2 \ind\big\{x_{\tau_t}^\star\in\cE_t \big\}\Bigg] \notag \\
	&\qquad \lesssim (1-\tau_t)(\tau_{t}-\tau_{t-1})^5 \frac{\tau_{t}-\tau_{t-1}}{(1-\tau_t)^2}\bigg(\frac{d}{\tau_t(1-\tau_t)}\bigg)^{5}  
	+ (1-\tau_t)(\tau_t-\tau_{t-1})^2\frac{(\tau_{t}-\tau_{t-1})^{4}}{(1-\tau_t)^2} \bigg(\frac{d}{\tau_t(1-\tau_t)}\bigg)^5 \notag \\
	&\qquad \asymp \frac{(\tau_{t-1}-\tau_t)^6}{1-\tau_t}\bigg(\frac{d}{\tau_t(1-\tau_t)}\bigg)^5
	 \asymp (1-\alpha_t)^6 \bigg( \frac{d}{1-\ol\alpha_t} \bigg)^5,
	\label{eq:x-tau-t-1-ub}
\end{align}
where the last step uses \eqref{eq:tau-t-t-1}.
\end{itemize}
Collecting \eqref{eq:x-tau-t-1-ub} together with the definitions of $\Phi_t$ in \eqref{eq:phi-construction} and $Y^\star_{t-1}$ in \eqref{eq:Y-t-1-alt} finishes the proof of \eqref{eq:psi-Y-star-dist-strong}.

\paragraph{Proof of Claim \eqref{eq:E-subset-Psi}.}
 Recall $\tau_t \defn 1-\ol\alpha_t$ and $\tau_{t-1} \defn 1-\ol\alpha_t(3-2\alpha_t)$. We can express
\begin{align}
\Psi_t(x_{\tau_t}^{\star})
&\numpf{i}{=} \frac{1}{1-\alpha_t}\big(\Phi_t(x_{\tau_t}^{\star}) - x_{\tau_t}^{\star}\big) - s_{\tau_t}^\star(x_{\tau_t}^{\star}) \nonumber \\
&\numpf{ii}{=} \frac{2(1-\tau_t)^{3/2}}{\tau_t - \tau_{t-1}}\bigg(\frac{\varphi^\star_t(x_{\tau_{t}}^{\star})}{\sqrt{1-\tau_{t-1}}} - \frac{x_{\tau_{t}}^{\star}}{\sqrt{1-\tau_{t}}}\bigg) - s_{\tau_t}^\star(x_{\tau_t}^{\star}) \nonumber\\
&\numpf{iii}{=} -\frac{(1-\tau_t)^{3/2}}{\tau_t - \tau_{t-1}}\int_{\tau_t}^{\tau_{t-1}}\frac{\theta^{\star}(\tau)}{(1-\tau)^{3/2}}\diff \tau - \theta^\star({\tau_t}) \nonumber\\
&= -\frac{(1-\tau_t)^{3/2}}{\tau_t - \tau_{t-1}}\int_{\tau_t}^{\tau_{t-1}}\int_{\tau_t}^{\tau'}\frac{\diff}{\diff \tau} \frac{\theta^\star(\tau)}{(1-\tau)^{3/2}} \diff \tau \diff \tau', \label{eq:Psi-expression-integral}
\end{align}
where (i) arises from the definition of $\Psi_t$ in \eqref{eq:psi-defn}; (ii) uses the construction of $\Phi_t$ in \eqref{eq:phi-construction} and $2(1-\tau_t)/(\tau_t-\tau_{t-1}) = 1/(1-\alpha_t)$; (iii) follows from the definitions of $\varphi^\star_t$ in \eqref{eq:x-star-tau-integral} and $\theta^\star(\tau_t)\defn s_{\tau_t}^\star(x_{\tau_t}^{\star})$.

Next, we pause to present Lemma~\ref{lem:bound_cond} below, which show that when $x_{\tau_t}^\star \in \cE_t$, any $\tau\in[\tau_{t-1},\tau_t]$ obeys $-\log p_{\wt{X}_\tau}(x_{\tau}^{\star}) \lesssim d\log T$.
\begin{lemma} \label{lem:bound_cond}
For any integer $k>0$ and any $\tau\in(0,1)$ satisfying $-\log p_{\wt{X}_\tau}(x^\star_{\tau}) \le \theta d\log T$ for some $\theta > 1$, there exists some constant $C_k$ that only depends on $k$ such that
\begin{align} \label{eq:bound_cond}
\mathbb{E}\Big[\big\|\wt{X}_\tau - \sqrt{1-\tau}\wt{X}_0 \big\|_2^k\mid \wt{X}_\tau = x_{\tau}^{\star}\Big] \leq C_k (\theta d\tau\log T)^{k/2}.
\end{align}
In particular, this implies that
\begin{align}\label{eq:bound_cond_score_fn}
	\big\| \tilde{s}^\star(x_\tau^\star,\tau)\big\|_2 \leq C_k \sqrt{\frac{\theta d\log T}{\tau}}.
\end{align}
Moreover, for any $\tau'\in(0,1)$ satisfying $|\tau' - \tau| \leq c_0\tau(1-\tau)$ for some sufficiently small constant $c_0 >0$, we have 
\begin{align}
-\log p_{\wt{X}_{\tau'}}(x_{\tau'}^{\star}) \leq 2 \theta d\log T. \label{lem:x-tau-pdf-lb}
\end{align}
\end{lemma}
\begin{proof}
	See Appendix \ref{sec:pf-lem:bound_cond}.
\end{proof}

As $\tau_{t}-\tau_{t-1} = o(1)\tau_t(1-\tau_t)$ shown in \eqref{eq:tau-t-t-1}, one can then combine \eqref{lem:x-tau-pdf-lb} with \eqref{eq:u1-l2norm-high-prob-ub} from Lemma~\ref{lemma:u1-u2-l2norm} to bound
\begin{align*}
\big\|\Psi_t(x_{\tau_t}^\star)\big\|_2 
& \lesssim \frac{(1-\tau_t)^{3/2}}{\tau_t - \tau_{t-1}} \int_{\tau_t}^{\tau_{t-1}}\int_{\tau_t}^{\tau_{t-1}}\frac{1}{(1-\tau_t)}\bigg(\frac{d\log T}{\tau(1-\tau)}\bigg)^{3/2} \diff \tau \diff \tau'  \\
& \lesssim \frac{(1-\tau_t)^{3/2}}{\tau_t - \tau_{t-1}}(\tau_t - \tau_{t-1})^2 \frac{1}{(1-\tau_t)}\bigg(\frac{d\log T}{\tau_t(1-\tau_t)}\bigg)^{3/2} \\
& = \frac{\tau_t-\tau_{t-1}}{1-\tau_t}\bigg(\frac{d\log T}{\tau_t}\bigg)^{3/2} \\
& \asymp (1-\alpha_t)\bigg(\frac{d\log T}{1-\overline{\alpha}_t}\bigg)^{3/2} \leq C_4 (1-\alpha_t)\bigg(\frac{d\log T}{1-\overline{\alpha}_t}\bigg)^{3/2},
\end{align*}
provided $C_4$ is large enough. This proves \eqref{eq:E-subset-Psi}. In particular, one knows from \citet[A.18b]{li2024d} that
\begin{align*}
\mathbb{P}\bigg\{\big\|\Psi_t(X_t)\big\|_2 > C_4 (1-\alpha_t)\bigg(\frac{d\log T}{1-\overline{\alpha}_t}\bigg)^{3/2}\bigg\}
\leq \mathbb{P}\big(\cA_t^\setc \big)
\lesssim \frac{1}{T^{20}}.
\end{align*}

\paragraph{Proof of Claim \eqref{eq:psi-E-c-ub}.} Equipped with \eqref{eq:Psi-expression-integral}, we can bound
\begin{align*}
\mathbb{E}\Big[\big\|\Psi_t(X_t)\big\|_2^4\Big] 
& \numpf{i}{\leq} \frac{(1-\tau_t)^6}{(\tau_t - \tau_{t-1})^4} (\tau_{t-1}-\tau_{t})^7 \int^{\tau_{t-1}}_{\tau_{t}} \bE\Bigg[\bigg\| \frac{\diff}{\diff \tau} \frac{s^{\star}(\tau)}{(1-\tau)^{3/2}} \bigg\|_2^4\Bigg] \diff \tau  \\
& \numpf{ii}{\lesssim} \frac{(1-\tau_t)^6}{(\tau_t - \tau_{t-1})^4} (\tau_{t-1}-\tau_{t})^8 \frac{1}{(1-\tau_t)^4} \bigg(\frac{d}{\tau_t(1-\tau_t)}\bigg)^6 \\
& = \bigg(\frac{\tau_t - \tau_{t-1}}{1-\tau_t}\bigg)^4\bigg(\frac{d}{\tau_t}\bigg)^6 \asymp \bigg(\frac{1-\alpha_t}{\alpha_t}\bigg)^4 \bigg(\frac{d}{1-\ol{\alpha}_t}\bigg)^6,
\end{align*}
where we use Jensen's inequality in (i); (ii) is due to \eqref{eq:u2-l2norm-ub-lemma} from Lemma~\ref{lemma:u1-u2-l2norm} and \eqref{eq:tau-1-tau-inv-diff}; the last step follows from $\tau_t= 1-\ol{\alpha}_t$ and $\tau_{t-1}= 1-\ol{\alpha}_t(3-2\alpha_t)$.
It follows from the Cauchy-Schwartz inequality 
that
\begin{align*}
	\bE_{\cA_t^\setc}\Big[ \big\|\Psi_t(X_t)\big\|_2^2 \Big] \leq \sqrt{\mathbb{E}\Big[\big\|\Psi_t(X_t)\big\|_2^4\Big] \bP\big(\cA_t^{\mathrm c}\big)} \lesssim \frac{1}{T^{10}}\bigg(\frac{1-\alpha_t}{\alpha_t}\bigg)^2 \bigg(\frac{d}{1-\ol{\alpha}_t}\bigg)^3,
\end{align*}
as claimed in \eqref{eq:psi-E-c-ub}.

This completes the proof of Lemma~\ref{lem:outlier}.

\paragraph{Proof of Claim \eqref{eq:J-l2norm-ub}.}
Recall the definition of $\varphi^\md_t(x_{\tau_t}^\star)$ in \eqref{eq:x-tau-minus-defn}, which can be viewed as a function of $x_{\tau_t^\md}^\star$ in light of the ODE \eqref{eq:ODE-star}. We shall prove that when $x_{\tau_t}^\star\in \mathcal{E}_{t}$, the following density ratio bound holds for all $\gamma \in [0,1]$:
\begin{align*}
\frac{p_{\gamma x_{\tau_t^\md}^{\star} + (1-\gamma)\varphi^\md_t(x_{\tau_t}^\star)}}{p_{x_{\tau_t^\md}^{\star}}}
= \mathsf{det}^{-1}\Bigg(\frac{\partial }{\partial x_{\tau_t^\md}^{\star}} \Big(\gamma x_{\tau_t^\md}^{\star} + (1-\gamma)\varphi^\md_t(x_{\tau_t}^\star)\Big)\Bigg)
\lesssim 1.
\end{align*}
We can then derive
\begin{align}
\sup_{\gamma \in [0,1]} \mathbb{E}\bigg[\Big\|\wt J\big(\gamma x_{\tau_t^\md}^{\star} + (1-\gamma)\varphi^\md_t(x_{\tau_t}^\star),\tau_t^\md \big)\Big\|^4\ind\{x_{\tau_t}^\star\in \mathcal{E}_{t}\big\}\bigg] 
& \lesssim \mathbb{E}\Big[\big\|\wt J\big(x_{\tau_t^\md}^{\star},\tau_t^\md\big)\big\|^4\Big]. \label{eq:J-diff-temp}
\end{align}
In light of the expression of $\wt J$ in \eqref{eq:J-expression}, it is not hard to derive that for any $\tau\in(0,1)$,
\begin{align*}
 \bE\bigg[\Big\| \bE\Big[\big(\wt{X}_\tau - \sqrt{1-\tau} \wt{X}_0\big)\big(\wt{X}_\tau - \sqrt{1-\tau} \wt{X}_0\big)^{\top}\mid \wt{X}_\tau = x^\star_\tau\Big] \Big\|^4 \bigg] 
 &	\numpf{i}{\leq} \bE\Big[\big\|\wt{X}_\tau - \sqrt{1-\tau} \wt{X}_0 \big\|_2^8 \Big] \numpf{ii}{\lesssim} (\tau d)^4,
\end{align*}
where (i) uses Jensen's inequality and the tower property, (ii) invokes \eqref{eq:x-star-x0-dist-l2norm} from Lemma~\ref{lemma:x-star-x0-dist-l2norm}.
Combined with $\tau_t^\md\asymp \tau_{t}$ due to \eqref{eq:1-olalpha-O1} from Lemma~\ref{lemma:step-size}, this yields
\begin{align*}
	\bE\Big[\big\|\wt J\big(x_{\tau_t^\md}^{\star},\tau_t^\md\big) \big\|^4\Big] \lesssim \frac{1}{\tau_t^4(1-\tau_t)^4} + \frac{(\tau_t d)^4}{\tau_t^8(1-\tau_t)^4} \asymp \bigg(\frac{d}{\tau_t(1-\tau_t)}\bigg)^4.
\end{align*}
Plugging this into \eqref{eq:J-diff-temp} gives the advertised claim \eqref{eq:J-l2norm-ub}.

Therefore, it remains to prove the bound on the density ratio. Fix an arbitrary $\gamma\in[0,1]$. 
We can compute
\begin{align}
\frac{\partial}{\partial x_{\tau_t^\md}^{\star}}\Big(\gamma x_{\tau_t^\md}^{\star} + (1-\gamma)\varphi^\md_t(x_{\tau_t}^\star)\Big)
&= \gamma I_d + (1-\gamma)\underbrace{\frac{\partial \varphi^\md_t(x_{\tau_t}^\star)/\sqrt{1-\tau_t^\md}}{\partial x_{\tau_{t}}^{\star}/\sqrt{1-\tau_{t}}}}_{(\mathrm{I})}
\underbrace{\frac{\partial x_{\tau_{t}}^{\star}/\sqrt{1-\tau_{t}}}
{\partial x_{\tau_t^\md}^{\star}/\sqrt{1-\tau_t^\md}}}_{(\mathrm{II})}. \label{eq:pdf-Jacob}
\end{align}
Note that for any positive semidefinite matrix $A\in\bR^{d\times d}$ and any $\gamma\in[0,1]$, one can use Jensen's inequality, $\log\big(\gamma+(1-\gamma)x\big)\geq (1-t)\log x$ for any $x>0$, to find
\begin{align*}
 	\deter\big(\gamma I_d + (1-\gamma)A\big) \geq \deter(A)^{1-\gamma}.
 \end{align*} 
Hence, it suffices to control the determinants of (I) and (II) separately.
\begin{itemize}
	\item Regarding (I), we know from \eqref{eq:x-tau-minus-defn} that
\begin{align*}
\mathrm{(I)}=\frac{\partial \varphi^\md_t(x_{\tau_t}^\star)/\sqrt{1-\tau_t^\md}}
{\partial x_{\tau_{t}}^{\star}/\sqrt{1-\tau_{t}}}
& = I_d - \frac{\tau_t^\md-\tau_t}{2} \frac{\partial \tilde{s}^\star(x_{\tau_t}^\star,\tau_t)/(1-\tau_t)^{3/2}}{\partial x_{\tau_{t}}^{\star}/\sqrt{1-\tau_{t}}}  = I_d + \frac{\tau_t-\tau_t^\md}{2} \wt J(x_{\tau_{t}}^{\star},{\tau_t}),
\end{align*}
where the last step uses the definition of $\wt J$ in \eqref{eq:J-defn}. By the expression of $\wt J$, it is straightforward to see that
$\wt J(x_{\tau}^{\star},\tau) + \frac{1}{\tau(1-\tau)}I_d \succeq 0$ for any $\tau\in(0,1)$. In addition, note that \eqref{eq:alpha-t-lb}--\eqref{eq:1-alpha-1-olalpha} from Lemma~\ref{lemma:step-size} implies
\begin{align}
	\frac{\tau_t-\tau_t^\md}{\tau_t(1-\tau_t)} = \frac{1-\alpha_t}{\alpha_t(1-\ol{\alpha}_t)}
	\lesssim \frac{\log T}{T}. \label{eq:tau-tau-mid-diff-ub}
\end{align}
Taken collectively, these demonstrate that
\begin{align}
	\deter\mathrm{(I)} = \deter\bigg(I_d + \frac{\tau_t-\tau_t^\md}{2}\wt J(x_{\tau_{t}}^\star,\tau_t)\bigg) 
	\geq \bigg(1-O\Big(\frac{\log T}{T}\Big)\bigg)^d = 1-O\bigg(\frac{d\log T}{T}\bigg). \label{eq:tau-minus-tau-ub}
\end{align}

\item As for (II), we can derive
\begin{align*}
\frac{\partial}{\partial \tau}
\frac{\partial x_{\tau}^{\star}/\sqrt{1-\tau}}{\partial x_{\tau_t^\md}^{\star}/\sqrt{1-\tau_t^\md}} 
	& = \frac{\partial}{\partial x_{\tau_t^\md}^{\star}/\sqrt{1-\tau_t^\md}}\frac{\partial x_{\tau}^{\star}/\sqrt{1-\tau}}{\partial \tau}  \\
& \numpf{i}{=} \frac{\partial}{\partial x_{\tau_t^\md}^{\star}/\sqrt{1-\tau_t^\md}} \bigg(-\frac12 \frac{\tilde{s}^\star(x_{\tau}^\star,\tau)}{(1-\tau)^{3/2}}\bigg) \\
& =   \frac{\partial}{\partial x_{\tau}^{\star}/\sqrt{1-\tau}}\bigg(-\frac12 \frac{\tilde{s}^\star(x_{\tau}^\star,\tau)}{(1-\tau)^{3/2}}\bigg) \frac{\partial x_{\tau}^{\star}/\sqrt{1-\tau}}{\partial x_{\tau_t^\md}^{\star}/\sqrt{1-\tau_t^\md}} \\
& \numpf{ii}{=} -\frac12 \wt J(x_\tau^\star,\tau) \frac{\partial x_{\tau}^{\star}/\sqrt{1-\tau}}{\partial x_{\tau_t^\md}^{\star}/\sqrt{1-\tau_t^\md}},
\end{align*}
where (i) arises from the ODE \eqref{eq:ODE-star} and (ii) follows from \eqref{eq:J-defn}. As $\frac{\mathrm d}{\mathrm d t}\deter(A) = \tr(A^{-1}\frac{\mathrm d}{\mathrm d t}A) \deter(A)$, we obtain
\begin{align*}
	\frac{\partial}{\partial \tau}\mathsf{det}\Bigg(\frac{\partial x_{\tau}^{\star}/\sqrt{1-\tau}}{\partial x_{\tau_t^\md}^{\star}/\sqrt{1-\tau_t^\md}}\Bigg) = -\frac12 \mathsf{tr}\big(\wt J(x_\tau^\star,\tau)\big)\mathsf{det}\Bigg(\frac{\partial x_{\tau}^{\star}/\sqrt{1-\tau}}{\partial x_{\tau_t^\md}^{\star}/\sqrt{1-\tau_t^\md}}\Bigg).
\end{align*}
Solving this equation gives
\begin{align}
	\mathsf{det}\Bigg(\frac{\partial x_{\tau'}^{\star}/\sqrt{1-\tau'}}{\partial x_{\tau_t^\md}^{\star}/\sqrt{1-\tau_t^\md}}\Bigg) = \exp\biggl(\int_{\tau_t^\md}^{\tau'} -\frac12 \mathsf{tr}\big(\wt J(x_\tau^\star,\tau)\big) \diff \tau \bigg),\quad \forall \tau'. \label{eq:partial-derivative-x}
\end{align}


Meanwhile, for any $\tau\in[\tau_t^\md,\tau_t]$, one has
\begin{align}
\tr\big(\wt J(x_{\tau}^{\star},\tau)\big)
& \leq  - \frac{d}{\tau(1-\tau)} + \frac{1}{\tau^2(1-\tau)}\tr\bigg(\bE\Big[\big(\wt{X}_\tau - \sqrt{1-\tau} \wt{X}_0\big)\big(\wt{X}_\tau - \sqrt{1-\tau} \wt{X}_0\big)^{\top}\mid \wt{X}_\tau = x^\star_\tau\Big] \bigg) \nonumber \\
	& = - \frac{d}{\tau(1-\tau)} +\frac{1}{\tau^2(1-\tau)} \bE\Big[\big\|\wt{X}_\tau - \sqrt{1-\tau} \wt{X}_0\big\|_2^2 \mid \wt{X}_\tau = x^\star_\tau\Big]. \label{eq:trace-J-ub-derivation}
\end{align}
where the inequality holds as $\cov(X)\preceq \bE[X^2] $ for any random vector $X$.
Recall that Lemma~\ref{lem:bound_cond} demonstrates that $-\log p_{\wt{X}_{\tau}}(x_{\tau}^{\star}) \lesssim d\log T$ for all $\tau\in[\tau_{t-1},\tau_t]$ when $x_{\tau_t}^{\star}\in\cE_t$.
Applying \eqref{eq:bound_cond} from Lemma~\ref{lem:bound_cond} shows that for any $\tau\in[\tau_t^\md,\tau_t]$,
\begin{align}
\tr\big(\wt J(x_{\tau}^{\star},\tau)\big)
&\leq - \frac{d}{\tau(1-\tau)} + O\bigg(\frac{d \log T}{\tau(1-\tau)}\bigg) \lesssim \frac{d \log T}{\tau(1-\tau)}. \label{eq:J-trace-ub}
\end{align}
 As a result, we know that
\begin{align}
	\deter(\mathrm{II}) & = \mathsf{det}\Bigg(\frac{\partial x_{\tau'}^{\star}/\sqrt{1-\tau'}}{\partial x_{\tau_t^\md}^{\star}/\sqrt{1-\tau_t^\md}}\Bigg)
	= \exp\biggl(  -\frac12\int_{\tau_t^\md}^{\tau_t} \tr\big(\wt J(x_\tau^\star,\tau)\big) \diff \tau \bigg) \geq \exp\bigg\{-\int_{\tau_t^\md}^{\tau_t} O\bigg(\frac{d\log T}{\tau(1-\tau)}\bigg) \diff \tau \bigg\} \nonumber \\
 & = \exp\bigg\{-O\bigg(\frac{(\tau_t-\tau_t^\md) d\log T}{\tau_t(1-\tau_t)}\bigg) \bigg\}
 = \exp\bigg\{-O\bigg(\frac{d\log^2 T}{T}\bigg) \bigg\} \nonumber  \\
	& = 1+O\bigg(\frac{d\log^2T}{T}\bigg), \label{eq:det-exp-int-J}
\end{align}
where the first inequality arises from \eqref{eq:J-trace-ub}; the second line uses \eqref{eq:tau-1-tau-inv-diff} and \eqref{eq:tau-tau-mid-diff-ub}; the last line holds as long as $d\log^2 T=o(T)$.

\item Finally, substituting \eqref{eq:tau-minus-tau-ub} and \eqref{eq:det-exp-int-J} into \eqref{eq:pdf-Jacob} allows us to bound that
\begin{align*}
	\deter \Bigg(\frac{\partial }{\partial x_{\tau_t^\md}^{\star}} \Big(\gamma x_{\tau_t^\md}^{\star} + (1-\gamma)\varphi^\md_t(x_{\tau_t}^\star)\Big)\Bigg) & \geq \big( \deter(\mathrm I) \cdot \deter(\mathrm{II}) \big)^{1-\gamma} \\
	& \geq \Big(1+O\big(T^{-1}d\log T \big)\Big)^{1-\gamma} \Big(1+O\big(T^{-1}d\log^2T\big)\Big)^{1-\gamma} \gtrsim 1,
\end{align*}
provided $d\log^2 T = o(T)$.  This finishes the proof.
\end{itemize}

\subsection{Proof of Lemma~\ref{lem:TV-crude}}
\label{sec:proof-lem:TV-crude}

For convenience of notation, we denote $Z \defn Z^\md_t$ henceforth. 
Recall the definitions of $\cB_t$ and $\{\cF_{t,i}\}_{i=1}^3$ in \eqref{eq:event-pdf-ratio} and \eqref{def:set-F}, respectively. 
To bound $\bP\{\cB_t^\setc\}$, let us define an additional set:
\begin{align*}
\cF_{t,4} &\defn \Big\{x\in\bR^d\colon p_{Y_t^\md(X_t)}(x)\leq 2\,p_{Y^{\star,\md}_t + (1-\alpha_t)Z }(x)\Big\}.
\end{align*}
It is straightforward to verify that 
\begin{align*}
	\big\{X_t + (1-\alpha_t)Z\in\cF_{t,1}\big\}   \cap \big\{Y^{\star,\md}_t+(1-\alpha_t)Z\in \cF_{t,3}\big\}\cap \big\{Y_t^\md(X_t)\in\cF_{t,4}\big\} \subset \cB_t.
\end{align*}
By the union bound, it suffices to bound the three probabilities individually.

\begin{itemize}
	\item Let us start with the first quantity involving $\cF_{t,1}$.
Note that for any two random vectors $X,Y\in\bR^d$, we can derive
\begin{align}
	\bP\big\{p_{Y}(Y) > 2\,p_{X}(Y) \big\}
	& = \int \ind\big\{p_{Y}(y) > 2\,p_{X}(y)\big\}p_{Y}(y)\diff y \notag\\
	& < \int \ind\big\{p_{Y}(y) > 2\,p_{X}(y)\big\}2\big(p_{Y}(y)-p_X(y)\big)\diff y \notag\\
	& \leq 2 \int \big|p_{Y}(y)-p_X(y)|\diff y = 4\, \TV(X,Y). \label{eq:density-bound-tv}
\end{align}
This allows us to focus on controlling the total variation distance
\begin{align*}
\bP\big\{X_t + (1-\alpha_t)Z\notin\cF_{t,1}\big\}\leq 4\, \TV(X_t + (1-\alpha_t)Z,\,X_t).
\end{align*}

Towards this, note that conditional on $X_0=x_0$, $X_{t} \sim \mathcal{N}\big(\sqrt{\overline{\alpha}_t}x_0, (1-\overline{\alpha}_t)I_d\big)$ and $X_{t} + (1-\alpha_t)Z \sim \mathcal{N}\big(\sqrt{\overline{\alpha}_t}x_0, \big(1-\overline{\alpha}_t + (1-\alpha_t)^2\big)I_d\big)$.
Using the formula of the KL divergence for normal distributions, one knows that for any $x_0\in\bR^d$,
\begin{align*}
\KL\big(&X_t\mid X_0=x_0 \,\|\, X_t+(1-\alpha_t)Z\mid X_0=x_0\big) \\
& = \frac d2\log\frac{1-\ol\alpha_t+(1-\alpha_t)^2}{1-\ol\alpha_t}	-\frac d2 + \frac d2 \frac{1-\ol\alpha_t}{1-\ol\alpha_t+(1-\alpha_t)^2} \\
& \leq \frac d2\bigg(\frac{(1-\alpha_t)^2}{1-\ol\alpha_t}-\frac14 \frac{(1-\alpha_t)^4}{(1-\ol\alpha_t)^2}-1+1 - \frac{(1-\alpha_t)^2}{1-\ol\alpha_t}+ \frac{(1-\alpha_t)^4}{(1-\ol\alpha_t)^2} \bigg) \\
&\lesssim \frac{(1-\alpha_t)^4d}{(1-\ol\alpha_t)^2},
\end{align*}
where we use $(1-\alpha_t)^2/(1-\ol\alpha_t)=o(1)$ due to \eqref{eq:1-alpha-1-olalpha} and $d\log^2 T = o(T)$, and $\log(1+x)\leq x - x^2/4$ and $1/(1+x)\leq 1-x+x^2$ for any $x\in[0,1]$.
It follows from the data processing inequality that
\begin{align*}
	\KL\big(X_t\,\|\, X_t+(1-\alpha_t)Z\big)\leq \bE_{x_0\sim p_{X_0}}\Big[\KL\big(X_t\mid X_0=x_0 \,\|\, X_t+(1-\alpha_t)Z\mid X_0=x_0\big)\Big] \lesssim \frac{(1-\alpha_t)^4d}{(1-\ol\alpha_t)^2}.
\end{align*}
Consequently, invoking Pinsker's inequality yields
\begin{align}
	\mathsf{TV}\big(X_t + (1-\alpha_t)Z, X_{t}\big) 
	\leq \sqrt{\frac12\KL\big(X_t + (1-\alpha_t)Z\,\|\, X_{t}\big)} \lesssim \frac{(1-\alpha_t)^2\sqrt{d}}{1-\ol\alpha_t}. \label{eq:prob-It-term1}
\end{align}

	\item For the second term related to $\cF_{t,3}$, recall the definition of $Y_{t}^{\star,\md}$ in \eqref{eq:Y-star-defn-1/2}. we can use the argument in \eqref{eq:density-bound-tv} to focus on the total variation:
	\begin{align*}
		\bP\big\{Y^{\star,\md}_t + (1-\alpha_t)Z\notin \cF_{t,3}\big\} \leq \mathsf{TV}\big(Y_{t}^{\star,\md} + (1-\alpha_t)Z, X_{t-1}\big).
	\end{align*}
	To this end, we can first use the triangle inequality to decompose
\begin{align*}
&\mathsf{TV}\big(Y_{t}^{\star,\md} + (1-\alpha_t)Z, X_{t-1}\big) \\ 
& \qquad = \mathsf{TV}\bigg(\frac{1}{\sqrt{\alpha_{t}}}\Big(X_{t} + \frac{1-\alpha_t}{2\alpha_t}s_t^{\star}(X_t)\Big) + (1-\alpha_t)Z, X_{t-1}\bigg) \nonumber \\
& \qquad \leq \mathsf{TV}\bigg(\frac{1}{\sqrt{\alpha_{t}}}\Big(X_{t} + \frac{1-\alpha_t}{2\alpha_t}s_t^{\star}(X_t)\Big)+ (1-\alpha_t)Z, X_{t-1}+ (1-\alpha_t)Z\bigg) + \mathsf{TV}\big(X_{t-1} + (1-\alpha_t)Z, X_{t-1}\big) \\
& \qquad \leq \TV\bigg(\frac{1}{\sqrt{\alpha_{t}}}\Big(X_{t} + \frac{1-\alpha_t}{2\alpha_t}s_t^{\star}(X_t)\Big), X_{t-1}\bigg) + \mathsf{TV}\big(X_{t-1} + (1-\alpha_t)Z, X_{t-1}\big).
\end{align*}
where the last step follows from the data processing inequality.
It is not hard to see that the bound in \eqref{eq:prob-It-term1} also holds for $\mathsf{TV}\big(X_{t-1} + (1-\alpha_t)Z, X_{t-1}\big)$.
Meanwhile, we can apply a similar argument for \citet[Lemma 4]{li2024sharp} to show that
\begin{align*}
	\mathsf{TV}\bigg(\frac{1}{\sqrt{\alpha_{t}}}\Big(X_{t} + \frac{1-\alpha_t}{2\alpha_t}s_t^{\star}(X_t)\Big), X_{t-1}\bigg) \lesssim \bigg(\frac{(1-\alpha_t)d\log T}{1-\overline{\alpha}_t}\bigg)^2.
\end{align*}
Combining these two bounds demonstrates that
\begin{align}
	\mathsf{TV}\bigg(\frac{1}{\sqrt{\alpha_{t}}}\Big(X_{t} + \frac{1-\alpha_t}{2\alpha_t}s_t^{\star}(X_t)\Big) + (1-\alpha_t)Z, X_{t-1}\bigg) &\lesssim \bigg(\frac{(1-\alpha_t)d\log T}{1-\overline{\alpha}_t}\bigg)^2 + \frac{(1-\alpha_t)^2\sqrt{d}}{1-\ol\alpha_t}. \label{eq:prob-It-term2}
\end{align}

\item For the term involving $\cF_{t,4}$, recall the definition of $Y_{t}^{\md}$ in \eqref{eq:sampler-mid}, which gives
\begin{align*}
	Y_{t}^{\md}(X_t)\mid X_t=x &\sim \cN\bigg(\frac{1}{\sqrt{\alpha_{t}}}\Big(x + \frac{1-\alpha_t}{2\alpha_t}s_t(x)\Big), (1-\alpha_t)^2I_d\bigg), \\ 
	Y^{\star,\md}_{t}+(1-\alpha)Z \mid X_t=x & \sim \cN\bigg(\frac{1}{\sqrt{\alpha_{t}}}\Big(x + \frac{1-\alpha_t}{2\alpha_t}s_t^{\star}(x)\Big), (1-\alpha_t)^2I_d\bigg).
\end{align*}
We shall control their density ratio, which depends on the following difference:
\begin{align}
	& \bigg\|y - \frac{1}{\sqrt{\alpha_{t}}}\Big(x + \frac{1-\alpha_t}{2\alpha_t}s_t(x)\Big)\bigg\|_2^2 - \bigg\|y - \frac{1}{\sqrt{\alpha_{t}}}\Big(x + \frac{1-\alpha_t}{2\alpha_t}s_t^\star(x)\Big)\bigg\|_2^2 \notag\\
	& \qquad = \bigg\|y - \frac{1}{\sqrt{\alpha_{t}}}\Big(x + \frac{1-\alpha_t}{2\alpha_t}s_t^\star(x)\Big) + \frac{1-\alpha_t}{2\alpha_t^{3/2}}\big(s_t^\star(x)-s_t(x)\big) \bigg\|_2^2 - \bigg\|y - \frac{1}{\sqrt{\alpha_{t}}}\Big(x + \frac{1-\alpha_t}{2\alpha_t}s_t^\star(x)\Big)\bigg\|_2^2\notag\\ 
	& \qquad \leq \frac{1-\alpha_t}{\alpha_t^{3/2}}\bigg|\Big[y - \frac{1}{\sqrt{\alpha_{t}}}\Big(x + \frac{1-\alpha_t}{2\alpha_t}s_t^\star(x)\Big)\Big]^\top\big(s_t^\star(x)-s_t(x)\big)\bigg|  + \frac{(1-\alpha_t)^2}{4\alpha_t^{3}}\big\|s_t^\star(x)-s_t(x) \big\|_2^2. \label{eq:y-mid-star-diff}
\end{align}

To control these two terms, we can use Markov's inequality and Assumption~\ref{assu:score-error} to get
\begin{align*}
	\bP\bigg\{\big\|s_t^\star(X_t)-s_t(X_t) \big\|_2^2 > \frac{1}{16\log T}\bigg\}\leq 16\,\bE\Big[\big\|s_t^\star(X_t)-s_t(X_t) \big\|_2^2\Big] \log T = 16\,\veps_t^2 \log T.
\end{align*}
In addition, conditioned on $X_t=x$, $Y_t^\md(x)-\frac{1}{\sqrt{\alpha_{t}}}\big(x + \frac{1-\alpha_t}{2\alpha_t}s_t(x)\big)=(1-\alpha_t)Z\sim \cN\big(0,(1-\alpha_t)^2I_d\big)$ is independent of $s_t^\star(x)-s_t(x)$. It follows from standard Gaussian properties that
\begin{align*}
	& \bP\bigg\{\Big[Y_t^\md(x) - \frac{1}{\sqrt{\alpha_{t}}}\Big(x + \frac{1-\alpha_t}{2\alpha_t}s_t(x)\Big)\Big]^\top \big(s_t^\star(x)-s_t(x)\big)  \\ 
	& \qquad > 2(1-\alpha_t) \big\|s_t^\star(x)-s_t(x)\big\|_2 \sqrt{\log T} \,\Big|\, X_t = x \bigg\}\leq 2\,T^{-2}.
\end{align*}
Hence, if we define the event
\begin{align*}
	\cD_t \defn \bigg\{x,y\colon &\Big[y - \frac{1}{\sqrt{\alpha_{t}}}\Big(x + \frac{1-\alpha_t}{2\alpha_t}s_t(x)\Big)\Big]^\top \big(s_t^\star(x)-s_t(x)\big) \leq \frac{1-\alpha_t}2 \text{ and }\big\|s_t^\star(x)-s_t(x) \big\|_2 \leq \frac{1}{4\sqrt{\log T}}
	 \bigg\},
\end{align*}
we know from the union bound that
\begin{align}
	\bP\big\{(Y_t^\md(X_t),X_t)\in\cD_t^\setc\big\} \leq 2\,T^{-2}+16\,\veps_t^2\log T. \label{eq:prob-D-ub}
\end{align}
Moreover, on the set $\cD_t$, one can use \eqref{eq:y-mid-star-diff} to derive
\begin{align*}
	&\exp\biggl(-\frac{1}{2(1-\alpha_t)^2}\Big\|y - \frac{1}{\sqrt{\alpha_{t}}}\Big(x + \frac{1-\alpha_t}{2\alpha_t}s_t(x)\Big)\Big\|_2^2\biggr) \\ 
	& \qquad \leq \exp\biggl(-\frac{1}{2(1-\alpha_t)^2}\Big\|y - \frac{1}{\sqrt{\alpha_{t}}}\Big(x + \frac{1-\alpha_t}{2\alpha_t}s_t^\star(x)\Big)\Big\|_2^2+\frac1{4\alpha_t^{3/2}}+\frac1{8\alpha_t^3 \log T}\biggr) \\
	& \qquad \leq \frac{3}{2}\exp\biggl(-\frac{1}{2(1-\alpha_t)^2}\Big\|y - \frac{1}{\sqrt{\alpha_{t}}}\Big(x + \frac{1-\alpha_t}{2\alpha_t}s_t^\star(x)\Big)\Big\|_2^2\bigg)
\end{align*}
where the last step follows from \eqref{eq:alpha-t-lb} that $1-\alpha_t\lesssim \log T/T=o(1)$ for $T$ large enough. This implies that on the set $\cD_t$:
\begin{align}
p_{X_t, Y_{t}^{\md}(X_t)}(x, y) &= p_{X_t}(x)\big(2\pi(1-\alpha_t)^2\big)^{-d/2}\exp\biggl(-\frac{1}{2(1-\alpha_t)^2}\Big\|y - \frac{1}{\sqrt{\alpha_{t}}}\Big(x + \frac{1-\alpha_t}{2\alpha_t}s_t(x)\Big)\Big\|_2^2\biggr) \notag\\
&\le \frac{3}{2}p_{X_t}(x)\big(2\pi(1-\alpha_t)^2\big)^{-d/2}\exp\biggl(-\frac{1}{2(1-\alpha_t)^2}\Big\|y - \frac{1}{\sqrt{\alpha_{t}}}\Big(x + \frac{1-\alpha_t}{2\alpha_t}s_t^{\star}(x)\Big)\Big\|_2^2\biggr) \notag\\
&= \frac{3}{2}p_{X_t, Y^{\star,\md}_{t}+(1-\alpha_t)Z}(x, y). \label{eq:p-diff-D-ub}
\end{align}
Hence, integrating over $\cF_{t,4}^\setc = \big\{y\in\bR^d\colon p_{Y_t^\md(X_t)}(y)> 2\,p_{Y^{\star,\md}_t + (1-\alpha_t)Z }(y)\big\}$ gives
\begin{align}
\bP\big\{Y_t^\md(X_t)\notin\cF_{t,4}\big\} & =\int_{y \in \cF_{t,4}^\setc}p_{Y_{t}^{\md}(X_t)}(y) \diff y \notag\\
&\le 4\int_{y \in \cF_{t,4}^\setc}\Big(p_{Y_{t}^{\md}(X_t)}(y) - \frac{3}{2}p_{Y^{\star,\md}_{t}+(1-\alpha_t)Z}(y)\Big) \diff y \notag\\
&= 4\int_{y \in \cF_{t,4}^\setc} \int_x \Big(p_{X_t, Y_{t}^{\md}(X_t)}(x, y) - \frac{3}{2}p_{X_t, Y^{\star,\md}_{t}+(1-\alpha_t)Z}(x, y)\Big) \diff x\diff y \notag\\
&\numpf{i}{\le} 4\int\int_{(x, y) \notin \mathcal{D}_t} p_{X_t,Y_{t}^{\md}(X_t)}(y) \diff x\diff y \notag\\
&\numpf{ii}{\le} 8T^{-2}+64\veps_t^2\log T. \label{eq:prob-It-term3}
\end{align}
where (i) arises from \eqref{eq:p-diff-D-ub} and (ii) uses \eqref{eq:prob-D-ub}.


\item
Finally, putting \eqref{eq:prob-It-term1}, \eqref{eq:prob-It-term2}, and \eqref{eq:prob-It-term3} together leads to our desired result
\begin{align*}
	\mathbb{P}\big(\cB_t^\setc\big)  
	& \lesssim 
	\frac{(1-\alpha_t)^2\sqrt{d}}{1-\ol\alpha_t} 
	+ \bigg(\frac{(1-\alpha_t)d\log T}{1-\overline{\alpha}_t}\bigg)^2
	+ \frac1{T^2}+\veps_t^2\log T \\
	& \asymp \bigg(\frac{(1-\alpha_t)d\log T}{1-\overline{\alpha}_t}\bigg)^2 + \frac1{T^2}+\veps_t^2\log T
\end{align*}
where the last step holds as $\ol\alpha_t\in(0,1)$.

\end{itemize}

\subsection{Proof of Lemma \ref{lemma:2-order-approx-error}}\label{sub:proof_of_lemma_ref_lemma_2_order_approx_error}
Let us apply the triangle inequality again to further decompose
\begin{align}	
&\mathbb{E}_{\cA_t\cap\cB_t}\bigg[\Big\|\alpha_{t}^{3/2}s_{t-1}\big(Y_{t}^{\md}\big)  - \alpha_{t}^{3/2}s_{t-1}^\star\big(Y_{t}^{\star,\md}\big)- s_t\big(X_t + (1-\alpha_t)Z^\md_t \big) + s_t^\star(X_t) \Big\|_2^2 \bigg] \notag\\
	&\qquad \lesssim \underbrace{\mathbb{E}_{\cA_t}\bigg[\Big\| \alpha_{t}^{3/2}s_{t-1}^\star\big(Y_{t}^{\star,\md}+\big(1-\alpha_t)Z^\md_t \big) - \alpha_{t}^{3/2}s_{t-1}^\star\big(Y^{\star,\md}_t\big) - s_t^\star\big(X_t + (1-\alpha_t)Z^\md_t \big) + s_t^\star(X_t) \Big\|_2^2\bigg]}_{\mathrm{(I)}} \notag\\
	& \qquad\quad + \underbrace{\mathbb{E}_{\cA_t\cap\cB_t}\bigg[\Big\|\alpha_{t}^{3/2}s_{t-1}\big(Y_{t}^{\md}\big)  - \alpha_{t}^{3/2}s_{t-1}^\star\big(Y_{t}^{\md}\big) - s_t\big(X_t + (1-\alpha_t)Z^\md_t \big) + s_t^\star\big(X_t + (1-\alpha_t)Z^\md_t \big)\Big\|_2^2\bigg]}_{\mathrm{(II)}} \notag\\
	& \qquad\quad + \underbrace{\mathbb{E}_{\cA_t\cap\cB_t}\bigg[\Big\|\alpha_{t}^{3/2}s_{t-1}^\star\big(Y^\md_t\big) - \alpha_{t}^{3/2}s_{t-1}^\star\big(Y_{t}^{\star,\md}+\big(1-\alpha_t)Z^\md_t \big) \Big\|_2^2\bigg]}_{\mathrm{(III)}}.\label{eq:clip-Psi-dist-term3-temp}
\end{align}
This decomposition yields three additional terms that require control, and we will bound each of them individually.

To control term (I), we can exploit some sort of continuity of the true score function $s_t^\star$. This is accomplished in Lemma~\ref{lem:random} below.
\begin{lemma} \label{lem:random}
Term (I) defined in \eqref{eq:clip-Psi-dist-term3-temp} satisfies
\begin{align}
\mathrm{(I)} 
& \lesssim \bigg(\frac{1-\alpha_t}{1-\overline{\alpha}_t}\bigg)^4(d\log T)^5 \label{eq:lemma-random-claim}.
\end{align}
\end{lemma}
\begin{proof}
	See Appendix \ref{sec:proof-lem:random}.
\end{proof}

Next, observe that terms (II) and (III) originate from imperfect score estimates evaluated at several auxiliary random variables. As discussed earlier, although Assumption~\ref{assu:score-error} only provides  score estimation guarantees when the score estimator is applied to $X_t$, the densities of these auxiliary random variables are comparable to that of $X_t$ on the event $\cB_t$. Hence, one can still use a change of measure to leverage Assumption~\ref{assu:score-error} and control terms (II) and (III), as formalized in Lemma~\ref{lem:estimation} below.
\begin{lemma} \label{lem:estimation}
Under Assumption~\ref{assu:score-error} on score matching, terms (II) and (III) defined in \eqref{eq:clip-Psi-dist-term3-temp} satisfy
\begin{align}
\mathrm{(II)} 
&\lesssim \varepsilon_{t-1}^2 + \varepsilon_t^2, \label{eq:score-error-1}
\end{align}
and
\begin{align}
\mathrm{(III)} 
  \lesssim \bigg(\frac{(1-\alpha_t)d\log T}{1-\overline{\alpha}_t}\bigg)^2  \varepsilon_t^2
+\frac{d}{(1-\ol\alpha_{t})T^{10}}. \label{eq:score-error-2}
\end{align}

\end{lemma}
\begin{proof}
	See Appendix~\ref{sec:proof-lem:estimation}.
\end{proof}

Consequently, substituting \eqref{eq:lemma-random-claim}--\eqref{eq:score-error-2} into \eqref{eq:clip-Psi-dist-term3-temp} yields the claimed bound:
\begin{align*}
& \mathbb{E}_{\cA_t\cap\cB_t}\bigg[\Big\| \alpha_{t}^{3/2}s_{t-1}\big(Y_{t}^{\md}(X_t)\big) - s_t\big(X_t + (1-\alpha_t)Z^\md_t \big) - \Psi_t(X_t) \Big\|_2^2\bigg]
\leq \mathrm{(I)} + \mathrm{(II)} + \mathrm{(III)} \nonumber \\
&\qquad \lesssim (1-\alpha_t)^4\bigg(\frac{d}{1-\overline{\alpha}_t}\bigg)^5
+ \bigg(\frac{1-\alpha_t}{1-\overline{\alpha}_t}\bigg)^4(d\log T)^5
+ \varepsilon_{t-1}^2 + \varepsilon_t^2 + \bigg(\frac{(1-\alpha_t)d\log T}{1-\overline{\alpha}_t}\bigg)^2  \varepsilon_t^2+\frac{d}{(1-\ol\alpha_{t})T^{10}} \nonumber \\
&\qquad \lesssim \frac{1}{1-\ol\alpha_t} \bigg(\frac{1-\alpha_t}{1-\overline{\alpha}_t}\bigg)^4(d\log T)^5
+ \varepsilon_{t-1}^2 +  \varepsilon_t^2 +\frac{1}{T^{10}}\frac{d}{1-\ol\alpha_{t}},
\end{align*}
where the last line uses \eqref{eq:1-alpha-1-olalpha} in Lemma~\ref{lemma:step-size} that $(1-\alpha)/(1-\ol\alpha_t)\lesssim \log T/T$ and the condition that $d \log^2 T = o(T)$.

\subsection{Proof of Lemma~\ref{lem:random}}
\label{sec:proof-lem:random}

Let $J_{t}(x) \defn \frac{\partial}{\partial x}s_{t}^{\star}(x)\in\bR^{d\times d}$ be the Jacobian of the score function $s_t^\star(x)$, which can be expressed as
\begin{align}
J_{t}(x)
&= -\frac{1}{1-\overline{\alpha}_t}I_d + \frac{1}{(1-\overline{\alpha}_t)^2} \Big(
\bE\big[ (X_t-\sqrt{\ol\alpha_t}X_0)(X_t-\sqrt{\ol\alpha_t}X_0)^\top \mid X_t = x \big] \nonumber \\
& \hspace{12.5em} - \bE\big[ X_t-\sqrt{\ol\alpha_t}X_0 \mid X_t = x \big]\bE\big[ X_t-\sqrt{\ol\alpha_t}X_0 \mid X_t = x \big]^\top \Big)
\label{eq:Jacob-expression}.
\end{align}

Recall that $\cA_t \defn \{X_t\in\cE_t\}$, $\mathbb{E}_{\cA_t}[\cdot] \defn \mathbb{E}\big[\cdot \ind\{\cA_t\}\big]$, and $Z\defn Z^\md_t\sim\cN(0,I_d)$. We can then use the triangle inequality to bound
\begin{align*}
\mathbb{E}_{\cA_t}&\bigg[\Big\|\alpha_{t}^{3/2}s_{t-1}^{\star}\big(Y_{t}^{\star,\md} + (1-\alpha_t)Z\big) - \alpha_{t}^{3/2}s_{t-1}^{\star}\big(Y_{t}^{\star,\md}\big) - s_t^{\star}(X_t + (1-\alpha_t)Z)  + s_t^{\star}(X_t)\Big\|_2^2 \bigg] \\
&\lesssim \alpha_t^3 \mathbb{E}_{\cA_t}\bigg[\Big\|s_t^{\star}(Y_{t}^{\star,\md} + (1-\alpha_t)Z) - s_t^{\star}\big(Y_{t}^{\star,\md}\big) - (1-\alpha_t)J_t\big(Y_{t}^{\star,\md}\big)Z\Big\|_2^2\bigg] \\
&\qquad+ \mathbb{E}_{\cA_t}\bigg[\Big\|s_t^{\star}(X_t + (1-\alpha_t)Z) - s_t^{\star}(X_t) - (1-\alpha_t)J_t(X_t)Z\Big\|_2^2\bigg] \\
&\qquad+ (1-\alpha_t)^2\mathbb{E}_{\cA_t}\bigg[\Big\|\Big(\alpha_{t}^{3/2}J_{t-1}\big(Y_{t}^{\star,\md}\big) - J_t(X_t)\Big)Z\Big\|_2^2\bigg].
\end{align*}

\begin{itemize}
	\item Regarding the last term, we claim that on the event $\cA_t$, one has
\begin{align}
\label{eq:J-Y-star-mid-X-opnorm}
\Big\|\alpha_{t}^{3/2}J_{t-1}\big(Y_{t}^{\star,\md}\big) - J_t(X_t)\Big\|
\lesssim (1-\alpha_t)\bigg(\frac{d\log T}{1-\overline{\alpha}_t}\bigg)^2,
\end{align}
with the proof postponed to the end of this section. As a result, one can obtain
\begin{align*}
	\mathbb{E}_{\cA_t}\bigg[\Big\|\Big(\alpha_{t}^{3/2}J_{t-1}\big(Y_{t}^{\star,\md}\big) - J_t(X_t)\Big)Z\Big\|_2^2\bigg] 
	& \leq \mathbb{E}_{\cA_t}\bigg[\Big\|\alpha_{t}^{3/2}J_{t-1}\big(Y_{t}^{\star,\md}\big) - J_t(X_t)\Big\|^2 \bigg]\bE\big[\|Z\|_2^2] \\
	& \lesssim (d\log T)^5 \frac{(1-\alpha_t)^2}{(1-\overline{\alpha}_t)^4}.
\end{align*}

	\item As for the second term, we can express it as
\begin{align*}
s_t^{\star}(X_t + (1-\alpha_t)Z) - s_t^{\star}(X_t) -  (1-\alpha_t)J_t(X_t)Z = \int_0^1 \Big(J_t\big(X_t + \gamma(1-\alpha_t)Z\big) - J_t(X_t)\Big)(1-\alpha_t)Z\diff \gamma.
\end{align*}
By applying a similar argument as for \eqref{eq:J-Y-star-mid-X-opnorm}, it can be shown that on the event $X_t\in\cE_t$, the following holds for any $\gamma\in[0,1]$:
\begin{align*}
\mathbb{E}_{\cA_t}\Big[\big\|J_t\big(X_t + \gamma(1-\alpha_t)Z\big) - J_t(X_t)\big\|^4\Big]
&= \mathbb{E}_{\cA_t}\Bigg[\bigg\|\int_0^\gamma\frac{\partial}{\partial \gamma'}J_t\big(X_t + \gamma'(1-\alpha_t)Z\big)\diff \gamma'\bigg\|^4\Bigg] \\
&\lesssim (1-\alpha_t)^4\bigg(\frac{d\log T}{1-\overline{\alpha}_t}\bigg)^8.
\end{align*}
%
We can subsequently use the Cauchy-Schwartz inequality to derive
\begin{align*}
&\mathbb{E}_{\cA_t}\Big[\big\|s_t^{\star}\big(X_t + (1-\alpha_t)Z\big) - s_t^{\star}(X_t) - (1-\alpha_t)J_t(X_t)Z\big\|_2^2\Big] \\
&\quad \lesssim (1-\alpha_t)^2 \int_0^1 \sqrt{\bE_{\cA_t}\Big[ \big\|J_t(X_t + \gamma(1-\alpha_t)Z) - J_t(X_t)\big\|^4 \Big] \bE \big[\|Z\|_2^4\big] } \diff \gamma \\
&\quad \lesssim \bigg(\frac{1-\alpha_t}{1-\overline{\alpha}_t}\bigg)^4(d\log T)^5.
\end{align*}

	\item Similarly, the last term can be controlled by
\begin{align*}
\mathbb{E}_{\cA_t}\bigg[&\Big\|s_{t-1}^{\star}\big(Y_{t}^{\star,\md} + (1-\alpha_t)Z\big) - s_{t-1}^{\star}\big(Y_{t}^{\star,\md}\big) - (1-\alpha_t)J_{t-1}\big(Y_{t}^{\star,\md}\big)Z\Big\|_2^2 \bigg] \\
& \lesssim (1-\alpha_t)^2  \int_0^1 \sqrt{\bE_{\cA_t}\Big[\big\|J_t\big(Y_{t}^{\star,\md} + \gamma(1-\alpha_t)Z\big) - J_t\big(Y_{t}^{\star,\md}\big)\big\|^4 \Big] \bE \big[\|Z\|_2^4\big]} \diff \gamma \\
&  \lesssim \bigg(\frac{1-\alpha_t}{1-\overline{\alpha}_t}\bigg)^4(d\log T)^5.
\end{align*}
\end{itemize}

Combining these relations leads to our desired result immediately.

\paragraph{Proof of Claim~\eqref{eq:J-Y-star-mid-X-opnorm}.}

In view of the expression of the Jacobian matrix \eqref{eq:Jacob-expression} and the triangle inequality, we shall control the spectral norms of the three terms separately.

First, it is easy to bound
\begin{align}
	\bigg|\frac{{\alpha}_t^{3/2}}{1-\overline{\alpha}_{t-1}}-\frac{1}{1-\overline{\alpha}_t}\bigg| 
	& \leq \frac{{1-\alpha}_t^{3/2}}{1-\overline{\alpha}_{t-1}}+ \frac{1}{1-\overline{\alpha}_{t-1}}-\frac{1}{1-\overline{\alpha}_t} \notag \\
	& \numpf{i}{\leq} \frac32\frac{1-\alpha_t}{1-\overline{\alpha}_{t-1}} + \frac{\ol\alpha_{t-1}(1-\alpha_t)}{(1-\overline{\alpha}_{t-1})(1-\overline{\alpha}_{t})}
	  \numpf{ii}{\lesssim} \frac{1-\alpha_t}{(1-\ol\alpha_t)^2}, \label{eq:Jacob-diff-term1}
\end{align}
where (i) uses $1-\alpha_t = o(1)$ by \eqref{eq:alpha-t-lb} and $(1-x)^{3/2}\geq 1-3x/2$ for all $x\in(0,1)$.

Next, let us consider the second moment term.
%
Recall $y_{t}^{\star,\md} = \frac{1}{\sqrt{\alpha_{t}}}\big(x_{t} + \frac{1-\alpha_t}{2\alpha_t}s_t^{\star}(x_t)\big)$ defined in \eqref{eq:Y-star-defn-1/2}. 
For convenience of notation, we denote
\begin{align*}
	\Sigma_0 &\defn \int p_{X_0 \mid X_{t}}(x_0 \mid x_t)\big(x_t - \sqrt{\overline{\alpha}_t}x_0\big)\big(x_{t} - \sqrt{\overline{\alpha}_t}x_0\big)^{\top}\diff x_0; \\
	\Sigma_1 &\defn \int p_{X_0 \mid X_{t-1}}(x_0 \mid y_{t}^{\star,\md})(x_t - \sqrt{\overline{\alpha}_{t}}x_0)(x_t - \sqrt{\overline{\alpha}_{t}}x_0)^{\top}\diff x_0; \\
	\eta_1 & \defn \int p_{X_0 \mid X_{t-1}}(x_0 \mid y_{t}^{\star,\md}) \big(x_t - \sqrt{\overline{\alpha}_{t}}x_0\big)\diff x_0.
\end{align*}
Straightforward computation yields
\begin{align}
&\frac{\alpha_{t}^{3/2}}{(1-\overline{\alpha}_{t-1})^2}\int_{x_0} p_{X_0 \mid X_{t-1}}(x_0 \mid y_{t}^{\star,\md}) \big(y_{t}^{\star,\md} - \sqrt{\overline{\alpha}_{t-1}}x_0\big)\big(y_{t}^{\star,\md} - \sqrt{\overline{\alpha}_{t-1}}x_0\big)^{\top}\diff x_0 \nonumber \\
	&\qquad -\frac{1}{(1-\overline{\alpha}_{t})^2} \int_{x_0} p_{X_0 \mid X_{t}}(x_0 \mid x_t)\big(x_t - \sqrt{\overline{\alpha}_t}x_0\big)\big(x_{t} - \sqrt{\overline{\alpha}_t}x_0\big)^{\top}\diff x_0 \nonumber \\
& \qquad \qquad = \underbrace{\frac{\sqrt{\alpha_{t}}}{(1-\overline{\alpha}_{t-1})^2} \Sigma_1 -\frac{1}{(1-\overline{\alpha}_{t})^2} \Sigma_0}_{(\mathrm{I})} + \underbrace{\frac{(1-\alpha_t)^2}{4\alpha_t^{3/2}(1-\overline{\alpha}_{t-1})^2}s_t^{\star}(x_t)s_t^{\star}(x_t)^{\top}}_{(\mathrm{II})} \nonumber \\
& \qquad \qquad  \quad + \underbrace{\frac{1-\alpha_t}{2\sqrt{\alpha_{t}}(1-\overline{\alpha}_{t-1})^2} \big(\eta_1 s_t^{\star}(x_t)^{\top} + s_t^{\star}(x_t)\eta_1^\top \big)}_{(\mathrm{III})}. \label{eq:Jacob-cov-temp}
\end{align}
In what follows, we shall control these quantities separately.
\begin{itemize}
	\item Let us start with (I).
One can use the Bayes formula to rewrite
\begin{align}
p_{X_0\mid X_{t-1}}(x_0 \mid y_{t}^{\star,\md})
= \frac{p_{X_0}(x_0)\exp\Big(-\frac{1}{2(\alpha_t-\overline{\alpha}_{t})}\big\|x_t+\frac{1-\alpha_t}{2\alpha_t}s_t^{\star}(x_t) - \sqrt{\overline{\alpha}_{t}}x_0\big\|_2^2\Big)}
{\int p_{X_0}(x_0)\exp\Big(-\frac{1}{2(\alpha_t-\overline{\alpha}_{t})}\big\|x_t+\frac{1-\alpha_t}{2\alpha_t}s_t^{\star}(x_t) - \sqrt{\overline{\alpha}_{t}}x_0\big\|_2^2\Big)\diff  x_0}. \label{eq:y-star-1/2-bayes}
\end{align}
Our goal is to show that $p_{X_0\mid X_{t-1}}(x_0 \mid y_{t}^{\star,\md})$ is sufficiently close to
\begin{align*}
	p_{X_0\mid X_{t}}(x_0 \mid x_t)
= \frac{p_{X_0}(x_0)\exp\Big(-\frac{1}{2(1-\overline{\alpha}_{t})}\big\|x_t - \sqrt{\overline{\alpha}_{t}}x_0\big\|_2^2\Big)}
{\int p_{X_0}(x_0)\exp\Big(-\frac{1}{2(1-\overline{\alpha}_{t})}\big\|x_t - \sqrt{\overline{\alpha}_{t}}x_0\big\|_2^2\Big)\diff  x_0}.
\end{align*}
Towards this, we know from \eqref{eq:bound_cond_score_fn} in Lemma~\ref{lem:bound_cond} that $\|s_t^{\star}(x_t)\|_2^2\lesssim{(d\log T)/{(1-\ol{\alpha}_t})}$ when $x_t \in \cE_t$. This implies that
\begin{align}
& \frac{1}{\alpha_t-\overline{\alpha}_{t}}\Big\|x_t+\frac{1-\alpha_t}{2\alpha_t}s_t^{\star}(x_t) - \sqrt{\overline{\alpha}_{t}}x_0\Big\|_2^2 - \frac{1}{1-\overline{\alpha}_{t}}\big\|x_t - \sqrt{\overline{\alpha}_{t}}x_0\big\|_2^2 \notag\\
& \qquad \numpf{i}{\leq} \bigg(-\frac{1}{1-\ol{\alpha}_t}+\frac{1}{\alpha_t-\overline{\alpha}_{t}}+\frac{1-\alpha_t}{(\alpha_t-\overline{\alpha}_{t})(1-\ol\alpha_t)}\bigg)\big\|x_t - \sqrt{\overline{\alpha}_{t}}x_0\big\|_2^2  \notag\\
& \qquad\quad + \frac{(1-\alpha_t)^2+{(1-\alpha_t)}{(1-\ol\alpha_t)}}{{\alpha}_t(1-\ol{\alpha}_{t-1})(2\alpha_t)^2}\big\|s_t^{\star}(x_t)\big\|_2^2 \notag\\
& \qquad \numpf{ii}{\lesssim} \frac{1-\alpha_t}{1-\ol\alpha_t} \frac{1}{1-\ol\alpha_t}\big\|x_t - \sqrt{\overline{\alpha}_{t}}x_0\big\|_2^2 +  \frac{(1-\alpha_t)d\log T}{1-\ol\alpha_t} , \label{eq:pdf-cond-y-mid-approx}
\end{align}	
where (i) uses the Cauchy-Schwartz inequality, and (ii) follows from \eqref{eq:alpha-t-lb}--\eqref{eq:1-olalpha-O1} from Lemma~\ref{lemma:step-size}.
Define the set
\begin{align*}
	\cG_t \defn \Big\{ x_0\in\bR^d \colon \big\|x_t - \sqrt{\ol\alpha_t} x_0 \big\|_2 \leq C_6 \sqrt{(1-\ol\alpha_t)d\log T} \Big\}.
\end{align*}
for some sufficiently large absolute constant $C_6>0$. Using \eqref{eq:pdf-cond-y-mid-approx}, we can bound the integral over $\cG_t$ by
\begin{align*}
	&\int_{\cG_t}  p_{X_0}(x_0)\exp\bigg(-\frac{1}{2(\alpha_t-\overline{\alpha}_{t})}\Big\|x_t+\frac{1-\alpha_t}{2\alpha_t}s_t^{\star}(x_t) - \sqrt{\overline{\alpha}_{t}}x_0\Big\|_2^2\bigg)\diff  x_0 \\
	&\qquad \leq \int p_{X_0}(x_0)\exp\Bigg(-\frac{1}{2(1-\overline{\alpha}_{t})}\big\|x_t - \sqrt{\overline{\alpha}_{t}}x_0\big\|_2^2 - O\bigg(\frac{(1-\alpha_t)d\log T}{1-\ol\alpha_t}\bigg)\Bigg)\diff  x_0 \\
	&\qquad \numpf{i}{=} \Bigg(1+O\bigg(\frac{(1-\alpha_t)d\log T}{1-\ol\alpha_t}\bigg)\Bigg)\int p_{X_0}(x_0)\exp\bigg(-\frac{1}{2(1-\overline{\alpha}_{t})}\big\|x_t - \sqrt{\overline{\alpha}_{t}}x_0\big\|_2^2\bigg)\diff  x_0 \\
	&\qquad \numpf{ii}{=} \Bigg(1+O\bigg(\frac{(1-\alpha_t)d\log T}{1-\ol\alpha_t}\bigg)\Bigg) \big(2\pi(1-\ol{\alpha}_t)\big)^{d/2} p_{X_t}(x_t),
\end{align*}
where (i) holds due to \eqref{eq:1-alpha-1-olalpha} that $(1-\alpha_t)/(1-\ol\alpha_t)\lesssim \log T/T$ and $d\log^2 T = o(T)$; (ii) is true as $X_t\mid X_0\sim\cN\big(\sqrt{\ol\alpha_t}X_0,(1-\ol\alpha_t)I_d\big)$.
Meanwhile, the integral over $\cG_t^\setc$ can be controlled by
\begin{align*}
&\int_{\cG_t^\setc}  p_{X_0}(x_0)\exp\bigg(-\frac{1}{2(\alpha_t-\overline{\alpha}_{t})}\Big\|x_t+\frac{1-\alpha_t}{2\alpha_t}s_t^{\star}(x_t) - \sqrt{\overline{\alpha}_{t}}x_0\Big\|_2^2\bigg)\diff  x_0 \\
	&\qquad \numpf{i}{=} \int_{\cG_t^\setc} p_{X_0}(x_0)\exp\bigg\{-\bigg[1+O\bigg(\frac{1-\alpha_t}{1-\ol\alpha_t}\bigg)\bigg]\frac{1}{2(1-\overline{\alpha}_{t})} \big\|x_t - \sqrt{\overline{\alpha}_{t}}x_0\big\|_2^2 +O\bigg(\frac{(1-\alpha_t)d\log T}{1-\ol\alpha_t}\bigg)\bigg\}\diff  x_0 \\
	&\qquad \numpf{ii}{\leq} p_{X_t}(x_t) \int_{\cG_t^\setc} p_{X_0}(x_0)\exp\bigg\{-\big(1-o(1)\big)\frac12C_6^2 d\log T+O\bigg(\frac{d\log^2T}{T}\bigg)+C_2 d \log T\bigg\}\diff  x_0 \\
	&\qquad \numpf{iii}{=} p_{X_t}(x_t) \exp\big(-\Omega(d\log T)\big).
\end{align*}
where (i) uses \eqref{eq:pdf-cond-y-mid-approx}; (ii) arises from $-\log p_{X_t}(x_t) \leq C_2 d\log T$, the definition of $\cG_t$ and $(1-\alpha_t)/(1-\ol\alpha_t)\lesssim \log T/T=o(1)$ by \eqref{eq:1-alpha-1-olalpha} from Lemma~\ref{lemma:step-size}, and (iii) holds provided $C_6^2/C_2$ is sufficiently large and $d\log^2 T = o(T)$.
Combining these two observations reveals that
\begin{align}
&\int p_{X_0}(x_0)\exp\bigg(-\frac{1}{2(\alpha_t-\overline{\alpha}_{t})}\Big\|x_t+\frac{1-\alpha_t}{2\alpha_t}s_t^{\star}(x_t) - \sqrt{\overline{\alpha}_{t}}x_0\Big\|_2^2\bigg)\diff  x_0 \nonumber \\ 
& \qquad = p_{X_t}(x_t) \exp\big(-\Omega(d\log T)\big) + \bigg\{1+O\bigg(\frac{(1-\alpha_t)d\log T}{1-\ol\alpha_t}\bigg)\bigg\} \big(2\pi(1-\ol{\alpha}_t)\big)^{d/2} p_{X_t}(x_t) \nonumber \\
& \qquad = \bigg\{1+O\bigg(\frac{(1-\alpha_t)d\log T}{1-\ol\alpha_t}\bigg)\bigg\} \big(2\pi(1-\ol{\alpha}_t)\big)^{d/2} p_{X_t}(x_t). \label{eq:Y-star-12-pdf}
\end{align}
Similarly, we can also derive 
\begin{align*}
	&\int_{\cG_t}  p_{X_0}(x_0)\exp\bigg(-\frac{1}{2(\alpha_t-\overline{\alpha}_{t})}\Big\|x_t+\frac{1-\alpha_t}{2\alpha_t}s_t^{\star}(x_t) - \sqrt{\overline{\alpha}_{t}}x_0\Big\|_2^2\bigg) \big(x_t - \sqrt{\overline{\alpha}_{t}}x_0\big)\big(x_t - \sqrt{\overline{\alpha}_{t}}x_0\big)^\top\diff  x_0 \\
	&\qquad = \bigg\{1+O\bigg(\frac{(1-\alpha_t)d\log T}{1-\ol\alpha_t}\bigg)\bigg\} \\
	& \qquad \qquad \cdot \int p_{X_0}(x_0)\exp\bigg(-\frac{1}{2(1-\overline{\alpha}_{t})}\big\|x_t - \sqrt{\overline{\alpha}_{t}}x_0\big\|_2^2\bigg) \big(x_t - \sqrt{\overline{\alpha}_{t}}x_0\big)\big(x_t - \sqrt{\overline{\alpha}_{t}}x_0\big)^\top\diff  x_0,
\end{align*}
and
\begin{align*}
\bigg\|\int_{\cG_t^\setc} & p_{X_0}(x_0)\exp\bigg(-\frac{1}{2(\alpha_t-\overline{\alpha}_{t})}\Big\|x_t+\frac{1-\alpha_t}{2\alpha_t}s_t^{\star}(x_t) - \sqrt{\overline{\alpha}_{t}}x_0\Big\|_2^2\bigg) \big(x_t - \sqrt{\overline{\alpha}_{t}}x_0\big)\big(x_t - \sqrt{\overline{\alpha}_{t}}x_0\big)^\top \diff  x_0 \bigg\| \\
& \leq \int_{\cG_t^\setc} p_{X_0}(x_0)\exp\bigg(-\frac{1}{2(\alpha_t-\overline{\alpha}_{t})}\Big\|x_t+\frac{1-\alpha_t}{2\alpha_t}s_t^{\star}(x_t) - \sqrt{\overline{\alpha}_{t}}x_0\Big\|_2^2\bigg) \big\|x_t - \sqrt{\overline{\alpha}_{t}}x_0\big\|_2^2 \diff  x_0 \\
	& = \int_{\cG_t^\setc} p_{X_0}(x_0)\exp\bigg\{-\frac{1-o(1)}{2(1-\overline{\alpha}_{t})} \big\|x_t - \sqrt{\overline{\alpha}_{t}}x_0\big\|_2^2 +O\bigg(\frac{(1-\alpha_t)d\log T}{1-\ol\alpha_t}\bigg)\bigg\} \big\|x_t - \sqrt{\overline{\alpha}_{t}}x_0\big\|_2^2\diff  x_0 \\
	& = p_{X_t}(x_t) \exp\big(-\Omega(d\log T)\big).
\end{align*}

Combined with \eqref{eq:y-star-1/2-bayes}, these reveal that
\begin{align*}
	\|\Sigma_1 - \Sigma_0 \| \leq \exp\big(-\Omega(d\log T)\big) + O\bigg(\frac{(1-\alpha_t)d\log T}{1-\overline{\alpha}_{t}}\bigg) \| \Sigma_0 \|.
\end{align*}
Moreover, we know that
\begin{align*}
	\|\Sigma_0\| & = \Big\| \bE\big[(X_t - \sqrt{\ol{\alpha}_t}X_0)(X_t - \sqrt{\ol{\alpha}_t}X_0)^\top \,\big|\,X_t = x_t\big] \Big\| \\
	& \numpf{i}{\leq} \bE\Big[ \big\| X_t - \sqrt{\ol{\alpha}_t}X_0  \big\|_2^2 \,\big|\,X_t = x_t\Big] \numpf{ii}{\lesssim} (1-\ol\alpha_t)d\log T,
\end{align*}
where (i) applies Jensen's inequality and (ii) follows from \eqref{eq:bound_cond} in Lemma~\ref{lem:bound_cond}.
Therefore, we obtain
\begin{align}
	\|(\mathrm{I})\|& = \bigg\| \frac{\sqrt{\alpha_{t}}}{(1-\overline{\alpha}_{t-1})^2}\Sigma_1 - \frac{1}{(1-\ol\alpha_t)^2}\Sigma_0 \bigg\| \notag \\
	& \leq \frac{\sqrt{\alpha_{t}}}{(1-\overline{\alpha}_{t-1})^2}\|\Sigma_1 - \Sigma_0 \| + \bigg|\frac{\sqrt{\alpha_{t}}}{(1-\overline{\alpha}_{t-1})^2} - \frac{1}{(1-\ol\alpha_t)^2}\bigg| \|\Sigma_0 \| \nonumber \\
	& \lesssim \frac{\sqrt{\alpha_{t}}}{(1-\overline{\alpha}_{t-1})^2} \bigg(\exp\big(-\Omega(d\log T)\big) + \frac{(1-\alpha_t)d\log T}{1-\overline{\alpha}_{t}} (1-\ol\alpha_t)d\log T \bigg) \nonumber \\
	& \quad + \bigg|\frac{\sqrt{\alpha_{t}}}{(1-\overline{\alpha}_{t-1})^2} - \frac{1}{(1-\ol\alpha_t)^2}\bigg| (1-\ol\alpha_t)d\log T \nonumber \\
	& \lesssim (1-\alpha_t)\bigg(\frac{d\log T}{1-\overline{\alpha}_{t-1}}\bigg)^2 + \frac{1-\alpha_t}{(1-\ol\alpha_{t})^2} (1-\ol\alpha_t)d\log T \nonumber \\
	& \asymp (1-\alpha_t)\bigg(\frac{d\log T}{1-\overline{\alpha}_{t}}\bigg)^2, \label{eq:Jacob-diff-cov-term1}
\end{align}
where we use $\sqrt{1-x}\leq 1-x/2$ for $x\in[0,1]$ and \eqref{eq:1-olalpha-O1}.

\item
Notice that $\int p_{X_0 \mid X_{t}}(x_0 \mid x_t) \big(x_t - \sqrt{\overline{\alpha}_{t}}x_0\big)\diff x_0 = -(1-\ol\alpha_t) s_t^\star(x_t)$.
Apply the same reasoning above, we also know that
\begin{align}
	\|\mathrm{(II)}\|&=\big\|\eta_1 +(1-\ol\alpha_t) s_t^\star(x_t) \big\|_2 \leq \exp\big(-\Omega(d\log T)\big)  + O\bigg(\frac{(1-\alpha_t)d\log T}{1-\overline{\alpha}_{t}}\bigg) \big\|(1-\ol\alpha_t) s_t^\star(x_t) \big\|_2 \nonumber \\
	& \lesssim \frac{(1-\alpha_t)(d\log T)^{3/2}}{\sqrt{1-\overline{\alpha}_{t}}}, \label{eq:eta-l2-dist-ub}
\end{align}
where the last step uses \eqref{eq:bound_cond} in Lemma~\ref{lem:bound_cond}. In particular, one has
\begin{align}
	\|\eta_1\|_2 &\lesssim \bigg\{1 + O\bigg(\frac{(1-\alpha_t)d\log T}{1-\overline{\alpha}_{t}}\bigg)\bigg\} (1-\ol\alpha_t) \big\|s_t^\star(x_t)\big\|_2 + \exp\big(-\Omega(d\log T)\big) \nonumber \\
	&  \numpf{i}{\lesssim} (1-\ol\alpha_t) \big\|s_t^\star(x_t)\big\|_2 + \exp\big(-\Omega(d\log T)\big) \numpf{ii}{\lesssim} \sqrt{(1-\ol\alpha_t) d\log T}. \label{eq:eta-l2-ub}
\end{align}
where (i) holds due to \eqref{eq:1-alpha-1-olalpha} from Lemma~\ref{lemma:step-size} that $(1-\alpha_t)/(1-\ol\alpha_t)\lesssim \log T/T$ and $d\log^2 T = o(T)$; (ii) arises from \eqref{eq:bound_cond} in Lemma~\ref{lem:bound_cond}.
As a result, we arrive at
\begin{align}
	\frac{1-\alpha_t}{2\sqrt{\alpha_{t}}(1-\overline{\alpha}_{t-1})^2} \big\|\eta_1 s_t^{\star}(x_t)^{\top} + s_t^{\star}(x_t)\eta_1^\top \big\| & \leq \frac{1-\alpha_t}{\sqrt{\alpha_{t}}(1-\overline{\alpha}_{t-1})^2} \|\eta_1\|_2 \big\|s_t^\star(x_t)\big\|_2  \nonumber \\
	& \lesssim \frac{(1-\alpha_t)d\log T}{(1-\overline{\alpha}_{t})^2}, \label{eq:Jacob-diff-cov-term2}
\end{align}
where the final step is due to $1-\ol\alpha_t \asymp 1-\ol\alpha_{t-1}$ in \eqref{eq:1-olalpha-O1} and $\alpha_t \asymp 1$ in \eqref{eq:alpha-t-lb} from Lemma~\ref{lemma:step-size}.
\item 
In addition, we can use \eqref{eq:bound_cond} from Lemma~\ref{lem:bound_cond} to get
\begin{align}
	\|\mathrm{(III)}\| = \Bigg\| \frac{(1-\alpha_t)^2}{\alpha_t^{3/2}(1-\overline{\alpha}_{t-1})^2}s_t^{\star}(x_t)s_t^{\star}(x_t)^{\top} \Bigg\| \lesssim \bigg(\frac{1-\alpha_t}{1-\ol\alpha_{t-1}}\bigg)^2 \big\|s_t^{\star}(x_t)\big\|_2^2 \lesssim \frac{(1-\alpha_t)^2d\log T}{(1-\ol\alpha_t)^3}, \label{eq:Jacob-diff-cov-term3}
\end{align}
where we use $1-\ol\alpha_t \asymp 1-\ol\alpha_{t-1}$ in \eqref{eq:1-olalpha-O1} and $\alpha_t \asymp 1$ in \eqref{eq:alpha-t-lb} in Lemma~\ref{lemma:step-size}.
\item Putting \eqref{eq:Jacob-diff-cov-term1}, \eqref{eq:Jacob-diff-cov-term2}, and \eqref{eq:Jacob-diff-cov-term3} together into \eqref{eq:Jacob-cov-temp} reveals that
\begin{align}
	& \frac{\alpha_{t}^{3/2}}{(1-\overline{\alpha}_{t-1})^2}\int_{x_0} p_{X_0 \mid X_{t-1}}(x_0 \mid y_{t}^{\star,\md}) \big(y_{t}^{\star,\md} - \sqrt{\overline{\alpha}_{t-1}}x_0\big)\big(y_{t}^{\star,\md} - \sqrt{\overline{\alpha}_{t-1}}x_0\big)^{\top}\diff x_0 \nonumber \\
	&\qquad -\frac{1}{(1-\overline{\alpha}_{t})^2} \int_{x_0} p_{X_0 \mid X_{t}}(x_0 \mid x_t)\big(x_t - \sqrt{\overline{\alpha}_t}x_0\big)\big(x_{t} - \sqrt{\overline{\alpha}_t}x_0\big)^{\top}\diff x_0 \leq \|\mathrm{(I)}\| + \|\mathrm{(II)}\| + \|\mathrm{(III)}\| \nonumber \\
	& \qquad \qquad 
	  \lesssim (1-\alpha_t)\bigg(\frac{d\log T}{1-\overline{\alpha}_{t}}\bigg)^2 + \frac{(1-\alpha_t)d\log T}{(1-\overline{\alpha}_{t})^2} + \frac{(1-\alpha_t)^2d\log T}{(1-\ol\alpha_t)^{3}} \nonumber \\
	& \qquad\qquad  \asymp (1-\alpha_t)\bigg(\frac{d\log T}{1-\overline{\alpha}_{t}}\bigg)^2, \label{eq:Jacob-diff-term2}
\end{align}
where the last step follows from \eqref{eq:1-alpha-1-olalpha} from Lemma~\ref{lemma:step-size} that $(1-\alpha_t)/(1-\ol\alpha_t)\lesssim \log T/T = o(1)$.

\end{itemize}

Third, we can bound
\begin{align}
	& \bigg\|\frac{\alpha_t^{3/2}}{(1-\ol\alpha_{t-1})^2}  \eta_1\eta_1^\top - s_t^\star(x_t)s_t^\star(x_t)^\top \bigg\| \\
	&\qquad \lesssim \bigg|\frac{\alpha_t^{3/2}}{(1-\ol\alpha_{t-1})^2}-\frac{1}{(1-\ol\alpha_{t})^2}\bigg|\|\eta_1 \|_2^2
	 + \frac{1}{(1-\ol\alpha_{t})^2} \big\|\eta_1+(1-\ol\alpha_t) s_t^\star(x_t) \big\|_2 \Big(\|\eta_1\|_2+ \big\|(1-\ol\alpha_t) s_t^\star(x_t)\big\|_2 \Big) \nonumber \\
	&\qquad \numpf{i}{\lesssim} \frac{1-\alpha_t}{(1-\ol\alpha_t)^3} (1-\ol\alpha_t)d\log T + \frac{1}{(1-\ol\alpha_{t})^2}\frac{(1-\alpha_t)(d\log T)^{3/2}}{\sqrt{1-\overline{\alpha}_{t}}} \sqrt{(1-\ol\alpha_t)d\log T}   \nonumber \\
	&\qquad \asymp (1-\alpha_t)\bigg(\frac{d\log T}{1-\ol\alpha_t}\bigg)^2, \label{eq:Jacob-diff-term3}
\end{align}
Here, (i) use \eqref{eq:eta-l2-dist-ub}, \eqref{eq:eta-l2-ub}, and the following bound:
\begin{align*}
	\bigg|\frac{{\alpha}_t^{3/2}}{(1-\overline{\alpha}_{t-1})^2}-\frac{1}{(1-\overline{\alpha}_t)^2}\bigg| 
	& \leq \frac{{1-\alpha}_t^{3/2}}{(1-\overline{\alpha}_{t-1})^2}+ \frac{1}{(1-\overline{\alpha}_{t-1})^2}-\frac{1}{(1-\overline{\alpha}_t)^2} \\
	& \numpf{i}{\leq} \frac32\frac{1-\alpha_t}{(1-\overline{\alpha}_{t-1})^2} + \frac{\ol\alpha_{t-1}(1-\alpha_t)(2-\ol\alpha_t-\ol\alpha_{t-1})}{(1-\overline{\alpha}_{t-1})^2(1-\overline{\alpha}_{t})^2} \numpf{ii}{\lesssim} \frac{1-\alpha_t}{(1-\ol\alpha_t)^3},
\end{align*}
where (i) applies $1-\alpha_t = o(1)$ by \eqref{eq:alpha-t-lb} and $(1-x)^{3/2}\geq 1-3x/2$ for all $x\in(0,1)$; (ii) arises from \eqref{eq:1-olalpha-O1} from Lemma~\ref{lemma:step-size} that $1-\ol\alpha_t\asymp 1-\ol\alpha_{t-1}$. 

Finally, putting \eqref{eq:Jacob-diff-term1}, \eqref{eq:Jacob-diff-term2}, \eqref{eq:Jacob-diff-term3} together finishes the proof of the claim.

\subsection{Proof of Lemma~\ref{lem:estimation}}
\label{sec:proof-lem:estimation}

\paragraph{Proof of Claim \eqref{eq:score-error-1}.}
Let us start with the first claim. Recall the definitions of the events $\{\cF_{t,i}\}_{i=1}^3$ and $\cB_t$ in \eqref{def:set-F} and \eqref{eq:event-pdf-ratio}, respectively. Recall the notation $Z \defn Z^\md_t$. Applying a change of measure yields
\begin{align*}
& \bE_{\cB_t} \Big[ \big\| s_t\big(X_t+(1-\alpha_t)Z\big)
-  s_t^{\star}\big(X_t+(1-\alpha_t)Z\big) \big\|_2^2\Big] \\
& \qquad = \int_{\cF_{t,1}}\big\| s_t(x)
-  s_t^{\star}(x) \big\|_2^2 \,p_{X_t+(1-\alpha_t)Z}(x) \diff x  
\numpf{i}{\leq}\int_{\cF_{t,1}}\big\| s_t(x)
-  s_t^{\star}(x) \big\|_2^2 \,2\,p_{X_{t-1}}(x) \diff x \\
& \qquad \leq 2\,\bE\Big[ \big\| s_t\big(X_{t-1}\big)
-  s_t^{\star}\big(X_{t-1}\big) \big\|_2^2\Big] \lesssim \veps^2_{t}, 
\end{align*}
where (i) arises from the construction of $\cF_{t,1}$. 
Similarly, by the definition of $\cF_{t,2}$,  we can also derive
\begin{align*}
	\bE_{\cB_t} \Big[ \big\| \alpha_{t}^{3/2} s_{t-1}\big(Y_{t}^{\md}\big)
- \alpha_{t}^{3/2} s_{t-1}^{\star}\big(Y_{t}^{\md}\big) \big\|_2^2\Big]
&  \leq 4\,\alpha_{t}^{3}\,\bE \Big[ \big\| s_{t-1}(X_{t-1})
- s_{t-1}^{\star}(X_{t-1}) \big\|_2^2\Big]
\lesssim \veps^2_{t-1}.
\end{align*}
where the last step is due to Assumption~\ref{assu:score-error} and $\alpha_t < 1$.

Putting these two bounds together finishes the proof of the first claim.

\paragraph{Proof of Claim \eqref{eq:score-error-2}.}
Turning to the second claim, let us define the set
\begin{align*}
\cC_t := \bigg\{ \big\|s_t(X_t) - s_t^{\star}(X_t)\big\|_2 \leq C_5\frac{1}{1-\alpha_t}\sqrt{(1-\overline{\alpha}_t)d\log T} \,\,\,\text{ and }\,\,\, \|Z\|_2 \leq \sqrt{d}+C_5\sqrt{\log T}\bigg\},
\end{align*}
for some absolute constant $C_5>0$.

Recall the expression of $Y^\md_t$ and $Y^{\star,\md}_t$ in \eqref{eq:sampler-mid} and \eqref{eq:Y-star-defn-1/2}, respectively.
\begin{align*}
&s_{t-1}^{\star}\big(Y^\md_t\big) - s_{t-1}^{\star}\big(Y_{t}^{\star,\md} + (1-\alpha_t)Z\big)\\
&\quad = \int_0^1 J_{t-1}\bigg(\frac{1}{\sqrt{\alpha_{t}}}\Big(X_{t} + \frac{1-\alpha_t}{2\alpha_t}\big(\gamma s_t(X_t) + (1-\gamma) s_t^{\star}(X_t)\big)\Big) + (1-\alpha_t)Z\bigg)\frac{1-\alpha_t}{2\alpha_t^{3/2}}\big(s_t(X_t) - s_t^{\star}(X_t)\big)\diff \gamma.
\end{align*}
We can then apply a similar argument for establishing \eqref{eq:J-Y-star-mid-X-opnorm} to show that on the event $\cA_t \cap \cC_t$,
\begin{align*}
\sup_{\gamma\in[0,1]}\bigg\|J_{t-1}\bigg(\frac{1}{\sqrt{\alpha_{t}}}\Big(X_{t} + \frac{1-\alpha_t}{2\alpha_t}\big(\gamma s_t(X_t) + (1-\gamma) s_t^\star(X_t)\big)\Big) + (1-\alpha_t)Z\bigg)\bigg\| 
\lesssim \frac{d\log T}{1-\overline{\alpha}_t}.
\end{align*}
This gives
\begin{align*}
	 \Big\|s_{t-1}^{\star}&\big(Y^\md_t\big) - s_{t-1}^{\star}\big(Y_{t}^{\star,\md} + (1-\alpha_t)Z\big)\Big\|_2 \\
	& \leq \frac{1-\alpha_t}{2\alpha_t^{3/2}} \max_{\gamma\in[0,1]} \bigg\|J_{t-1}\bigg(\frac{1}{\sqrt{\alpha_{t}}}\Big(X_{t} + \frac{1-\alpha_t}{2\alpha_t}\big(\gamma s_t(X_t) + (1-\gamma) s_t^\star(X_t)\big)\Big) + (1-\alpha_t)Z\bigg)\bigg\|  \big\|s_t(X_t) - s_t^{\star}(X_t) \big\|_2  \\
	& \lesssim \frac{(1-\alpha_t)d\log T}{1-\overline{\alpha}_t} \big\|s_t(X_t) - s_t^{\star}(X_t) \big\|_2,
\end{align*}
where the last step holds due to \eqref{eq:alpha-t-lb} in Lemma~\ref{lemma:step-size} that $\alpha_t \asymp 1$.
Consequently, one has
\begin{align}
	 \mathbb{E}_{\cA_t\cap\cB_t\cap\cC_t}\bigg[\Big\|s_{t-1}^{\star}\big(Y_{t}^{\md}\big) - s_{t-1}^{\star}\big(Y_{t}^{\star,\md} + (1-\alpha_t)Z\big)\Big\|_2^2\bigg]
	& \lesssim \bigg(\frac{(1-\alpha_t)d\log T}{1-\overline{\alpha}_t}\bigg)^2 \mathbb{E}\Big[\big\|s_t(X_t) - s_t^{\star}(X_t) \big\|_2^2\Big] \nonumber \\
	& \lesssim \bigg(\frac{(1-\alpha_t)d\log T}{1-\overline{\alpha}_t}\bigg)^2  \varepsilon_t^2. \label{eq:exp-tau-ub}
\end{align}
where we use Assumption~\ref{assu:score-error} in the last line.

Meanwhile, invoking Markov's inequality and the standard Gaussian concentration inequality, we can derive
\begin{align*}
\mathbb{P}\big(\cC^\setc_t\big)
& \leq \frac{(1-\alpha_t)^2}{C_5(1-\ol\alpha_t)d\log T}\bE\Big[\big\|s_t(X_t) - s_t^{\star}(X_t) \big\|_2^2\Big] + O(T^{-10}) \\ 
& \lesssim \frac{(1-\alpha_t)^2}{(1-\overline{\alpha}_t)d\log T} \varepsilon_t^2+\frac{1}{T^{10}},
\end{align*}
provided $C_5$ is sufficiently large. We can then apply H\"older's inequality to get that for any constant $k>0$, 
\begin{align*}
\mathbb{E}_{\cA_t\cap\cB_t\cap\cC_t^{\mathrm c}}\Big[\big\|s_{t-1}^{\star}\big(Y_{t}^{\md}\big)\big\|_2^2\Big] & \leq \bigg(\mathbb{E}_{\cA_t\cap\cB_t}\Big[\big\|s_{t-1}^{\star}\big(Y_{t}^{\md}\big)\big\|_2^{2k}\Big]\bigg)^{1/k}\bP\big(\cC^\setc_t\big)^{1-1/k} \\
& \numpf{i}{\leq} \bigg(4\,\bE_{\cA_t\cap\cB_t}\Big[\big\|s_{t-1}^{\star}(X_{t-1})\big\|_2^{2k}\Big]\bigg)^{1/k}\bP\big(\cC^\setc_t\big)^{1-1/k} \\
& \numpf{ii}{\lesssim} \frac{d}{1-\ol\alpha_{t-1}}\bigg(\frac{(1-\alpha_t)^2}{(1-\overline{\alpha}_t)d\log T} \varepsilon_t^2+\frac{1}{T^{10}}\bigg)^{1-1/k} \\
& \asymp \frac{(1-\alpha_t)^{2{(1-1/k})}d^{1/k}}{(1-\overline{\alpha}_t)^{2-1/k}\log^{1-1/k} T} \varepsilon_t^{2({1-1/k})}\
+\frac{d}{(1-\ol\alpha_{t})T^{10(1-1/k)}},
\end{align*}
where (i) arises from a change of measure and the definitions of $\cF_{t,2}$ and $\cB_t$ in \eqref{def:set-F} and \eqref{eq:event-pdf-ratio}, respectively; (ii) is due to \eqref{eq:s-star-l2norm-ub} in Lemma~\ref{lemma:x-star-x0-dist-l2norm}; (iii) uses \eqref{eq:1-olalpha-O1} that $1-\ol\alpha_{t-1}\asymp 1-\ol\alpha_{t}$. Taking the limit $k\rightarrow \infty$, we obtain
\begin{align*}
\mathbb{E}_{\cA_t\cap\cB_t\cap\cC_t^{\mathrm c}}\Big[\big\|s_{t-1}^{\star}\big(Y_{t}^{\md}\big)\big\|_2^2\Big] & \lesssim
 \frac{(1-\alpha_t)^{2}}{(1-\overline{\alpha}_t)^{2}\log T} \varepsilon_t^{2}\
+\frac{d}{(1-\ol\alpha_{t})T^{10}},
\end{align*}

Clearly, by the construction of $\cF_{t,3}$ and $\cB_t$, the bound also applies to $\mathbb{E}_{\cA_t\cap\cB_t\cap\cC_t^{\mathrm c}}\Big[\big\| s_{t-1}^{\star}\big(Y_{t}^{\star,\md} + (1-\alpha_t)Z^\md_t\big)\big\|_2^2\Big]$.
Therefore, we find that
\begin{align}
	& \mathbb{E}_{\cA_t\cap\cB_t\cap\cC_t^{\mathrm c}}\bigg[\Big\|s_{t-1}^{\star}\big(Y_{t}^{\md}\big) - s_{t-1}^{\star}\big(Y_{t}^{\star,\md} + (1-\alpha_t)Z^\md_t\big)\Big\|_2^2\bigg] \lesssim
\frac{(1-\alpha_t)^{2}}{(1-\overline{\alpha}_t)^{2}\log T} \varepsilon_t^{2}\
+\frac{d}{(1-\ol\alpha_{t})T^{10}} \label{eq:exp-tau-c-ub}.
\end{align}
Putting \eqref{eq:exp-tau-ub} and \eqref{eq:exp-tau-c-ub} together finishes the proof of the second claim.

\section{Proof of auxiliary lemmas}

\subsection{Proof of Lemma~\ref{lemma:step-size}}
\label{sec:pf-lemma:step-size}

\begin{itemize}
	\item We begin with proving \eqref{eq:1-alpha-1-olalpha}. By the update rule \eqref{eq:learning-rate} and $\alpha_t\in(0,1)$, it is straightforward to verify that
	\begin{align*}
		\frac{1-\alpha_t}{1-\ol\alpha_t} = C_1\frac{ \alpha_t\log T}{T} \leq C_1\frac{\log T}{T},\quad 2\leq t \le T.
	\end{align*}
	\item \eqref{eq:alpha-t-lb} follows an an immediate consequence of \eqref{eq:1-alpha-1-olalpha}: as $\ol\alpha_t = \prod_{i=1}^t \alpha_i < 1$ for all $t\in[T]$, we find
	\begin{align*}
		1-\alpha_t \leq \frac{1-\alpha_t}{1-\ol\alpha_t} \leq C_1\frac{\log T}{T},\quad 2\leq t\leq T.
	\end{align*}
	\item To prove \eqref{eq:1-olalpha-O1}, observe that
	\begin{align*}
		\frac{1-\ol\alpha_t}{1-\ol\alpha_{t-1}} = 1 + \frac{\ol\alpha_{t-1}-\ol\alpha_t}{1-\ol\alpha_{t-1}} = 1+ C_1 \frac{\ol\alpha_t\log T}{T}\frac{1-\ol\alpha_t}{1-\ol\alpha_{t-1}}, 
	\end{align*}
	where the last step uses \eqref{eq:learning-rate}. Rearranging the equation leads to
	\begin{align}
		\frac{1-\ol\alpha_t}{1-\ol\alpha_{t-1}} = \bigg(1-C_1 \frac{\ol\alpha_t\log T}{T}  \bigg)^{-1} \leq 1+2C_1 \frac{\log T}{T},
		  \label{eq:alpha-t-olalpha-t-1}
	\end{align}
	where the final step holds because $(C_1\ol\alpha_t\log T)/T \leq (C_1\log T)/T \leq 1/2$ as $\alpha_t \in (0,1)$ for all $t$, and $(1-x)^{-1}\leq 1+2x$ for $x\in[0,1/2]$.
	\item Finally, let us consider \eqref{eq:alpha-1-lb}. 
	We first claim that $\ol\alpha_{t} \ge 1/2$ for all $1\leq t \leq T/2$. Here, we assume $T$ is even for simplicity of presentation. The argument can be easily generalized to the odd case. 

	Suppose the claim is false, that is, there exists some $1\leq t \leq T/2$ such that $\ol\alpha_{t} < 1/2$. Since $\ol\alpha_{t}$ is decreasing in $t$, this means that $\ol\alpha_{t} < 1/2$ for all $t \geq T/2$. It follows from the update rule \eqref{eq:learning-rate} that 
	\begin{align*}
		\ol\alpha_{t-1}  = \ol\alpha_{t}\bigg(1+(1-\ol\alpha_t)\frac{C_1\log T}{T}\bigg) > \ol\alpha_{t}\bigg(1+\frac{C_1\log T}{2T}\bigg)
	\end{align*}
	for all $T/2 \leq t \leq T$. Combined with $\ol\alpha_{T}=T^{-C_0}$ chosen in \eqref{eq:learning-rate} and $(C_1\log T)/T \leq 1/2$, this implies that as long as $C_1/C_0$ is large enough,
	\begin{align*}
	\ol\alpha_{T/2} \ge \ol\alpha_{T}\bigg(1+\frac{C_1\log T}{2T}\bigg)^{T/2} \geq \frac{1}{T^{C_0}} \bigg(1+\frac{1}{4}\bigg)^{C_1 \log T}  > \frac{1}{2},
	\end{align*}
	since $(1+1/x)^x$ is increasing in $x$ for $x>0$, thereby leading to a contradiction. 

	This establishes the claim above, which implies that for all $1\leq t \le T/2$,
	\begin{align*}
		1-\ol\alpha_{t-1}= (1-\ol\alpha_{t})\bigg(1-\ol\alpha_t\frac{C_1\log T}{T}\bigg) \le (1-\ol\alpha_{t})\bigg(1-\frac{C_1\log T}{2T}\bigg).
	\end{align*}
	Therefore, we conclude that
	\begin{align*}
	1-\ol\alpha_{1} \le \big(1-\ol\alpha_{T/2}\big)\bigg(1-\frac{C_1\log T}{2T}\bigg)^{T/2-1} \leq \frac{1}{T^{C_1/4}},
	\end{align*}
    where we use $\ol\alpha_{T/2} \ge 1/2$, $(C_1\log T)/T\leq 1/2$, and $(1-1/x)^{x} \leq \exp(-1)$ for any $x>1$.
\end{itemize}

\subsection{Proof of Lemma \ref{lemma:mixing}}\label{sub:proof_of_lemma_ref_lemma_mixing}
Recall that $Y_T\sim\cN(0,I_d)$. We can derive
\begin{align*}
	\KL\big(p_{X_{T}}\,\|\,p_{Y_{T}}\big) &\numpf{i}{\leq} \bE_{x_0\sim p_{X_0}} \Big[ \KL\big(p_{X_{T}}(\cdot \mid x_0)\,\|\,p_{Y_{T}}(\cdot )\big)\Big] \nonumber \\
	& \numpf{ii}{=} \frac12 \bE \Big[d(1-\ol\alpha_T)-d+\big\|\sqrt{\ol\alpha_T}X_0\big\|_2^2 - d\log(1-\ol\alpha_T) \Big] \nonumber \\
	& \numpf{iii}{\leq} \frac12 \ol\alpha_T \bE\big[\| X_0\|_2^2\big] \numpf{iv}{\lesssim} T^{-(C_0-C_R)}  \numpf{v}{\lesssim} T^{-10},
\end{align*}
where (i) follows from the convexity of the KL divergence; (ii) uses the KL divergence formula for normal distributions; (iii) holds as the learning rate \eqref{eq:learning-rate} that $\ol\alpha_T = T^{-C_0}$ is sufficiently small and $\log(1-x)\geq-x$ for all $x\in[0,1/2]$; (iv) uses Assumption~\ref{assu:distribution} about the bounded second moment; (v) holds as long as $C_0$ is sufficiently large.
\subsection{Proof of Lemma~\ref{lemma:u1-u2-l2norm}}
\label{sec:pf-lemma:u1-u2-l2norm}

\paragraph{Step 1: deriving the expressions of the derivatives.}
%
Recall the definition of $\theta^\star(\tau)\defn \tilde{s}^\star(x_\tau^\star,\tau)$ in \eqref{eq:s-star-defn}, which can be expressed as
\begin{align}
\frac{\theta^\star(\tau)}{(1-\tau)^{3/2}}
&= -\frac{1}{\tau(1-\tau)}\int_{x_0} p_{\wt{X}_0 \mid \wt{X}_\tau}(x_0 \mid x^\star_\tau)\bigg(\frac{x^\star_\tau}{\sqrt{1-\tau}} - x_0\bigg) \diff x_0 \notag\\
& = -\frac{1}{\tau(1-\tau)}\frac{\int_{x_0} p_{\wt{X}_0 \mid \wt{X}_\tau}(x_0 \mid x^\star_\tau)\Big(\frac{x^\star_\tau}{\sqrt{1-\tau}} - x_0\Big) \diff x_0}{\int_{x_0} p_{\wt{X}_0 \mid \wt{X}_\tau}(x_0 \mid x_\tau^\star) \diff x_0} \label{eq:score-fn-expression}
\end{align}
Note that conditional on $\wt{X}_0$, $\wt{X}_\tau \sim \mathcal{N}(\sqrt{1-\tau}\wt{X}_0,\tau I_d)$ with density
\begin{align*}
p_{\wt{X}_\tau \mid \wt{X}_0}(x_\tau^\star \mid x_0)
= (2\pi\tau)^{-d/2}\exp\biggl(-\frac{1-\tau}{2\tau}\bigg\|\frac{x_\tau^\star}{\sqrt{1-\tau}} - x_0\bigg\|_2^2\bigg).
\end{align*}
Hence, by defining
\begin{align}
	\Delta_\delta(\tau) \defn \frac{1-\tau}{2\tau}\bigg\|\frac{x_\tau^\star}{\sqrt{1-\tau}} - x_0\bigg\|_2^2 -\frac{1-\tau-\delta}{2(\tau+\delta)}\bigg\|\frac{x_{\tau+\delta}^\star}{\sqrt{1-\tau-\delta}}-x_0\bigg\|_2^2, \label{def:delta}
\end{align}
one can express
\begin{align*}
	p_{\wt{X}_{\tau+\delta} \mid \wt{X}_0}(x_{\tau+\delta}^{\star} \mid x_0)
	& = \big(2\pi(\tau+\delta) \big)^{-d/2}\exp\biggl(-\frac{1-\tau-\delta}{2(\tau+\delta)}\bigg\|\frac{x_{\tau+\delta}^\star}{\sqrt{1-\tau-\delta}}-x_0\bigg\|_2^2\bigg) \\
	& = \bigg(\frac{\tau+\delta}{\tau}\bigg)^{-d/2} (2\pi\tau)^{-d/2}\exp\biggl(-\frac{1-\tau}{2\tau}\bigg\|\frac{x_{\tau}^\star}{\sqrt{1-\tau}}-x_0\bigg\|_2^2\bigg) \exp\big(\Delta_\delta(\tau)\big) \\
	& = \bigg(\frac{\tau+\delta}{\tau}\bigg)^{-d/2} p_{\wt{X}_\tau \mid \wt{X}_0}(x_{\tau}^{\star} \mid x_0) \exp\big(\Delta_\delta(\tau)\big).
\end{align*}
It follows from the Bayes formula that
\begin{align*}
	p_{\wt{X}_0 \mid \wt{X}_{\tau+\delta}}(x_0 \mid x_{\tau+\delta}^\star) & = \frac{p_{\wt{X}_{\tau+\delta} \mid \wt{X}_0}(x_{\tau+\delta}^{\star} \mid x_0) p_{\wt{X}_0}(x_0)}{p_{\wt{X}_{\tau+\delta}}(x_{\tau+\delta}^\star)} = \frac{p_{\wt{X}_\tau \mid \wt{X}_0}(x_{\tau}^{\star} \mid x_0) \exp\big(\Delta_\delta(\tau)\big) p_{\wt{X}_0}(x_0)}{p_{\wt{X}_\tau}(x_{\tau}^\star)} \frac{p_{\wt{X}_\tau}(x_{\tau}^\star)}{p_{\wt{X}_{\tau+\delta}}(x_{\tau+\delta}^\star)} \\
	& = p_{\wt{X}_0 \mid \wt{X}_\tau}(x_0 \mid x_{\tau}^\star) \exp\big(\Delta_\delta(\tau)\big) \bigg(\frac{\tau+\delta}{\tau}\bigg)^{-d/2} \frac{p_{\wt{X}_\tau}(x_{\tau}^\star)}{p_{\wt{X}_{\tau+\delta}}(x_{\tau+\delta}^\star)}.
\end{align*}

Therefore, plugging this into \eqref{eq:score-fn-expression} shows that for sufficiently small $\delta$,
\begin{align}
\frac{\theta^{\star}(\tau+\delta)}{(1-\tau-\delta)^{3/2}}
& = -\frac{1}{(\tau+\delta)(1-\tau-\delta)} \bigg(\frac{x_{\tau+\delta}^{\star}}{\sqrt{1-\tau-\delta}} - \frac{x_{\tau}^{\star}}{\sqrt{1-\tau}}\bigg) \nonumber \\
& \quad 
-\frac{1}{(\tau+\delta)(1-\tau-\delta)}
\frac{\int_{x_0} p_{\wt{X}_0 \mid \wt{X}_{\tau+\delta}}(x_0 \mid x_{\tau+\delta}^{\star})\Big(\frac{x_{\tau}^{\star}}{\sqrt{1-\tau}} - x_0\Big) \diff x_0}
{\int_{x_0} p_{\wt{X}_0 \mid \wt{X}_{\tau+\delta}}(x_0 \mid x_{\tau+\delta}^{\star}) \diff x_0} \notag\\
&= \underbrace{- \frac{1}{(\tau+\delta)(1-\tau-\delta)} 
\bigg(\frac{x_{\tau+\delta}^{\star}}{\sqrt{1-\tau-\delta}} - \frac{x_{\tau}^{\star}}{\sqrt{1-\tau}}\bigg)}_{\mathrm{(I)}} \nonumber \\
& \quad 
\underbrace{-\frac{1}{(\tau+\delta)(1-\tau-\delta)}
 \frac{\int_{x_0} p_{\wt{X}_0 \mid \wt{X}_\tau}(x_0 \mid x_{\tau}^{\star})\exp\big(\Delta_\delta(\tau)\big)\Big(\frac{x_{\tau}^{\star}}{\sqrt{1-\tau}} - x_0\Big) \diff x_0}
 {\int_{x_0} p_{\wt{X}_0 \mid \wt{X}_\tau}(x_0 \mid x_{\tau}^{\star})\exp\big(\Delta_\delta(\tau)\big) \diff x_0}}_{\mathrm{(II)}} \label{eq:expansion-temp}\\
&= \frac{\theta^\star(\tau)}{(1-\tau)^{3/2}} + u_1(\tau)\delta  + \frac{1}{2} u_2(\tau)\delta^2 + o(\delta^2), \label{eq:expansion}
\end{align}
where we denote the first and second derivatives of $\theta^\star$ by
\begin{align*}
u_1(\tau) := \frac{\diff}{\diff \tau}\frac{\theta^\star(\tau)}{(1-\tau)^{3/2}}
\qquad\text{and}\qquad
u_2(\tau) := \frac{\diff^2}{\diff \tau^2}\frac{\theta^\star(\tau)}{(1-\tau)^{3/2}},
\end{align*}
respectively.
Now, let us proceed to compute the expressions of terms (I) and (II).

\begin{itemize}
	\item For the term (I) in \eqref{eq:expansion-temp}, notice that the ODE \eqref{eq:ODE-star} implies that
\begin{align}
\frac{x_{\tau+\delta}^{\star}}{\sqrt{1-\tau-\delta}}-\frac{x_{\tau}^{\star}}{\sqrt{1-\tau}}
= - \frac{\theta^\star(\tau)}{2(1-\tau)^{3/2}}\delta - \frac{u_1(\tau)}{4} \delta^2 + o(\delta^2). \label{eq:x-tau-delta-diff}
\end{align}
Combining this with
\begin{align}
	\frac{1}{(\tau+\delta)(1-\tau-\delta)} = \frac{1}{\tau(1-\tau)} + \frac{2\tau-1}{\tau^2(1-\tau)^2} \delta + \frac{3\tau^2-3\tau+1}{\tau^3(1-\tau)^3} \delta^2 + o(\delta^2), \label{eq:tau-delta}
\end{align}
we can express the first term in \eqref{eq:expansion-temp} as
\begin{align}
 &-\frac{1}{(\tau+\delta)(1-\tau-\delta)} 
\bigg(\frac{x_{\tau+\delta}^{\star}}{\sqrt{1-\tau-\delta}} - \frac{x_{\tau}^{\star}}{\sqrt{1-\tau}}\bigg) \nonumber \\ 
 & \qquad=  \frac{\theta^\star(\tau)}{2\tau(1-\tau)^{5/2}} \delta + \frac12\bigg(\frac{u_1(\tau)}{2\tau(1-\tau)}+\frac{\theta^\star(\tau)}{(1-\tau)^{3/2}} \frac{2\tau-1}{\tau^2(1-\tau)^2}\bigg)\delta^2 + o(\delta^2) \label{eq:temp1}
\end{align}

\item Turning to (II) in \eqref{eq:expansion-temp}, let us first consider the term $\exp\big(\Delta_\delta(\tau)\big)$. We can decompose
\begin{align*}
	& \bigg(\frac{1}{\tau+\delta}-1\bigg)\bigg\|\frac{x_{\tau+\delta}^{\star}}{\sqrt{1-\tau-\delta}}-x_0\bigg\|_2^2 \\
	& \qquad \numpf{i}{=} \bigg(\frac{1-\tau}{\tau}-\frac{\delta}{\tau^2} + \frac{\delta^2}{\tau^3}+ o(\delta^2) \bigg)\bigg\|\frac{x_{\tau}^{\star}}{\sqrt{1-\tau}} - x_0 -\frac{\theta^\star(\tau)}{2(1-\tau)^{3/2}}\delta  - \frac{u_1(\tau)}{4}\delta^2 + o(\delta^2)\bigg\|_2^2 \\
	& \qquad= \bigg(\frac1\tau -1\bigg)\bigg\|\frac{x_{\tau}^{\star}}{\sqrt{1-\tau}} - x_0\bigg\|_2^2 
	-\delta\bigg(\frac{1}{\tau^2}\bigg\|\frac{x_{\tau}^{\star}}{\sqrt{1-\tau}} - x_0\bigg\|_2^2 + \frac{1-\tau}{\tau}\bigg\langle\frac{x_\tau^\star}{\sqrt{1-\tau}}-x_0,\frac{\theta^\star(\tau)}{(1-\tau)^{3/2}} \bigg\rangle \bigg) \\
	& \qquad\quad + \delta^2 \Bigg(\frac{1}{\tau^3}\bigg\|\frac{x_{\tau}^{\star}}{\sqrt{1-\tau}} - x_0\bigg\|_2^2+\bigg\langle\frac{x_\tau^\star}{\sqrt{1-\tau}}-x_0\,,\frac{\theta^\star(\tau)}{\tau^2(1-\tau)^{3/2}} - \frac{(1-\tau)u_1(\tau)}{2\tau} \bigg\rangle + \frac{1-\tau}{4\tau}\frac{\|\theta^\star(\tau)\|_2^2}{(1-\tau)^3}   \Bigg) \\
	& \qquad\quad + o(\delta^2),
\end{align*}
where (i) applies \eqref{eq:x-tau-delta-diff}.
Substituting this into \eqref{def:delta} yields that
\begin{align*}
\Delta_\delta(\tau) = f_1(\tau) \delta + \frac12f_2(\tau) \delta^2  + o(\delta^2),
\end{align*}
where $f_1(\tau)$ and $f_2(\tau)$ are given by
\begin{subequations}
\label{eq:f}
\begin{align}
	f_1(\tau) &\defn \frac12 \bigg(\frac{1}{\tau^2}\bigg\|\frac{x_{\tau}^{\star}}{\sqrt{1-\tau}} - x_0\bigg\|_2^2 + \frac{1-\tau}{\tau} \bigg\langle\frac{x_\tau^\star}{\sqrt{1-\tau}}-x_0,\,\frac{\theta^\star(\tau)}{(1-\tau)^{3/2}} \bigg\rangle \bigg); \label{eq:f1} \\
	f_2(\tau) &\defn -\frac{1}{\tau^3}\bigg\|\frac{x_{\tau}^{\star}}{\sqrt{1-\tau}} - x_0\bigg\|_2^2-\bigg\langle\frac{x_\tau^\star}{\sqrt{1-\tau}}-x_0,\,\frac{\theta^\star(\tau)}{\tau^2(1-\tau)^{3/2}} - \frac{(1-\tau)u_1(\tau)}{2\tau} \bigg\rangle - \frac{\|\theta^\star(\tau)\|_2^2}{4\tau(1-\tau)^2}. \label{eq:f2}
\end{align}
\end{subequations}
In particular, we obtain that for $\delta$ sufficiently small,
\begin{align*}
	\exp(\Delta) & = 1 + \Delta+\frac12\Delta^2 + o(\Delta^2) = 1 + f_1 \delta + \frac12 \big(f_2 +  f_1^2 \big)\delta^2 + o(\delta^2).
\end{align*}
This allows us to express
\begin{align*}
& \frac{\int_{x_0} p_{\wt{X}_0 \mid \wt{X}_\tau}(x_0 \mid x_{\tau}^{\star})\exp(\Delta)\Big(\frac{x_{\tau}^{\star}}{\sqrt{1-\tau}} - x_0\Big) \diff x_0}
	{\int_{x_0} p_{\wt{X}_0 \mid \wt{X}_\tau}(x_0 \mid x_{\tau}^{\star})\exp(\Delta) \diff x_0} \\
&\qquad = \frac{\int_{x_0} p_{\wt{X}_0 \mid \wt{X}_\tau}(x_0 \mid x_{\tau}^{\star})\big(1 + f_1 \delta + (f_2+f_1^2)\delta^2/2 \big)\Big(\frac{x_{\tau}^{\star}}{\sqrt{1-\tau}} - x_0\Big) \diff x_0 + o(\delta^2)}
	{1 + \int_{x_0} p_{\wt{X}_0 \mid \wt{X}_\tau}(x_0 \mid x_{\tau}^{\star})\big(f_1 \delta + (f_2+f_1^2)\delta^2/2 \big) \diff x_0 + o(\delta^2)} \\
&\qquad = \bigg(1 - \int_{x_0} p_{\wt{X}_0 \mid \wt{X}_\tau}(x_0 \mid x_{\tau}^{\star})\Big(f_1 \delta + \frac12(f_2+f_1^2)\delta^2 \Big) \diff x_0 + o(\delta^2) \bigg)  \\
	& \qquad \qquad \cdot \bigg(\int_{x_0} p_{\wt{X}_0 \mid \wt{X}_\tau}(x_0 \mid x_{\tau}^{\star})\Big(1 + f_1 \delta + \frac12(f_2+f_1^2)\delta^2 \Big)\Big(\frac{x_{\tau}^{\star}}{\sqrt{1-\tau}} - x_0\Big) \diff x_0 + o(\delta^2)\bigg) \\
&\qquad = \Big(1-g_1\delta-\frac12g_2\delta^2 + o(\delta^2)\Big)\Big(h_0+h_1\delta+\frac12h_2\delta^2+ o(\delta^2)\Big) \\
& \qquad = h_0 + \big(h_1-g_1h_0\big)\delta + \frac12(h_2-2g_1h_1-g_2h_0)\delta^2 + o(\delta^2),
\end{align*}
where we denote
\begin{subequations}
\label{eq:gh}
\begin{align}
	h_0(\tau) &\defn \int_{x_0} p_{\wt{X}_0 \mid \wt{X}_\tau}(x_0 \mid x_{\tau}^{\star})\Big(\frac{x_{\tau}^{\star}}{\sqrt{1-\tau}} - x_0\Big) \diff x_0 = -\frac{\tau}{\sqrt{1-\tau}}\theta^\star(\tau); \label{eq:h0} \\
	g_1(\tau) &\defn \int_{x_0} p_{\wt{X}_0 \mid \wt{X}_\tau}(x_0 \mid x_{\tau}^{\star})f_1(\tau)  \diff x_0; \label{eq:g1}
		\\
	g_2(\tau) &\defn \int_{x_0} p_{\wt{X}_0 \mid \wt{X}_\tau}(x_0 \mid x_{\tau}^{\star})\big(f_2(\tau) +  f_1^2(\tau) \big)  \diff x_0; \label{eq:g2}
		 \\
	h_1(\tau) &\defn \int_{x_0} p_{\wt{X}_0 \mid \wt{X}_\tau}(x_0 \mid x_{\tau}^{\star})f_1(\tau) \Big(\frac{x_{\tau}^{\star}}{\sqrt{1-\tau}} - x_0\Big)  \diff x_0; \label{eq:h1} \\
	h_2(\tau) &\defn \int_{x_0} p_{x_0 \mid \wt{X}_\tau}(x_0 \mid x_{\tau}^{\star})\big(f_2(\tau) +  f_1^2(\tau) \big) \Big(\frac{x_{\tau}^{\star}}{\sqrt{1-\tau}} - x_0\Big)  \diff x_0. \label{eq:h2}
\end{align}
\end{subequations}
Combining this with \eqref{eq:tau-delta} yields
\begin{align}
 -(\mathrm{II})&= \frac{1}{(\tau+\delta)(1-\tau-\delta)}
\frac{\int_{x_0} p_{\wt{X}_0 \mid \wt{X}_\tau}(x_0 \mid x_{\tau}^{\star})\exp(\Delta)\Big(\frac{x_{\tau}^{\star}}{\sqrt{1-\tau}} - x_0\Big) \diff x_0}
{\int_{x_0} p_{\wt{X}_0 \mid \wt{X}_\tau}(x_0 \mid x_{\tau}^{\star})\exp(\Delta) \diff x_0} \nonumber \\
&  = \frac{h_0}{\tau(1-\tau)} + \bigg( \frac{h_1-g_1h_0}{\tau(1-\tau)} + \frac{(2\tau-1)h_0}{\tau^2(1-\tau)^2} \bigg)\delta \nonumber \\
& \quad + \bigg( \frac{h_2-2g_1h_1-g_2h_0}{2\tau(1-\tau)} + \frac{(2\tau-1)(h_1-g_1h_0)}{\tau^2(1-\tau)^2} + \frac{(3\tau^2-3\tau+1)h_0}{\tau^3(1-\tau)^3}\bigg)\delta^2 + o(\delta^2) \label{eq:temp2}.
\end{align}

\item Putting \eqref{eq:temp1} and \eqref{eq:temp2} together with \eqref{eq:expansion} reveals that
\begin{subequations}
\begin{align}
	u_1(\tau) &= \frac{\theta^\star(\tau)}{2\tau(1-\tau)^{5/2}} - \frac{h_1(\tau)-g_1(\tau)h_0(\tau)}{\tau(1-\tau)} - \frac{(2\tau-1)h_0(\tau)}{\tau^2(1-\tau)^2}, \label{eq:u1-expression} \\
	u_2(\tau) &= \frac{u_1(\tau)}{2\tau(1-\tau)}+\frac{(2\tau-1)\theta^\star(\tau)}{\tau^2(1-\tau)^{7/2}} -  \frac{h_2(\tau)-2g_1(\tau)h_1(\tau)-g_2(\tau)h_0(\tau)}{\tau(1-\tau)}\nonumber \\ 
	& \quad -  \frac{2(2\tau-1)\big(h_1(\tau)-g_1(\tau)h_0(\tau)\big)}{\tau^2(1-\tau)^2}  -  \frac{2(3\tau^2-3\tau+1)h_0(\tau)}{\tau^3(1-\tau)^3}. \label{eq:u2-expression}
\end{align}
\end{subequations}
\end{itemize}

\paragraph{Step 2: bounds for the expected $\ell_2$ norms of the derivatives.}
\begin{itemize}
	\item Let us begin with bounding $\bE\big[\|u_1\|_2^k\big]$.
First, for any integer $k>0$, we can use \eqref{eq:s-star-l2norm-ub} in Lemma~\ref{lemma:x-star-x0-dist-l2norm} to bound $h_0$ in \eqref{eq:h0} by
\begin{align}
	\bE\big[\|h_0\|_2^k\big] = \frac{\tau^k}{(1-\tau)^{k/2}} \bE\Big[\big\|\tilde{s}^\star\big(\wt{X}_\tau,\tau\big)\big\|_2^k\Big] \leq C_k \frac{\tau^k}{(1-\tau)^{k/2}} \bigg(\frac{d}{\tau}\bigg)^{k/2} = C_k \bigg(\frac{\tau d}{1-\tau}\bigg)^{k/2} \label{eq:c0-l2norm-ub}
\end{align}
for some sufficiently large constant $C_k>0$ that only replies on $k$.

As for $f_1$ in \eqref{eq:f1} and $g_1$ in \eqref{eq:g1}, one can bound
\begin{align}
	\bE\big[|g_1|^k\big] 
	& \numpf{i}{\leq} \bE\big[|f_1|^k\big]\nonumber \numpf{ii}{\lesssim} C_k \bigg\{ \frac{1}{\tau^{2k}(1-\tau)^k} \bE \Big[\big\|\wt{X}_\tau - \sqrt{1-\tau} \wt{X}_0\big\|_2^{2k} \Big] \notag \\
	& \hspace{9em} + \frac{1}{\tau^k(1-\tau)^k} \sqrt{\bE\Big[\big\| \wt{X}_\tau-\sqrt{1-\tau} \wt{X}_0\big\|^{2k}\Big] \bE\Big[\big\|\tilde{s}^\star\big(\wt{X}_\tau,\tau\big) \big\|_2^{2k}\Big]} \bigg\} \nonumber\\
	& \numpf{iii}{\leq} C_k \bigg\{\frac{1}{\tau^{2k}(1-\tau)^k} (\tau d)^k + \frac{1}{\tau^k(1-\tau)^k} (\tau d)^{k/2} \bigg(\frac{d}{\tau}\bigg)^{k/2}\bigg\} \asymp C_k \bigg(\frac{d}{\tau(1-\tau)}\bigg)^k, \label{eq:b1-l2norm-ub}
\end{align}
where (i) follows from Jensen's inequality and the tower property, (ii) arises from the Cauchy-Schwarz inequality, and (iii) is due to \eqref{eq:x-star-x0-dist-l2norm}--\eqref{eq:s-star-l2norm-ub} in Lemma~\ref{lemma:x-star-x0-dist-l2norm}.

Turning to $h_1$ in \eqref{eq:h1}, we can bound
\begin{align}
	\bE\big[\| h_1 \|_2^k\big] & \numpf{i}{\leq} \frac{1}{(1-\tau)^{k/2}}
	\sqrt{\bE\big[ f_1^{2k} \big] \bE\Big[ \big\|\wt{X}_\tau - \sqrt{1-\tau} \wt{X}_0 \big\|_2^{2k} \Big]} \notag \\
	& \numpf{ii}{\lesssim} C_k \bigg(\frac{d}{\tau(1-\tau)}\bigg)^k \bigg(\frac{\tau d}{1-\tau}\bigg)^{k/2} =C_k \frac{1}{\tau^{k/2}}\bigg(\frac{d}{1-\tau}\bigg)^{3k/2}. \label{eq:c1-l2norm-ub}
\end{align}
where (i) follows from Jensen's inequality, the tower property, and the Cauchy-Schwarz inequality; (ii) uses \eqref{eq:b1-l2norm-ub} and \eqref{eq:x-star-x0-dist-l2norm} in Lemma~\ref{lemma:x-star-x0-dist-l2norm}.
Consequently, we have
\begin{align}
	\bE\big[\| h_1-g_1h_0 \|_2^{k}\big]
	&\lesssim \bE\big[\| h_1\|_2^{k}\big] + \sqrt{\bE\big[g_1^{2k}\big] \bE\big[\| h_0\|_2^{2k}\big]} \notag\\
	& \leq C_k \bigg\{\frac{1}{\tau^{k/2}}\bigg(\frac{d}{1-\tau}\bigg)^{3k/2} + \bigg(\frac{d}{\tau(1-\tau)}\bigg)^k \bigg(\frac{\tau d}{1-\tau}\bigg)^{k/2} \bigg\}\nonumber \\
	& \asymp C_k \frac{1}{\tau^{k/2}}\bigg(\frac{d}{1-\tau}\bigg)^{3k/2}. \label{eq:c1-b1c0-l2norm-ub}
\end{align}

With these bounds in place, we can substitute bounds \eqref{eq:b1-l2norm-ub}--\eqref{eq:c1-b1c0-l2norm-ub} into \eqref{eq:u1-expression} to obtain
\begin{align}
	\bE\big[\| u_1 \|_2^{k}\big] & \lesssim C_k\bigg\{ \frac{\bE\big[\|\theta^\star\|_2^k\big]}{\tau^k(1-\tau)^{5k/2}} + \frac{\bE\big[\| h_1-g_1h_0 \|_2^k\big]}{\tau^k(1-\tau)^k} + \frac{\bE\big[\| h_0 \|_2^k\big]}{\tau^{2k}(1-\tau)^{2k}} \bigg\}  \nonumber \\
	& \lesssim C_k \bigg\{\frac{1}{\tau^k (1-\tau)^{5k/2}} \bigg(\frac{d}{\tau}\bigg)^{k/2} 
	+ \frac{1}{\tau^k(1-\tau)^k}\frac{1}{\tau^{k/2}}\bigg(\frac{d}{1-\tau}\bigg)^{3k/2} 
	+ \frac{1}{\tau^{2k} (1-\tau)^{2k}} \bigg(\frac{\tau d}{1-\tau}\bigg)^{k/2} \bigg\} \nonumber \\
	& \asymp C_k \frac{1}{(1-\tau)^k}\bigg(\frac{d}{\tau(1-\tau)}\bigg)^{3k/2}. \label{eq:u1-l2norm-ub}
\end{align}
This proves the claim in \eqref{eq:u1-l2norm-ub-lemma}.

\item 
It remains to control the $\ell_2$ norm of $u_2$. To this end,
applying the Cauchy-Schwartz inequality to $f_2$ defined in \eqref{eq:f2} leads to
\begin{align}
	\bE\big[ |f_2|^{k}\big] & \lesssim C_k\bigg\{ \frac{1}{\tau^{3k}(1-\tau)^k} \bE\Big[\big\|\wt{X}_\tau - \sqrt{1-\tau}\wt{X}_0\big\|_2^{2k} \Big] \notag \\
	& \qquad \qquad + \frac{1}{\tau^{2k}(1-\tau)^{2k}} \sqrt{\bE\Big[\big\| \wt{X}_\tau-\sqrt{1-\tau} \wt{X}_0\big\|^{2k}\Big] \bE\Big[\big\|\tilde{s}^\star\big(\wt{X}_\tau,\tau\big) \big\|_2^{2k}\Big]} \nonumber \\
	& \qquad \qquad+ \frac{(1-\tau)^{k/2}}{\tau^k} \sqrt{\bE\Big[\big\| \wt{X}_\tau-\sqrt{1-\tau} \wt{X}_0\big\|^{2k}\Big] \bE\big[\|u_1\|_2^{2k}\big]} + \frac{1}{\tau^k(1-\tau)^{2k}}\bE\Big[\big\|\tilde{s}^\star\big(\wt{X}_\tau,\tau\big) \big\|_2^{2k}\Big] \bigg\} \nonumber \\
	& \lesssim C_k \bigg\{\frac{1}{\tau^{3k}(1-\tau)^k} (\tau d)^{k}  + \frac{1}{\tau^{2k}(1-\tau)^{2k}} (\tau d)^{k/2} \bigg(\frac{d}{\tau}\bigg)^{k/2}  \nonumber \\
	& \qquad\qquad + \frac{(1-\tau)^{k/2}}{\tau^k} (\tau d)^{k/2} \frac{1}{(1-\tau)^k}\bigg(\frac{d}{\tau(1-\tau)}\bigg)^{3k/2}  + \frac{1}{\tau^k(1-\tau)^{2k}} \bigg(\frac{d}{\tau}\bigg)^{k} \bigg\} \nonumber \\
	& \asymp C_k \bigg(\frac{d}{\tau(1-\tau)}\bigg)^{2k}, \label{eq:a2-ub}
\end{align}
where the second inequality uses \eqref{eq:x-star-x0-dist-l2norm}--\eqref{eq:s-star-l2norm-ub} in Lemma~\ref{lemma:x-star-x0-dist-l2norm} and \eqref{eq:u1-l2norm-ub}.

Consequently, we can bound $g_2$ in \eqref{eq:g2} by
\begin{align}
	\bE\big[|g_2|^k\big] & \numpf{i}{\lesssim} C_k\Big\{ \bE\big[ |f_2|^k \big]+\bE\big[f_1^{2k} \big] \Big\}
	 \numpf{ii}{\lesssim} C_k \bigg(\frac{d}{\tau(1-\tau)}\bigg)^{2k} \label{eq:b2_ub},
\end{align}
where (i) arises from Jensen's inequality and the tower property; (ii) uses \eqref{eq:a2-ub} and \eqref{eq:b1-l2norm-ub}.

Furthermore, one can control $h_2$ in \eqref{eq:h2} by
\begin{align}
	\bE\big[ \|h_2\|_2^k \big] 
	& \numpf{i}{\lesssim} C_k \bigg\{\sqrt{\bE\big[f_1^{4k}\big]}+\sqrt{\bE\big[ f_2^{2k}\big]}\bigg\} \sqrt{\bE\bigg[\Big\|\frac{\wt{X}_\tau}{\sqrt{1-\tau}}- \wt{X}_0\Big\|_2^{2k}\bigg]} \nonumber \\
	& \numpf{ii}{\lesssim} C_k \bigg(\frac{d}{\tau(1-\tau)}\bigg)^{2k} \bigg(\frac{\tau d }{1-\tau}\bigg)^{k/2} \nonumber\\
	& \asymp C_k \frac{1}{\tau^{3k/2}}\bigg(\frac{d }{1-\tau}\bigg)^{5k/2} \label{eq:c2-l2norm-ub},
\end{align}
where (i) results from Jensen's inequality, the tower property, the Cauchy-Schwarz inequality, and $\sqrt{a+b}\leq \sqrt{a}+\sqrt{b}$ for any $a,b>0$; (ii) applies \eqref{eq:b1-l2norm-ub}, \eqref{eq:a2-ub}, and \eqref{eq:x-star-x0-dist-l2norm} in Lemma~\ref{lemma:x-star-x0-dist-l2norm}.
This allows us to bound
\begin{align}
	& \bE \big[\|h_2-g_1h_1-g_2h_0\|_2^k\big] \notag \\
	& \qquad \lesssim C_k\bigg\{\bE \big[\|h_2\|_2^k\big] + \sqrt{\bE\big[g_1^{2k}\big]\bE\big[\|h_1\|_2^{2k}\big]} + \sqrt{\bE\big[g_2^{2k}\big]\bE\big[\|h_0\|_2^{2k}\big]}\bigg\} \nonumber \\
	& \qquad\lesssim C_k\bigg\{ \frac{1}{\tau^{3k/2}}\bigg(\frac{d }{1-\tau}\bigg)^{5k/2} + \bigg(\frac{d}{\tau(1-\tau)}\bigg)^k \frac{1}{\tau^{k/2}}\bigg(\frac{d}{1-\tau}\bigg)^{3k/2} + \bigg(\frac{d}{\tau(1-\tau)}\bigg)^{2k}\bigg(\frac{\tau d}{1-\tau}\bigg)^{k/2} \bigg\} \nonumber \\
	& \qquad \asymp C_k \frac{1}{\tau^{3k/2}}\bigg(\frac{d }{1-\tau}\bigg)^{5k/2}. \label{eq:c2-b1c1-b2c0-l2norm-ub}
\end{align}

Finally, substituting bounds \eqref{eq:u1-l2norm-ub}, \eqref{eq:s-star-l2norm-ub}, \eqref{eq:c2-b1c1-b2c0-l2norm-ub}, \eqref{eq:c1-b1c0-l2norm-ub}, and \eqref{eq:c0-l2norm-ub} into \eqref{eq:u2-expression} leads to
\begin{align*}
	\bE\big[\|u_2\|_2^k\big] &\lesssim C_k \bigg\{ \frac{\bE\big[\|u_1\|_2^k\big]}{\tau^k(1-\tau)^k}+\frac{\bE\big[\|\theta^\star\|_2^k\big]}{\tau^{2k}(1-\tau)^{7k/2}} \\
	& \qquad\qquad + \frac{\bE\big[\|h_2-g_1h_1-g_2h_0\|_2^k\big]}{\tau^k(1-\tau)^k} + \frac{\bE\big[\|h_1-g_1h_0\|_2^k\big]}{\tau^{2k}(1-\tau)^{2k}} + \frac{\bE\big[\|h_0\|_2^k\big]}{\tau^{3k}(1-\tau)^{3k}} \bigg\} \\
	& \lesssim C_k \bigg\{\frac{1}{\tau^k(1-\tau)^k} \frac{1}{(1-\tau)^k}\bigg(\frac{d}{\tau(1-\tau)}\bigg)^{3k/2} 
	+\frac{1}{\tau^{2k}(1-\tau)^{7k/2}} \bigg(\frac{d}{\tau}\bigg)^{k/2}  \notag\\
	& \qquad \qquad + \frac{1}{\tau^k(1-\tau)^k} \frac{1}{\tau^{3k/2}}\bigg(\frac{d }{1-\tau}\bigg)^{5k/2}
	 + \frac{1}{\tau^{2k}(1-\tau)^{2k}} \frac{1}{\tau^{k/2}}\bigg(\frac{d}{1-\tau}\bigg)^{3k/2} 
	\nonumber \\
	& \qquad \qquad+ \frac{1}{\tau^{3k}(1-\tau)^{3k}} \bigg(\frac{\tau d}{1-\tau}\bigg)^{k/2} \bigg\} \\
	& \asymp C_k \frac{1}{(1-\tau)^k}\bigg(\frac{d }{\tau(1-\tau)}\bigg)^{5k/2},
\end{align*}
as claimed in \eqref{eq:u2-l2norm-ub-lemma}.
\end{itemize}
\paragraph{Step 3: a high-probability bound for the $\ell_2$ norm of the first derivative.}
Fix an arbitrary $\tau\in[0,1)$ such that $-\log p_{\wt{X}_\tau}(x_{\tau}^{\star}) \lesssim d\log T$.
By Lemma \ref{lem:bound_cond}, we know that
\begin{align*}
\mathbb{E}\Bigg[\bigg\|\frac{x^\star_\tau}{\sqrt{1-\tau}} - \wt{X}_0 \bigg\|_2^k\,\Big|\, \wt{X}_\tau = x_{\tau}^{\star}\Bigg] &\leq C_k \bigg(\frac{d\tau\log T}{1-\tau}\bigg)^{k/2}
\end{align*}
for some constant $C_k>0$ that only depends on $k$.
Further, recalling $\theta^\star(\tau) \defn \tilde{s}^\star(x_\tau^\star,\tau)$ is the score function of $p_{\wt X_\tau}$, we know from \eqref{def-score-cont} that
\begin{align*}
	\big\|\theta^\star(\tau)\big\|_2 & = \frac{1}{\tau} \mathbb{E}\Big[\big\|x^\star_\tau - \sqrt{1-\tau} \wt{X}_0 \big\|_2\mid \wt{X}_\tau = x_{\tau}^{\star}\Big] \lesssim \sqrt{\frac{d\log T}{\tau}}.
\end{align*}
As a result, plugging these bounds into the expressions in \eqref{eq:f} and \eqref{eq:gh} yields that
\begin{align*}
	\big\|h_0(\tau)\big\| &\lesssim \frac{\tau}{\sqrt{1-\tau}} \sqrt{\frac{d\log T}{\tau}} = \sqrt{\frac{d\tau\log T}{1-\tau}};
\end{align*}
\begin{align*}
	\big|g_1(\tau)\big| & \lesssim \frac{1}{\tau^2} \int_{x_0}  p_{\wt{X}_0 \mid \wt{X}_\tau}(x_0 \mid x_{\tau}^{\star}) \bigg\|\frac{x_{\tau}^{\star}}{\sqrt{1-\tau}} - x_0\bigg\|_2^2  \diff x_0  + \frac{1}{\tau} \int_{x_0}  p_{\wt{X}_0 \mid \wt{X}_\tau}(x_0 \mid x_{\tau}^{\star}) \bigg\|\frac{x_{\tau}^{\star}}{\sqrt{1-\tau}} - x_0\bigg\|_2 \frac{\| \theta^\star(\tau)\|_2}{\sqrt{1-\tau}}  \diff x_0 \\
& = \frac{1}{\tau^2} \frac{d\tau\log T}{1-\tau} + \frac{1}{\tau} \sqrt{\frac{d\tau\log T}{1-\tau}} \sqrt{\frac{d\log T}{\tau(1-\tau)}} \asymp \frac{d\log T}{\tau(1-\tau)};
\end{align*}
and
\begin{align*}
\big\|h_1(\tau)\big\|_2 & \lesssim \frac{1}{\tau^2} \int_{x_0}  p_{\wt{X}_0 \mid \wt{X}_\tau}(x_0 \mid x_{\tau}^{\star}) \bigg\|\frac{x_{\tau}^{\star}}{\sqrt{1-\tau}} - x_0\bigg\|_2^3  \diff x_0 \\
& \quad + \frac{1}{\tau} \int_{x_0}  p_{\wt{X}_0 \mid \wt{X}_\tau}(x_0 \mid x_{\tau}^{\star}) \bigg\|\frac{x_{\tau}^{\star}}{\sqrt{1-\tau}} - x_0\bigg\|_2^2 \frac{\| \theta^\star(\tau)\|_2}{\sqrt{1-\tau}}  \diff x_0 \\
& = \frac{1}{\tau^2} \bigg(\frac{d\tau\log T}{1-\tau}\bigg)^{3/2} + \frac{1}{\tau} \frac{d\tau\log T}{1-\tau} \sqrt{\frac{d\log T}{\tau(1-\tau)}} \asymp \frac{1}{\sqrt{\tau}} \bigg(\frac{d\log T}{1-\tau}\bigg)^{3/2}.
\end{align*}
Taking these bounds collectively with \eqref{eq:u1-expression}, we arrive at
\begin{align*}
	\big\|u_1(\tau)\big\|_2^k & \lesssim C_k\bigg\{ \frac{\big\|\theta^\star(\tau)\big\|_2^k}{\tau^k(1-\tau)^{5k/2}} + \frac{\big\| h_1(\tau)\big\|_2^k + |g_1(\tau)|^k\big\|h_0(\tau) \big\|_2^k}{\tau^k(1-\tau)^k} + \frac{\big\| h_0(\tau) \big\|_2^k}{\tau^{2k}(1-\tau)^{2k}} \bigg\} \\
	& \lesssim C_k \bigg\{ \frac{1}{\tau^k(1-\tau)^{5k/2}} \bigg(\frac{d\log T}{\tau} \bigg)^{k/2} + \frac{1}{\tau^k(1-\tau)^{k}} \frac{1}{\tau^{k/2}} \bigg(\frac{d\log T}{1-\tau}\bigg)^{3k/2} \notag \\
	& \qquad \qquad+ \frac{1}{\tau^{2k}(1-\tau)^{2k}} \bigg(\frac{d\tau\log T}{1-\tau}\bigg)^{k/2} \bigg\} \\
	& \asymp C_k \frac{1}{(1-\tau)^{k}} \bigg(\frac{d\log T}{\tau(1-\tau)}\bigg)^{3k/2},
\end{align*}
as claimed in \eqref{eq:u1-l2norm-high-prob-ub}.
This completes the proof of Lemma~\ref{lemma:u1-u2-l2norm}.

\subsection{Proof of Lemma~\ref{lemma:x-star-x0-dist-l2norm}}
\label{sec:pf-lemma:x-star-x0-dist-l2norm}

We begin with the first claim. Note that conditional on $\wt{X}_0$, $\wt{X}_\tau$ satisfies $\wt{X}_\tau \sim \cN\big(\sqrt{1-\tau}\wt{X}_0,\tau I_d\big)$. For any integer $k\geq 2$, the standard Gaussian random vector $Z\sim\cN(0,I_d)$ satisfies
\begin{align*}
	\bE \big[ \|Z\|_2^k \big] & = \bE \bigg[ \Big(\sum_{i=1}^d Z_i^2\Big)^{k/2} \bigg] \numpf{i}{\leq} d^{k/2-1} \bE \bigg[ \sum_{i=1}^d| Z_i|^k \bigg] = d^{k/2} \bE \big[|Z_1|^k\big]
	 \numpf{ii}{\leq} \frac{2^{k/2} \Gamma\big( \frac{k+1}{2} \big)}{\sqrt{\pi}}  d^{k/2} ,
\end{align*}
where (i) uses Jensen's equality that the convexity of $x\mapsto x^{k/2}$ when $k\geq 2$; (ii) uses the moment property of the Gaussian random variable and $\Gamma(\cdot)$ is the Gamma function. As for $k=1$, we can also invoke Jensen's inequality to find
\begin{align*}
	\bE \big[ \|Z\|_2 \big] \leq \sqrt{\bE \big[ \|Z\|_2^2 \big]} = \sqrt{d}.
\end{align*}
Consequently, it follows from the tower property that
\begin{align*}
	\bE \Big[ \big\|\wt{X}_\tau -\sqrt{1-\tau}\wt{X}_0\big\|_2^k \Big] = \bE \bigg[\bE \Big[ \big\|\wt{X}_\tau -\sqrt{1-\tau}\wt{X}_0\big\|_2^k \,\big|\, \wt{X}_0 \Big] \bigg] \leq C_k (\tau d)^{k/2}
\end{align*}
for some constant $C_k$ that only relies on $k$.

Turning to the second claim, note that $\tilde s^\star(x,{\tau}) = -\frac1\tau \bE \big[ \wt{X}_\tau -\sqrt{1-\tau}\wt{X}_0  \,\big|\, \wt{X}_\tau =x \big] $. As a result, one can apply Jensen's inequality and the tower property to get
\begin{align*}
	\bE\Big[ \big\|\tilde s^\star\big(\wt{X}_\tau,\tau\big)\big\|_2^k \Big] & = \frac{1}{\tau^k} \bE\bigg[ \Big\|\bE \big[ \wt{X}_\tau -\sqrt{1-\tau}\wt{X}_0  \,\Big|\, \wt{X}_\tau  \big]\Big\|_2^k \bigg]
	\leq \frac{1}{\tau^k} \bE\bigg[ \bE \Big[ \big\| \wt{X}_\tau -\sqrt{1-\tau}\wt{X}_0  \big\|_2^k \,\big|\, \wt{X}_\tau  \Big] \bigg] \\
	& = \frac{1}{\tau^k} \bE \Big[ \big\|\wt{X}_\tau -\sqrt{1-\tau}\wt{X}_0\big\|_2^k \Big] \leq C_k \bigg(\frac d\tau\bigg)^{k/2}.
\end{align*}

\subsection{Proof of Lemma~\ref{lem:bound_cond}}
\label{sec:pf-lem:bound_cond}
The claim \eqref{eq:bound_cond} can be established using the same analysis as in~\citet[Lemma 1]{li2024d}. For brevity of presentation, we omit the detailed proof.
In what follows, we shall focus on proving \eqref{lem:x-tau-pdf-lb}.

Let us assume that there exists some $\wt \tau\in(0,1)$ with $|\wt \tau - \tau| \leq c_0 \tau(1-\tau)$ such that
\begin{align*}
\theta' \defn \sup_{\wt \tau \wedge \tau \le \tau' \le \wt \tau \vee \tau} \frac{-\log p_{\wt{X}_{\tau'}}(x_{\tau'}^{\star})}{d\log T} > 2\theta. 
\end{align*}
Otherwise, \eqref{lem:x-tau-pdf-lb} holds directly. 
In particular, by \eqref{eq:bound_cond} and \eqref{eq:trace-J-ub-derivation}, this assumption implies that
\begin{align*}
	\sup_{\wt \tau \wedge \tau \le \tau' \le \wt \tau \vee \tau} \tr \big(\wt J(x^\star_{\tau'},\tau') \big) \leq \frac{C_7\theta'd\log T}{\tau'(1-\tau')} - \frac{d}{\tau'(1-\tau')} \leq \frac{C_7}{2} \frac{\theta'd\log T}{\tau'(1-\tau')},
\end{align*}
for some sufficiently large absolute constant $C_7>0$, provided $T$ is sufficiently large.

Now, by the probability flow ODE \eqref{eq:ODE-star}, it is straightforward to calculate that
\begin{align*}
	\frac{\diff}{\diff \tau'}  \log p_{\wt{X}_{\tau'}/\sqrt{1-\tau'}}\big(x_{\tau'}^{\star}/\sqrt{1-\tau'}\big) 
	& = \frac{\diff}{\diff \tau'}\log \frac{p_{\wt{X}_{\tau'}/\sqrt{1-\tau'}}\big(x_{\tau'}^{\star}/\sqrt{1-\tau'}\big)}{p_{\wt{X}_\tau/\sqrt{1-\tau}}\big(x_{\tau}^{\star}/\sqrt{1-\tau}\big)} \\
	& = \frac{\diff}{\diff \tau'} \log \deter^{-1}\bigg(\frac{\partial x_{\tau'}^{\star}/\sqrt{1-\tau'}}{\partial x_{\tau}^{\star}/\sqrt{1-\tau}}\bigg) \\
	& = - \frac{\diff}{\diff \tau'} \log \deter \bigg\{\exp\biggl(\int_{\tau}^{\tau'} -\frac12 \wt J(x_{\tau''}^\star,\tau'') \diff \tau'' \bigg)\bigg\} \\
	& = \frac12 \tr\big(\wt J(x_{\tau'}^\star,\tau')\big),
\end{align*}
where the penultimate line uses \eqref{eq:partial-derivative-x}.

Also, for any $\tau'\in(0,1)$ such that $|\tau' - \tau|\leq |\wt \tau - \tau| \leq \tau(1-\tau)/2$, one has
\begin{align*}
	\bigg|\frac{1}{\tau'(1-\tau')}-\frac{1}{\tau(1-\tau)}\bigg| = \frac{|(\tau'-\tau)(1-\tau'-\tau)|}{\tau(1-\tau)} \frac{1}{\tau'(1-\tau')} \leq \frac{1}{2\tau'(1-\tau')}.
\end{align*}
Therefore, combined with $p_{\wt{X}_\tau/\sqrt{1-\tau}}(x_{\tau}^{\star}/\sqrt{1-\tau}) = (1-\tau)^{d/2} p_{X_{\tau}}(x_{\tau}^{\star})$,
this tells us that any $\tau' \in [\wt \tau \wedge \tau, \wt \tau \vee \tau]$, the following holds for $c_0 $ sufficiently small:
\begin{align*}
-\log p_{\wt{X}_{\tau'}}(x_{\tau'}^{\star}) & \leq -\log p_{\wt{X}_\tau}(x_{\tau}^{\star}) + \frac d2 \big|\log(1-\tau')-\log(1-\tau)\big| + |\tau' - \tau|\frac{C_7}{2}\frac{\theta' d\log T}{\tau(1-\tau)} \\
& \numpf{i}{\leq} \theta d \log T + \frac d2 \frac{|\tau'-\tau|}{1-\tau} + |\tau' - \tau|\frac{C_7}{2}\frac{\theta' d\log T}{\tau(1-\tau)} \\
& \numpf{ii}{\leq} \theta d \log T + \frac{c_0 }{2}d + \frac{c_0 C_7}{2}\theta' d\log T
\numpf{iii}{\leq} \frac34 \theta' d \log T,
\end{align*}
where we use $\log(1+x)\leq |x|$ for any $x$, (ii) is true as $|\tau' - \tau|\leq |\wt \tau - \tau| \leq c_0 \tau(1-\tau) \leq c_0 (1-\tau)$, (iii) holds by the assumption $\theta' > 2\theta$.
This leads to a contradiction. Hence, the proof of \eqref{lem:x-tau-pdf-lb} is complete.

\end{document}